\documentclass{article}
\usepackage[abbrvbib,preprint]{jmlr2e}

\usepackage{graphicx}
\usepackage{amsmath}
\usepackage{multicol}
\usepackage{mathrsfs}
\usepackage{amssymb}
\usepackage[dvipsnames]{xcolor}
\usepackage{enumerate}
\usepackage{bm}
\usepackage{fancyvrb}
\usepackage{booktabs}
\usepackage{ulem}
\usepackage{changepage}
\usepackage{paralist}
\usepackage{subfigure}
\usepackage{pbox}
\usepackage[onehalfspacing]{setspace}
\usepackage[top=100pt,bottom=100pt,left=100pt,right=110pt,footskip=30pt]{geometry}

\usepackage[capitalise,nameinlink,noabbrev]{cleveref}
\Crefname{lemma}{Lemma}{Lemmas}
\Crefname{proposition}{Proposition}{Propositions}
\Crefname{remark}{Remark}{Remarks}
\Crefname{corollary}{Corollary}{Corollaries}
\let\oldlemma\lemma
\renewcommand{\lemma}{%
  \crefalias{theorem}{lemma}
  \oldlemma}
\let\oldproposition\proposition
\renewcommand{\proposition}{%
  \crefalias{theorem}{proposition}
  \oldproposition}
\let\oldremark\remark
\renewcommand{\remark}{%
  \crefalias{theorem}{remark}
  \oldremark}
\let\oldcorollary\corollary
\renewcommand{\corollary}{%
  \crefalias{theorem}{corollary}
  \oldcorollary}

\usepackage{caption}

\allowdisplaybreaks

\graphicspath{{./},{./figures/}}

\hypersetup{
    colorlinks = true, linkcolor=blue, citecolor=blue}

\newcommand{\bbE}{\mathbf{E}}

\newcommand{\bbH}{\mathbb{H}}

\newcommand{\bbN}{\mathbb{N}}

\newcommand{\bbR}{\mathbb{R}}
\newcommand{\bbS}{\mathbb{S}}

\newcommand{\bp}{\mathbf{p}}

\newcommand{\bfE}{\mathbf{E}}

\newcommand{\ind}{\mathbf{1}}

\newcommand{\cN}{\mathcal{N}}

\DeclareMathOperator{\var}{Var}

\newcommand{\R}{\mathbb{R}}
\newcommand{\N}{\mathbb{N}}

\def\bi{\begin{itemize}}
	\def\ei{\end{itemize}}

\def\eps{\epsilon}

\def\be{\begin{equation}}
	\def\ee{\end{equation}}
\def\bea{\begin{eqnarray}}
	\def\eea{\end{eqnarray}}

\def\({\left(}
\def\){\right)}

\def\t{\tilde}

\def\1{\mathbf{1}}

\def \bx{\mathbf{x}}
\def \bn{\mathbf{p}}
\def \E{\mathbf{E}}

\def \K{\mathbb{K}}

\def \Clip{C_{\mathrm{Lip}}^2}

\DeclareMathOperator{\gammadist}{Gamma}
\DeclareMathOperator{\betadist}{Beta}

\DeclareMathOperator{\IG}{IG}
\DeclareMathOperator{\Ber}{Bernoulli}
\DeclareMathOperator{\Cauchy}{Cauchy_+}
\DeclareMathOperator{\Poi}{Poisson}
\DeclareMathOperator{\id}{\operatorname{ID}}
\DeclareMathOperator{\GP}{\operatorname{GP}}

\def \din{{d_{\mathrm{in}}}}
\def \dout{{d_{\mathrm{out}}}}
\DeclareMathOperator{\cvp}{\mathtt{pr}}
\DeclareMathOperator{\cvd}{\mathtt{d}}
\let\cite\citep
\DeclareMathOperator{\cvdto}{\overset{\cvd}{\to}}
\DeclareMathOperator{\cvpto}{\overset{\cvp}{\to}}
\DeclareMathOperator{\stable}{Stable}
\DeclareMathOperator{\ellipK}{EllipK}
\DeclareMathOperator{\ellipE}{EllipE}
\DeclareMathOperator{\pareto}{Pareto}
\DeclareMathOperator{\betaprime}{Betaprime}

\usepackage{lastpage}
\ShortHeadings{Deep neural nets with dependent weights}{Lee, Ayed, Jung, Lee, Yang and Caron}
\firstpageno{1}

\begin{document}

\title{Deep Neural Networks with Dependent Weights: \\Gaussian Process Mixture Limit, Heavy Tails, Sparsity and Compressibility}

\author{\name Hoil Lee \email hoil.lee@kaist.ac.kr \\
       \addr Department of Mathematical Sciences, KAIST\\
       Daejeon, South Korea
       \AND
       \name Fadhel Ayed \email fadhel.ayed@gmail.com\\
       \addr Huawei Technologies\\
       Paris, France
       \AND
       \name Paul Jung \email pjung3@fordham.edu\\
       \addr Department of Mathematics, Fordham University\\
       New York City, USA
       \AND
       \name Juho Lee \email juholee@kaist.ac.kr\\
       \addr Kim Jaechul Graduate School of AI, KAIST\\
       Daejeon, South Korea
       \AND
       \name Hongseok Yang \email hongseok.yang@kaist.ac.kr\\
       \addr School of Computing and Kim Jaechul Graduate School of AI, KAIST\\
       Discrete Mathematics Group, Institute for Basic Science (IBS)\\
       Daejeon, South Korea
       \AND
       \name Fran\c cois Caron \email caron@stats.ox.ac.uk\\
       \addr Department of Statistics, University of Oxford\\
       Oxford, United Kingdom
       }

\editor{Daniel Roy}

\maketitle

\begin{abstract}
This article studies the infinite-width limit of deep feedforward neural networks whose weights are dependent, and modelled via a mixture of Gaussian distributions. Each hidden node of the network is assigned a nonnegative random variable that controls the variance of the outgoing weights of that node. We make minimal assumptions on these per-node random variables: they are iid and their sum, in each layer, converges to some finite random variable in the infinite-width limit. Under this model, we show that each layer of the infinite-width neural network can be characterised by two simple quantities: a non-negative scalar parameter and a L\'evy measure on the positive reals. If the scalar parameters are strictly positive and the L\'evy measures are trivial at all hidden layers, then one recovers the classical Gaussian process (GP) limit, obtained with iid Gaussian weights. More interestingly, if the L\'evy measure of at least one layer is non-trivial, we obtain a mixture of Gaussian processes (MoGP) in the large-width limit. The behaviour of the neural network in this regime is very different from the GP regime. One obtains correlated outputs, with non-Gaussian distributions, possibly with heavy tails. Additionally, we show that, in this regime, the weights are compressible, and some nodes have asymptotically non-negligible contributions, therefore representing important hidden features. Many sparsity-promoting neural network models can be recast as special cases of our approach, and we discuss their infinite-width limits; we also present an asymptotic analysis of the pruning error. We illustrate some of the benefits of the MoGP regime over the GP regime in terms of representation learning and compressibility on simulated, MNIST and Fashion MNIST datasets.
\end{abstract}


\newpage
\tableofcontents

\newpage
\section{Introduction}
\label{sec:intro}
Two decades after the seminal work of Radford~\citet{Neal1996}, the last few years have seen a renewed and growing interest in the analysis of (deep) neural networks, with random weights, in the infinite-width limit. When the weights are independently, identically  distributed (iid) and suitably scaled Gaussian random variables, the random function associated to this random neural network converges to a Gaussian process~\citep{Neal1996,Lee2018,Matthews2018,Yang2019,Bracale2021}. The connection to Gaussian processes has deepened our understanding of large neural networks, and motivated the use of Bayesian or kernel regression inference methods~\cite{Lee2018} or the development of kernel methods for training via gradient descent~\citep{Jacot2018} in the infinite-width limit.

While insightful, the Gaussian process connection also highlighted some of the limitations of large-width neural networks with iid Gaussian weights. As already noted by~\citet{Neal1995}, \textit{``with Gaussian priors the contributions of individual hidden units are all negligible, and consequently, these units do not represent `hidden features' that capture important aspects of the data."} %
Additionally, the different dimensions of the output of the neural network become independent Gaussian processes in the infinite-width limit, which is generally undesirable. Finally, from a Bayesian perspective, the Gaussian independence assumption on weights is often seen as unrealistic: estimated weights of deep neural networks generally exhibit dependencies and heavy tails~\cite{Martin2019,Wenzel2020,Fortuin2021a}, and thus a family of prior distributions which allow for heavy tails is desirable. To alleviate some of these limitations, iid non-Gaussian random weights have been considered, either assuming stable~\cite{Neal1996,Der2006,Favaro2020}, or more generally light-tailed/heavy-tailed distributions~\cite{Jung2021}. However, due to the same iid assumption, some of the above limitations pertain, such as independence of the dimensions of the output.

We consider a more structured distribution on the weights of the neural network. We assume that weights emanating from a given node are dependent, where the dependency is captured via a scale mixture of Gaussians. More precisely, for a weight $W_{jk}^{(l+1)}$ between node $j=1,\ldots,p_{l}$ at hidden layer $l$ and node $k=1,\ldots,p_{l+1}$ at hidden layer $l+1$, we assume that
\begin{align}
W_{jk}^{(l+1)}=\sqrt{\lambda_{p_{l},j}^{(l)}} V_{jk}^{(l+1)}\label{eq:scalemixtureW}
\end{align}
where $\lambda_{p_{l},j}^{(l)}$, for $j=1,\ldots,p_{l}$, are nonnegative iid random variance parameters, one for each node $j=1,\ldots,p_{l}$ at layer $l$, and $V_{jk}^{(l+1)}$ are iid centred Gaussian random variables with variance $\sigma_v^2>0$. The per-node variance term $\lambda_{p_l,j}^{(l)}$ induces some dependency over the weights $W_{j1}^{(l+1)},\ldots,W_{jp_{l+1}}^{(l+1)}$ connected to node $j$.
  As we describe in the next paragraph, this assumption has been considered by a number of authors for training (finite) neural networks either (i) as a prior for Bayesian learning and pruning of neural networks, or (ii) as an implicit prior where a regularised empirical risk minimiser with group-sparse penalty is interpreted as a maximum a posteriori estimator, or (iii) as a random weight initialisation scheme for stochastic gradient descent.

A number of articles considered prior distributions of the form in \cref{eq:scalemixtureW} for Bayesian learning of deep neural networks. Examples of distributions considered for the random variance $\lambda_{p_{l},j}^{(l)}$ include the Bernoulli~\cite{Jantre2021}, the horseshoe~\cite{Louizos2017,Ghosh2018,Ghosh2019,Popkes2019}, the gamma~\cite{Scardapane2017,Wang2017}, the inverse gamma~\cite{Ober2021}, or the improper Jeffrey distributions~\cite{Louizos2017}. See \cite[Section 4.1]{Fortuin2021} for a recent review. Distributions concentrated around 0, like the horseshoe, or with mass at 0, like the Bernoulli, favour more sparse-like representations, and they have often been used for compression of deep neural networks, by pruning nodes based on the posterior distributions of the per-node variance parameter. Using a similar idea but with a slightly different formulation, \citet{Adamczewski2021} considered a joint Dirichlet distribution for the square root of the variances. In \cref{sec:examples} and \cref{app:examples}, we discuss these examples in the context of our general framework.

These structured priors are also related to non-Bayesian estimators based on regularised empirical risk minimisation, where the estimator can be interpreted as a maximum a posteriori estimator under these priors. A typical example is the group lasso penalty on the weights of a neural network, used in a number of articles~\cite{Murray2015,Scardapane2017,Wang2017,Ochiai2017}, which can be interpreted as a negative log-prior on the weights when $\lambda_{p_{l},j}^{(l)}$ follows a gamma distribution.

Finally, random weights of the form in \cref{eq:scalemixtureW} have been used to initialise the weights in stochastic gradient descent algorithms, departing from the standard iid Gaussian initialisation commonly used for training deep neural networks~\cite{Glorot2010}. \citet{Blier2018} use per-node random learning rates in stochastic gradient descent. This is equivalent to using the prior in \cref{eq:scalemixtureW} at initialisation, and then learning $V_{jk}^{(l+1)}$ while keeping the variances fixed after initialisation. A similar approach was considered by \citet{Wolinski2020a}, but with deterministic variances. %

As outlined above, neural networks with random weights of the form in \cref{eq:scalemixtureW} have been extensively used in practice. A flurry of different distributions have been proposed for the random variance $\lambda^{(l)}_{p_l,j}$, and it is unclear which one we should choose in practice, and how this choice influences the %
properties of the resulting random neural network function.

The objective of this work is to analyse the infinite-width properties of feedforward neural networks with dependent weights of the form in \cref{eq:scalemixtureW}. Our work shows that the choice of the distribution of the per-node variance is crucial and can lead to fundamentally different infinite-width limits. Our main assumption is that, at each hidden layer $l$,
\begin{equation}
        \sum_{j=1}^{p_{l}} \lambda^{(l)}_{p_l,j}\overset{\cvd}{\to} \Lambda^{(l)}~~~\text{ as the width }p_l\to\infty, \label{eq:convergencedist}
\end{equation}
where $\overset{\cvd}{\to}$ refers to convergence in distribution and $\Lambda^{(l)}$ is some nonnegative random variable, which may be constant. This assumption is natural as it implies that the activations and outputs of the neural network are almost surely finite in the infinite-width limit. Note that $\sum_{j=1}^{p_{l}} \var\left(\left.W_{jk}^{(l+1)}\right|(\lambda^{(l)}_{p_l,j})_{j\geq 1}\right )= \sigma_v^2 \sum_{j=1}^{p_{l}} \lambda^{(l)}_{p_l,j}$. Hence, the assumption in \cref{eq:convergencedist} is similar to the commonly made assumption, in the iid case, that the sum of the variances of the incoming weights to a node converges to a constant in the infinite-width limit~\cite{Glorot2010,He2015}.  The iid Gaussian case indeed arises as a special case by setting $\lambda^{(l)}_{p_l,j}=\frac{c}{p_l}$ for all $j=1,\ldots,p_l$ for some $c>0$. Note that $\Lambda^{(l)}=c$ is deterministic in this particular case.

The random variable $\Lambda^{(l)}$ is necessarily infinitely divisible (see \cref{sec:background_infdiv}), and parameterised by
\begin{compactitem}
\item[(i)] a location parameter $a^{(l)}\geq 0$ and
\item[(ii)] a L\'evy measure $\rho^{(l)}$ on $(0,\infty)$.
\end{compactitem}
We prove that, if $a^{(l)}>0$ and the L\'evy measures are trivially zero (that is $\int_0^\infty \rho^{(l)}(dx)=0$) at all hidden layers $l$, then the limit is a Gaussian process (GP), as in the iid Gaussian case. As a consequence, all weights are uniformly small, with $\max_{j=1,\ldots,p_l} |W_{jk}^{(l+1)}|\to 0$ in probability. We show that this GP limit arises with a few models proposed in the literature, such as the group lasso~\cite{Scardapane2017,Wang2017} and inverse gamma~\cite{Ober2021} priors. These neural network models therefore are asymptotically equivalent to a model with iid Gaussian weights in the infinite-width limit.

More interestingly, if at least one of the L\'evy measures is non-trivial, we obtain a very different behaviour, and the limit is now a \textit{mixture of Gaussian processes} (MoGP), with a given random covariance kernel. Under the MoGP regime, we show that the following results hold in the infinite-width limit, none of which hold for the iid Gaussian case.
\begin{compactitem}
\item
$\max_{j=1,\ldots,p_l} |W_{jk}^{(l+1)}|$ converges in probability to a random variable which is not degenerately 0 (see \cref{prop:max weight}). That is, some weights remain non-negligible asymptotically. It is natural to interpret this as being connected to nodes representing important hidden features.
\item
The different dimensions of the output remain dependent (see \cref{th:singleinput,th:multipleinputs}).
\item
The outputs are non-Gaussian, and may exhibit heavy tails depending on the behaviour of the L\'evy measures at infinity  (see
\cref{prop:powerlawactivations,prop:powerlawactivations_nobias}).
\item
        Pruning the network according to the variance parameter $\lambda_{p_l,j}^{(l)}$ at some level $\epsilon>0$ sufficiently small, provides a finite, non-empty neural network with positive probability.\footnote{Note that there is always some small probability of pruning everything and leaving an empty network.} The resulting error associated to the pruned network can be related to the behaviour of the L\'evy measure near~0 (see \cref{thm: eps-pruning}).
\item
The network is compressible: when pruning the network by removing a fixed proportion $(1-\kappa)\in(0,1)$ of nodes at each layer according to the variance parameter $\lambda_{p_l,j}^{(l)}$, the difference between the outputs of the pruned and unpruned networks converges to 0 in probability in the infinite-width limit (see \cref{th:kappapruning}).
\end{compactitem}

\paragraph{Some illustrative examples.} To give a sense of the range of results covered in this article, we now briefly present some illustrative examples in the case of a simple feedforward neural network with one hidden layer, $\din$-dimensional input $\bx=(x_1,\ldots,x_\din)^T$, $2$-dimensional output $(Z^{(2)}_{1}(\bx;\bn),Z^{(2)}_{2}(\bx;\bn))^T$, no bias, $\sigma_v=1$ and rectified linear unit (ReLU) activation function. For $k=1,2$, the output is such that
\begin{align*}
Z^{(2)}_{k}(\bx;p_1) = \sum_{j=1}^{{p}_{1}} \sqrt{\lambda_{p_1,j}^{(1)}} V^{(2)}_{jk} \max\left (0,\frac{1}{\sqrt{\din}}\sum_{i=1}^{\din}V^{(1)}_{ij} x_{i}\right ).
\end{align*}
More general deep neural networks and other examples are considered later in this article. As mentioned above, it is well known (see for instance \cite{Lee2018}) that, if $\lambda_{p_1,j}=\frac{2}{p_1}$ (iid Gaussian weights, or He initialisation~\cite{He2015}), the outputs are asymptotically independent Gaussian processes with, for $k=1,2$,
\begin{equation}\label{eq:exampleGPlimit}
        \left(
        \begin{array}{cc}
                Z^{(2)}_{k}(\bx;p_1) \\
                Z^{(2)}_{k}(\bx';p_1) \\
        \end{array}
        \right)
        \overset{\cvd}{\to}
        \mathcal N \left (0, \left(
        \begin{array}{cc}
                \mathcal K^{(2)}(\bx,\bx) & \mathcal K^{(2)}(\bx,\bx') \\
                \mathcal K^{(2)}(\bx,\bx') & \mathcal K^{(2)}(\bx',\bx') \\
        \end{array} \right) \right )
        \text{ as }
        p_1\to\infty
\end{equation}
where the (deterministic) covariance kernel $\mathcal K^{(2)}(\bx,\bx')$ is defined by
\begin{align}
        \mathcal K^{(2)}(\bx,\bx')&=\frac{\|\bx\| \|\bx'\|}{\din}\times \frac{1}{\pi}\left(\sqrt{1-\rho_{\bx,\bx'}^{2}}+\left(  \frac{\pi}{2}+\arcsin\rho_{\bx,\bx'}\right)  \rho_{\bx,\bx'}\right),\label{eq:relukernelnobias}
\end{align} with correlation  $\rho_{\bx,\bx'}=\frac{\bx^T \bx'}{\|\bx\| \| \bx'\|}$, see \cref{app:relu} for background on ReLU kernels.

\begin{table}\scriptsize
  \centering
  \begin{tabular}{@{}|@{\,}c@{\,}|@{\,}c@{\,}|@{\,}c@{\,}|@{\,}c@{\,}|@{\,}c@{\,}|@{\,}c@{\,}|@{\,}c@{\,}|@{\,}c@{\,}|@{\,}c@{\,}|@{\,}c@{\,}|@{}}
     \hline
     Model & Limit & Depend. &  Distribution & Tail of & Number of&  $\max|W_{jk}^{(2)}|$ & Tail of & $(W^{(2)}_{(j)k})^2$& Compressible\\
      &  process & outputs & of $Z_k^{(2)}(\bx,p_1)$ & $Z_k^{(2)}(\bx,p_1)$ & active nodes &$\cvpto 0$  & of $W_{jk}^{(2)}$ & decrease in&\\
     \hline
     iid & GP & No & Gaussian & Expon. & $\infty$ & Yes & Expon.  & --&No\\
     (a) & GP & No & Gaussian & Expon. &$\infty$ & Yes & Expon. & --&No\\
     (b) & MoGP & Yes & Compound Poisson & Expon. &$\Poi(2)$& No & Expon.& -- &Yes\\
     (c) & MoGP & Yes & Normal-gamma & Expon. &$\infty$ & No &  Expon. & $O(e^{-cj})$&Yes \\
     (d) & MoGP & Yes & Cauchy & Power-law &$\infty$ & No & Power-law & $O(j^{-2})$& Yes\\
     \hline
   \end{tabular}
  \caption{Summary of the properties of the neural network models for four different distributions on the per-node variances.}\label{tab:main_examplesintro}
\end{table}

Consider now the following models for $p_1\geq 2$:
\begin{center}
\begin{tabular}{ll}
  (a) $\lambda^{(1)}_{p_1,j}\sim \IG\left(2, \frac{2}{p_1} \right)$ & (b) $\lambda^{(1)}_{p_1,j}\sim \Ber\left(\frac{2}{p_1} \right)$ \\
        (c) $\lambda^{(1)}_{p_1,j}\sim \betadist\left(\frac{1}{p_1},\frac{1}{2}\right)$ & (d) $\lambda^{(1)}_{p_1,j}=\pi^2\frac{U_j^2}{p^2_1}\text{ where }U_j\sim \Cauchy(0,1)$
\end{tabular}
\end{center}
where $\IG(\beta_1,\beta_2)$ denotes the inverse gamma distribution with shape $\beta_1>0$ and scale $\beta_2>0$, and $\Cauchy(0,1)$ denotes the half-Cauchy distribution with pdf
\begin{align}
        f(u)=\frac{2}{\pi(1+u^{2})} \times \ind_{\{u>0\}}.\label{eq:halfcauchypdf}
\end{align}

For all the above models (a-d), we have $\lambda^{(1)}_{p_1,j}\to 0$ in probability as $p_1\to\infty$. For (a-c), $\bfE[\sum_{j} \lambda^{(1)}_{p_1,j}]\to 2$ as $p_1\to\infty$ (the expectation is infinite for the horseshoe example (d)), as in the iid Gaussian case. However, the infinite-width limits are all very different.

Under the inverse gamma model (a), the infinite-width limit is the same as the iid Gaussian case. Under models (b-d), the infinite-width limit is a mixture of Gaussian processes, i.e. a Gaussian process with a random covariance kernel (see \cref{th:multipleinputs}). These models illustrate some of the benefits of the MoGP regime. The outputs are now dependent in the infinite-width limit. The models (b-d) are compressible in the sense that the difference between the output of the pruned network and the output of the unpruned network  vanishes in the infinite-width limit (see \cref{th:compressibility}). This is not the case for the iid Gaussian model, nor for model (a). The weights as well as the outputs can have an exponential tail (b-c) or a power-law tail (d).  The properties of the different models are summarised in~\cref{tab:main_examplesintro}. More details on these illustrative examples can be found in \cref{app:illustrative_example}.

\paragraph{Organisation of the article.} In~\cref{sec:background_infdiv}, we provide some background material on infinitely divisible random variables. The feedforward neural network model with dependent weights is described in~\cref{sec:model}, together with the asymptotic assumptions. We also show how the behaviour of the L\'evy measure around zero and infinity tunes the properties of large and small weights.  In~\cref{sec: inf-width limit single input}, we give the asymptotic distribution of the outputs for a single input $\bx$, in the case of ReLU-like activation functions. We discuss some of the implications of our result in terms of pruning and heavy tails, depending on the asymptotic properties of the model. In~\cref{sec: inf width limit multi inputs}, the result is extended to multiple inputs $\bx_1,\ldots,\bx_n$ and general activation functions. In~\cref{sec:examples}, we show how many models proposed in the literature can be formulated in our general framework, and present their limiting properties.
  In~\cref{sec:experiments}, we provide some illustrative experiments on Bayesian inference under this class of models, and in~\cref{sec:discussion}, we discuss related approaches. The Appendix contains the details of the illustrative example from above, further examples, most of the proofs, some additional background material and secondary lemmas. The code to reproduce the experiments is available at \href{https://github.com/FadhelA/mogp}{https://github.com/FadhelA/mogp}.

\paragraph{Notations.} For a random variable $X$, $X\sim F$ indicates that $X$ is distributed according to $F$. For functions (or sequences) $a(x)$ and $b(x)$, we use the notation $a(x)\overset{x\to \infty}{\sim}b(x)$ for $\lim_{x\to\infty} a(x)/b(x)=1$. The notation $\cvpto$ and $\cvdto$ respectively mean `convergence in probability' and `convergence in distribution'. We also use the notation $X\overset{\cvd}{=}Y$ to indicate that the two random variables $X$ and $Y$ have the same distribution. For two sequences of random variables $X_n,Y_n$, we write `$X_n\overset{n\to\infty}{\sim}Y_n$ in probability' for $X_n/Y_n \cvpto 1$ as $n\to\infty$.

\section{Background Material on Infinitely Divisible Random Variables}
\label{sec:background_infdiv}

A nonnegative random variable $X$ is said to have an infinitely divisible distribution  if, for every $n\in\mathbb N$, there exist iid nonnegative random variables $Y_{n1},\ldots,Y_{nn}$ such that $X\overset{\cvd}{=}\sum_{i=1}^n Y_{ni}$ \cite{Sato1999levy}.
Examples of infinitely divisible nonnegative distributions are the lognormal,
log-Cauchy, Pareto, gamma, betaprime, constant and positive stable distributions.
(\Cref{app:positive stable} discusses the last positive-stable case in detail.)
 If $X$ is nonnegative and infinitely divisible, its distribution is uniquely characterised by a scalar $a\geq 0$ and a L\'evy measure $\rho$ on $(0,\infty)$ (that is, it is a Borel measure that  satisfies $\int_{0}^\infty \min(1,x)\rho(dx)<\infty$). We write $X\sim\id(a,\rho)$. The scalar $a$ is a location parameter, and $X-a\sim\id(0,\rho)$.
The L\'evy measure $\rho$ may be
\begin{itemize}
\setlength\itemsep{0em}
\item Trivial, that is $\int_0^\infty \rho(dw)=0$; in this case, $X=a$ is constant;
\item Finite, that is $\int_0^\infty \rho(dw)<\infty$; in this case, $X\geq a$, with $\Pr(X=a)>0$;
\item Infinite, that is $\int_0^\infty \rho(dw)=\infty$; in this case, $X=a+Y$, where $Y$ is an absolutely continuous random variable on $(0,\infty)$.
\end{itemize}
The Laplace transform is given, for any $t\geq 0$, by
$\bfE[e^{-tX}]=e^{-ta-\psi(t)}$,
where
$\psi(t):=\int_0^\infty (1-e^{-wt})\rho(dw)$.
Infinitely divisible random variables are closely related to Poisson point processes. The random variable $X\sim\id(a,\rho)$ admits the representation
$X\overset{\cvd}{=}a+\sum_{i\geq 1} \xi_i$,
where $(\xi_i)_{i\geq 1}$ are the points of a Poisson process on $(0,\infty)$ with mean measure $\rho$.

\section{Statistical Model}
\label{sec:model}
\subsection{Feedforward Neural Network}

We consider a feedforward neural network (FFNN) with $L$ hidden layers and $p_l\geq 1$ nodes at each layer $l$. We let $p_0=\din$ be the input dimension and $p_{L+1}=\dout$ be the output dimension. We write $\bn=({p}_1,\ldots, {p}_{L})^T\in\bbN^{L}$. For $ l $ with $ 1\le l\le L+1 $, the pre-activation values at these nodes are given, for an input $ \bx = (x_1,\ldots,x_\din)^T \in \bbR^\din $, recursively by
\begin{equation}
\label{eq:FFNN1}
\begin{aligned}
&Z^{(1)}_{k}(\bx;\bn) = \sum_{j=1}^{\din}W^{(1)}_{jk} x_{j} + B^{(1)}_{k},\\
&Z^{(l)}_{k}(\bx;\bn) = \sum_{j=1}^{{p}_{l-1}} W^{(l)}_{jk} \phi(Z^{(l-1)}_{j}(\bx;\bn)) + B^{(l)}_{k} , \quad l \ge 2,
\end{aligned}
\end{equation}
where $\phi:\mathbb R\to\mathbb R$ is the activation function, $W_{jk}^{(l)}$ is the weight between node $j$ at layer $l-1$ and node $k$ at layer $l$, and $B_k^{(l)}$ is the bias term of node $k$ at layer $l$. The vector $(Z^{(L+1)}_{1}(\bx;\bn),\ldots,Z^{(L+1)}_{\dout}(\bx;\bn))^T$ is
the output of the neural network for the input $\bx$.

Let $\sigma_{b}\geq 0$.  We assume that, for all $k\geq 1$ and $l=1,\ldots,L+1$,
\begin{equation}
B_{k}^{(l)}\overset{\text{iid}}{\sim}\mathcal{N}\left(  0,\sigma_{b}%
^{2}\right)\label{eq:B}
\end{equation}
if $\sigma_{b}> 0$, and $B_{k}^{(l)}=0$ otherwise.

\subsection{Distribution of the Weights}
For $0 \leq l \leq L$, we assume that $W_{jk}^{(l+1)}$ follows a scale mixture
of Gaussian distributions, with
\begin{equation}
W_{jk}^{(l+1)}=\sqrt{\lambda^{(l)}_{p_{l},j}}V_{jk}^{(l+1)}\label{eq:W}
\end{equation}
where
\begin{itemize}\setlength\itemsep{0em}
\item[(a)] for each layer $l=0,\ldots, L$, $j=1,\ldots, {p_{l}}$, and
$k=1,\ldots,{p_{l+1}}$,

\begin{equation}
V_{jk}^{(l+1)}\overset{}{\sim}\mathcal{N}\left(  0,\sigma_v^{2}\right)\
\text{for some $\sigma_v > 0$;} \label{eq:V}
\end{equation}
\item[(b)]
 $\lambda^{(0)}_{p_0,j} = \frac{1}{\din}$,
and for each layer $l=1,\ldots, L$ and each node $j=1,\ldots, {p_l}$ at
layer $l$, $\lambda^{(l)}_{p_{l},j}\geq0$ is a (hidden) node variance parameter,
with
\begin{align}
\lambda^{(l)}_{p_{l},j}\overset{}{\sim}\mu_{{p_l}}^{(l)}\
\text{for some probability distribution $\mu_{{p_l}}^{(l)}$ on $[0,\infty)$;}
\label{eq:lambda}
\end{align}
\item[(c)] all the random variables $\{  \lambda^{(l)}_{p_{l},j},V_{jk}^{(l)}\}_{l,j,k}  $ are assumed to be independent among themselves,
and also with $\{B^{(l)}_k\}_{l,k}$.
\end{itemize}

\subsection{Asymptotic Assumptions and Infinite Divisibility}
As mentioned in the introduction, for any node $k$,
$$\sum_{j=1}^{p_{l}} \var\left(\left.W_{jk}^{(l+1)}\right|(\lambda^{(l)}_{p_l,j})_{j\geq 1}\right )= \sigma_v^2 \sum_{j=1}^{p_{l}} \lambda^{(l)}_{p_l,j}.$$ In order to have a.s. finite activations in the infinite-width limit, we need $\sum_{j=1}^{p_{l}} \lambda^{(l)}_{p_l,j}$ to remain a.s. finite as $p_l$ tends to infinity. To that end, recall from \cref{eq:convergencedist} that
$$
\sum_{j=1}^{p_{l}}\lambda^{(l)}_{p_l,j} \overset{\cvd}{\to}\Lambda^{(l)}
$$
as $p_l\to\infty$, for some nonnegative random variable $\Lambda^{(l)}$. This natural and general assumption, together with the iid assumption, has two consequences.
\begin{itemize}
\item[(i)] By \cite[Theorem 15.12]{Kallenberg2002}, $\Lambda^{(l)}$ is necessarily an infinitely divisible random variable, characterised by a location parameter $a^{(l)}\geq 0$ and a L\'evy measure $\rho^{(l)}$ on $(0,\infty)$. We express this by writing
\begin{align}
\sum_{j=1}^{p_{l}}\lambda^{(l)}_{p_l,j} \overset{\cvd}{\to}\id(a^{(l)},\rho^{(l)}).\label{eq:lambdaconvergenceinfdiv}
\end{align}
\item[(ii)] By \cite[Lemma 15.13]{Kallenberg2002}, we have $\lambda^{(l)}_{p_l,j}\overset{\cvp}{\to} 0$ for any $j\geq 1$.
\end{itemize}
As we will show in the next subsections, the asymptotic properties of the neural network in the infinite-width limit are fully characterised by the activation function $\phi$, the bias variance $\sigma_b^2$, the scaling factor $\sigma^2_v$ and the parameters $(a^{(l)},\rho^{(l)})$ at each hidden layer $l=1,\ldots,L$.

The following result shows that the infinite divisibility of {the sum of} per-node variances implies that the squared $ \ell^{2} $-norm of {the vector of} incoming weights of a node converges in distribution to an infinitely divisible random variable in the infinite-width limit. The proposition follows from~\cref{lemma:idlaplaceproduct} in the Appendix.

\begin{proposition}\label{prop:W2convergenceinfdiv}
Let $l \in \{1,\ldots,L\}$.
Assume \cref{eq:W,eq:V,eq:lambdaconvergenceinfdiv} hold for some $\sigma_v>0$, $a^{(l)}\geq0 $ and some L\'evy measure $\rho^{(l)}$. Then, for any $k\geq 1$,
$$
\frac{1}{\sigma_v^2}\sum_{j=1}^{p_l} (W^{(l+1)}_{jk})^2=\sum_{j=1}^{p_l} \lambda^{(l)}_{p_l,j}\left(\frac{V^{(l+1)}_{jk}}{\sigma_v}\right)^2\cvdto \id(a^{(l)},\nu^{(l)})
$$
where $\nu^{(l)}$ is a L\'evy measure on $(0,\infty)$ defined by
\begin{equation}
\nu^{(l)}(dz)=\int_0^\infty \rho^{(l)}(dz/x)\gammadist(x;1/2,1/2)dx, \label{eq:Levymeasurenu}
\end{equation}
where $\rho^{(l)}(dz/x)$ denotes the measure that assigns $\rho^{(l)}((a/x,b/x))$ to each interval $(a,b) \subseteq \R$.

\end{proposition}

\begin{remark}
In the iid Gaussian case where $\lambda^{(l)}_{p_l,j}=\frac{c}{p_l}$ for some $c$, the sum of variances $\sum_{j=1}^{p_{l}}\lambda^{(l)}_{p_l,j}=c$ is constant and $ \frac{1}{\sigma_v^2} \sum_{j=1}^{p_{l}}(W^{(l+1)}_{jk})^2 =\frac{c}{p_l}\sum_{j=1}^{p_{l}}(V^{(l+1)}_{jk}/\sigma_v)^2 \cvpto c$, where the convergence is
by the law of large numbers.
\end{remark}

\subsection{Properties of the Largest Weights in the Infinite-Width Limit}
\label{sec:powerlawweights}

We discuss here some general structural properties of the FFNN in the infinite-width limit, depending on the parameters $a^{(l)}$ and $\rho^{(l)}$. In particular, we answer the following question: In which cases are the largest variances/weights of the FFNN asymptotically non-negligible?

We interpret a layer $l$ to capture important features in the infinite-width limit if some of the per-node variances remain asymptotically non-negligible as $p_l\to\infty$. The following proposition, which follows from \cite[Theorem 15.29]{Kallenberg2002}, shows that this arises whenever $\rho^{(l)}$ is a non-trivial L\'evy measure.

\begin{proposition}[Necessary and sufficient conditions for uniform convergence to 0]\label{prop:max weight}
Let $l \in \{1,\ldots,L\}$.
The following are equivalent:
\begin{itemize}
\setlength\itemsep{0em}
\item[i)] $\rho^{(l)}$ is trivial;
\item[ii)] $\max_{j} \lambda_{p_l,j}^{(l)}\overset{\cvp}{\to} 0$;
\item[iii)] for every $k\geq 1$, $\max_{j} |W_{jk}^{(l+1)}|\overset{\cvp}{\to} 0$.
\end{itemize}
\end{proposition}

The next proposition goes a bit further and describes the asymptotic distribution of the extreme weights.
For a L\'evy measure $\rho$ on $(0,\infty)$, define the tail L\'evy measure
\[
\overline\rho(x):=\int_{(x,\infty)} \rho(dw)\text{ for all }x>0.
\]
For all $u>0$, let $\overline\rho^{-1}(u):=\inf\{x > 0 : \overline\rho(x) < u\}$ denote the generalised inverse of $\overline\rho$, called the inverse tail L\'evy intensity of $\rho$. Note that both $\overline\rho$ and $\overline\rho^{-1}$  are non-increasing functions, and are both equal to zero if $\rho$ is trivial. The following proposition is a direct corollary of~\cref{prop:extremesID} in the Appendix and of~\cref{prop:W2convergenceinfdiv} in the main text.

\begin{proposition}[Extremes of the variances and weights]\label{prop: extremes}
Consider $l \in \{1,\ldots,L\}$, and let $\lambda^{(l)}_{p_l,(1)}\geq \lambda^{(l)}_{p_l,(2)}\geq \ldots$ be the order statistics of the per-node variances. Then, for any $k\geq 1$, as $p_l\to\infty$,
\begin{align*}
        & \lambda^{(l)}_{p_l,(k)}\overset{\cvp}{\to} 0 \quad \text{if }\rho^{(l)}\text{ is trivial};
        &
        & \lambda^{(l)}_{p_l,(k)}\overset{\cvd}{\to} (\overline\rho^{(l)})^{-1}(G_k) \quad \text{otherwise},
\end{align*}
where $G_k\sim\gammadist(k,1)$. Here $(\overline\rho^{(l)})^{-1}(G_k)$ is a nonnegative random variable, non-degenerate at 0 if the L\'evy measure is non-trivial. Additionally, let $W^{(l+1)}_{(1),m}\geq W^{(l+1)}_{(2),m}\geq \ldots$ be the order statistics of the incoming weights of node $m$ at layer $l+1$. Similarly, we have
\begin{align*}
        & (W^{(l+1)}_{(k),m})^2\overset{\cvp}{\to} 0 \ \ \text{if }\rho^{(l)} \text{(hence }\nu^{(l)})\text{ is trivial};
        & \quad
        & (W^{(l+1)}_{(k),m})^2\overset{\cvd}{\to} \sigma_v^2\times(\overline\nu^{(l)})^{-1}(G_k) \ \ \text{otherwise},
\end{align*}
where $(\overline\nu^{(l)})^{-1}$ is the inverse tail L\'evy intensity of the measure $\nu^{(l)}$ defined in~\cref{eq:Levymeasurenu}.
\end{proposition}

What about the properties of small weights? One answer is given in \cref{app:properties of small weights}.

\subsection{Compressibility of the Neural Network}\label{sec:compression}

About a decade ago, \citet{gribonval2012compressible} established a connection between heavy tails and compressibility in the compressed sensing literature. Recently, a series of works  \citep{arora2018stronger,suzuki2019compression,kuhn2021robustness,suzuki2018spectral}
have shown that the compressibility of a neural network is related to how well the network generalises, both from a theoretical and an empirical point of view.
These two lines of works were brought together by \citet{shin2021compressing,barsbey2021heavy}, who proposed theoretical frameworks to establish a direct connection among the heavy tail index of the distribution of the weights of a neural network, the compressibility of the network and its generalisation properties.
In the setting of our model, these studies on compressibility can be extended from the heavy-tailed case to the much larger class of models for which there is a non-trivial L\'evy measure of the limiting infinitely divisible random variable of \cref{eq:convergencedist}.

Let $v_{p,(1)}\ge v_{p,(2)}\ge \cdots\ge v_{p,(p)}$ be the coordinates, reordered by size, of ${{\bf v}(p)} =(v_{p,1},\ldots, v_{p,p})$. Motivated by similar notions in \cite{gribonval2012compressible}, we say that a sequence $({\bf v}(p))_p$ is $\ell^2$-compressible as $p\to\infty$ if for any $\kappa\in(0,1)$,
\begin{align}\label{def:compressible}
        \lim_{p \to \infty}\frac{\sum_{j=1}^p\ind_{\{v_{p,j} \leq v_{p,(\lfloor\kappa p\rfloor)}\}}v^2_{p,j}}{\sum_{j=1}^p v^2_{p,j}} = 0.
\end{align}
If $v_{p,i}\neq v_{p,j}$ when $i\neq j$, the indicator $\ind_{\{v_{p,j} \leq v_{p,(\lfloor\kappa p\rfloor)}\}}$ retains the top $\kappa$-proportion of $v^2_{p,j}$ values.

To place this in the context of neural networks, we will say that layer $l$ is compressible if \cref{def:compressible} holds in probability for the $\ell^2$-norms of vectors of outgoing weights, for all nodes $j$ in layer $l$. More precisely,
for any $j=1,\ldots,p_l$, denote the squared norm of the outgoing weights of the hidden node $j$ at layer $l$
by
$$ T^{(l+1)}_j:=\|W^{(l+1)}_{j,:}\|^2=\sum_{k=1}^{p_{l+1}}\lambda^{(l)}_{p_l,j} (V_{j,k}^{(l+1)})^2$$ and let $T^{(l+1)}_{(1)}\geq \ldots \geq T^{(l+1)}_{(p_l)}$ denote the ordered values. Then, layer $l$ is $\ell^2$-norm-compressible
if for every $\kappa \in (0,1)$,
	\begin{align}
	\frac{\sum_{j=1}^{p_l} \ind_{\{T^{(l+1)}_j \leq T^{(l+1)}_{(\lfloor \kappa p_l \rfloor)}\}}T^{(l+1)}_{j}}{\sum_{j=1}^{p_l} T^{(l+1)}_{j}}\cvpto 0\label{eq:compressibility_w}\text{ as $p_l\to\infty$.}
\end{align}
In our model, compressible layers are easily characterised simply by the value of $a^{(l)}$ as our next result shows.
\begin{theorem}[Characterisation of compressibility]\label{th:compressibility}
	For each layer $l$ with $1 \leq l \leq L$,
	if $a^{(l)}=0$, then for all $\kappa\in(0,1)$,
	\begin{align}
		\frac
		{\sum_{j=1}^{p_l} \ind_{\{\lambda^{(l)}_{p_l,j} \leq \lambda^{(l)}_{p_l,(\lfloor \kappa p_l\rfloor)}\}}\lambda^{(l)}_{p_l,j}}
		{\sum_{j=1}^{p_l} \lambda^{(l)}_{p_l,j}}\cvpto 0
		\text{ as $p_l\to\infty$,}\label{eq:compressibility_lambda}
	\end{align}
	where $\lambda^{(l)}_{p_l,(1)} \geq \lambda^{(l)}_{p_l,(2)} \geq \ldots \geq \lambda^{(l)}_{p_l,(p_l)}$ are the ordered per-node variance terms.
	In such a case, \cref{eq:compressibility_w} holds so that layer $l$ is $\ell^2$-norm-compressible.
\end{theorem}

\subsection{Heavy Tail and Power-Law Properties of the Variances and Weights}

A random variable $X$ has  a regularly varying tail if $\Pr(X>x)\overset{x\to\infty}{\sim} L(x) x^{-\tau}$ for some power-law exponent $\tau>0$ and some slowly varying function $L$, that is, a function satisfying $L(\gamma x)/L(x)\to 1$ as $x\to\infty$ for all $\gamma>0$. The simplest slowly varying function is the constant function $L(x)=c>0$, and in this case we say that $X$ has a power-law tail; to simplify the presentation, we restrict the presentation to this case here. The next proposition shows that, if the tail L\'evy intensity decays polynomially at infinity, then the extremes of the per-node variance parameters and of the weights have power-law tails asymptotically.

\begin{proposition}[Power law properties of the variances and weights]\label{prop:regvarweightsinf}
Assume that for some $\tau> 0$ and some constant $c >0$,
$\overline\rho^{(l)}(x)\overset{x\to\infty}{\sim}c  x^{-\tau}$.
Then, for any $k,m\geq 1$,
\begin{align}
\lim_{p_l\to\infty} \Pr(\lambda^{(l)}_{p_l,(k)}> x) &\overset{x\to\infty}{\sim} \frac{\overline\rho^{(l)}(x)^k}{k!}\overset{x\to\infty}{\sim} \frac{c ^k }{k!}x^{-k\tau}\label{eq:01} \\
\lim_{p_l\to\infty} \Pr(|W^{(l)}_{(k),m}|> x) &\overset{x\to\infty}{\sim} \frac{\overline\nu^{(l)}(x^2/\sigma_v^2)^k}{k!}\overset{x\to\infty}{\sim}\frac{\left(\frac{2^\tau\Gamma(\tau+1/2)}{\sqrt{\pi}}(\sigma_v^2)^\tau c  \right )^k }{k!}x^{-2k\tau}\label{eq:02}
\end{align}
where $\overline\nu^{(l)}$ is the tail L\'evy intensity of the measure $\nu^{(l)}$ defined in~\cref{eq:Levymeasurenu}.
\end{proposition}

\section{Infinite-Width Limit for a Single Input for Homogeneous Activation Functions}\label{sec: inf-width limit single input}

\begin{definition}
A function $\phi:\mathbb R\to\mathbb R$ is positive homogeneous if and only if
$\phi(\gamma x)=\gamma\phi(x)$ for all $\gamma > 0$ and $x \in \R$.
\end{definition}
The following standard activation functions are positive homogeneous:
\begin{align}
\phi(x)  &  =x & \text{[Linear]}\label{eq:activationlinear}\\
\phi(x)  &  =\max(x,0) & \text{[ReLU]}\label{eq:activationrelu}\\
\phi(x)  &  =\left\{
\begin{array}
[c]{ll}%
x & x>0\\
\beta x & x\leq0
\end{array}
\right.  & \text{[Leaky ReLU]}\label{eq:activationleakyrelu}%
\end{align}
for some $\beta>0$. Note that the $\tanh$ and sigmoid functions are not positive homogeneous. We present later in \cref{th:multipleinputs} more general assumptions that include these two cases.

\subsection{Statement of the Main Theorem}

We consider one FFNN for each $\bn\in\bbN^{L}$.
The following result is stated for positive homogeneous activation functions, which include many important particular cases, in particular the ReLU. A similar result holds under more general assumptions on $\phi$. See~\cref{th:multipleinputs}.

\begin{theorem}[Single input case, ReLU-type activation]\label{th:singleinput}
Consider the feedforward neural network model defined by \cref{eq:FFNN1,eq:B,eq:W,eq:V,eq:lambda}. Assume that the activation function $\phi$ is positive homogeneous and that, for all hidden layers $l=1,\ldots,L$, we have
$$
\sum_{j=1}^{p_l} \lambda_{p_l,j}^{(l)}\overset{\cvd}{\to} \id(a^{(l)},\rho^{(l)})\text{ as } p_l \to \infty
$$
for some $a^{(l)}\geq 0$ and some L\'evy measure $\rho^{(l)}$. Then, as $\min(p_1,\ldots,p_L)\to\infty$, for any $m\geq 1$, any layer $l=1,\ldots,L+1$ and any input $\bx\in \mathbb R^\din$,
\begin{equation}
\left (Z_1^{(l)}(\bx;\bn),\ldots,Z_m^{(l)}(\bx;\bn)\right )\overset{\cvd}{\to}\bfE\left[\underset{k=1,\ldots,m}{\bigotimes} \mathcal N\left(0,\Sigma^{(l)}(\bx)\right)\right].
\end{equation}
Here, for each $\bx\in\mathbb R^{\din}$,   $\left(\Sigma^{(1)}(\bx),\ldots,\Sigma^{(L+1)}(\bx)\right)$ is a Markov sequence of  nonnegative random variables, defined recursively via the following stochastic recurrence equations:
\begin{align}
\Sigma^{(1)}(\bx)&:=\sigma_b^2 + (\sigma_v^2 \| \bx\|^2 /\din)\label{eq:Sigmarecsingle1}\\
\Sigma^{(l)}(\bx)&:=\sigma_b^2 + \sigma_v^2 S^{(l-1)} \Sigma^{(l-1)}(\bx)~~~\text{ for }~l=2,\ldots,L+1,\label{eq:Sigmarecsingle2}
\end{align}
where $S^{(1)},\ldots,S^{(L)}$ are independent random variables which additionally do not depend on the input $\bx$. Moreover, $S^{(l)}\sim \id(c^{(l)},\eta^{(l)})$ where
$\eta^{(l)}$ is a L\'{e}vy measure on $(0,\infty)$ with tail
L\'{e}vy intensity
\[
\overline{\eta}^{(l)}(x):=\int_{\{z:\phi(z)\neq0\}}\overline\rho^{(l)}(x/\phi(z)^{2})  \varphi(z)dz
\]
when $\varphi$ denotes the pdf of the standard normal distribution, and
$c^{(l)}$ is a nonnegative scalar defined by
\[
c^{(l)}:=a^{(l)}\int_{-\infty}^{\infty}\phi(z)^{2}\varphi(z)dz.
\]
\end{theorem}

\begin{example}
Recall that $\overline\nu(x)=\int_0^\infty \overline\rho(x/z)\gammadist(z;1/2,1/2)dz$ is the tail L\'evy intensity of the L\'evy measure in \cref{eq:Levymeasurenu} associated to the sum of the squares of the weights.
For the linear activation function in \cref{eq:activationlinear}, we have
\begin{align}
c^{(l)}=a^{(l)},\qquad\overline\eta^{(l)}(x)=\overline\nu^{(l)}(x).
\end{align}
For the ReLU activation function in \cref{eq:activationrelu}, we have
\begin{align}
c^{(l)}=a^{(l)}/2,\qquad\overline\eta^{(l)}(x)=\overline\nu^{(l)}(x)/2.\label{eq:relulevy}
\end{align}
For the leaky ReLU activation function in \cref{eq:activationleakyrelu}, we have
\begin{align}
c^{(l)}=a^{(l)}(\beta^2+1)/2,\qquad\overline\eta^{(l)}(x)=(\overline\nu^{(l)}(x)+\overline\nu^{(l)}(x/\beta^2))/2.
\end{align}

\end{example}

\subsection{Proof of \cref{th:singleinput}}
Denote $\Sigma^{(1)}(\bx;\bn):=\Sigma^{(1)}(\bx):=\sigma_v^{2}\| \bx\|^{2}/\din+\sigma_{b}^{2}$ and, for each $l=2,\ldots,L+1$,%
\begin{equation}
\Sigma^{(l)}(\bx;\bn):=\sigma_b^2 + \sigma_v^2\sum_{j=1}^{p_{l-1}}\lambda_{p_{l-1},j}^{(l-1)}%
\phi\left(Z_{j}^{(l-1)}(\bx;\bn)\right)^{2} .\label{eq:Rp}
\end{equation}
We have, for all $l=2,\ldots,L+1$,
\begin{align*}
Z^{(1)}_{k}(\bx;\bn)& = \sum_{j=1}^{\din}\frac{1}{\sqrt{\din}}V^{(1)}_{jk} x_{j} + B^{(1)}_{k},
& \ \
Z^{(l)}_{k}(\bx;\bn) &= \sum_{j=1}^{{p}_{l-1}} \sqrt{\lambda^{(l-1)}_{p_{l-1},j}}V^{(l)}_{jk} \phi(Z^{(l-1)}_{j}(\bx;\bn)) + B^{(l)}_{k}.
\end{align*}
Since the $V_{jk}^{(l)}\sim\mathcal N(0,\sigma_b^2)$ are independent among themselves, and also independent from the families $\{\lambda^{(l-1)}_{j},Z^{(l-1)}_{j}(\bx;\bn)\}_j$ and $\{B^{(l)}_{k}\}_k$, we may condition on $ \Sigma^{(l)}(\bx;\bn)$ to obtain, for all $l=1,\ldots,L+1$,%
\[
(Z_{1}^{(l)}(\bx;\bp),\ldots,Z_{p_l}^{(l)}(\bx%
;\bp)) ~\bigr| ~ \Sigma^{(l)}(\bx;\bn)~\overset{\text{iid}%
}{\sim}~\mathcal{N}\left(0,\Sigma^{(l)}(\bx;\bn)\right).
\]
Hence,
\begin{equation}
\label{eq:singleinput:onelayer-representation}
(Z_{1}^{(l)}(\bx;\bp),\ldots,Z_{p_l}^{(l)}(\bx;\bp))
\overset{\cvd}{=}\sqrt{\Sigma^{(l)}(\bx;\bn)}(\varepsilon_{1}^{(l)},\ldots,\varepsilon_{p_l}^{(l)})
\end{equation}
when $\varepsilon_{k}^{(l)}\overset{\text{iid}}{\sim}\mathcal{N}(0,1)$.
By
the positive homogeneity of $\phi$, \cref{eq:Rp} can be rewritten as%
\[
\Sigma^{(l)}(\bx;\bn)\overset{\cvd}{=}\sigma_b^2 + \sigma_v^2 \Sigma^{(l-1)}(\bx;\bn) \sum_{j=1}^{p_{l-1}}\lambda_{p_{l-1},j}^{(l-1)}\phi(\varepsilon
_{j}^{(l-1)})^{2}.%
\]
Let us set
$S^{(l)}(p_l):=\sum_{j=1}^{p_l}\lambda_{p_{l},j}^{(l)}\phi(\varepsilon
_{j}^{(l)})^{2}$ for $l=1,\ldots,L$.
Then, the recurrence relation in \cref{eq:Rp} defines a continuous map $\Psi$ from $[0,\infty)^L$ to $[0,\infty)^{L+1}$ satisfying
\[
(\Sigma^{(1)}(\bx;\bp),\ldots,\Sigma^{(L+1)}(\bx;\bp)) \overset{\cvd}{=} \Psi(S^{(1)}(p_1),\ldots,S^{(L)}(p_{L})).
\]
Note that by the recurrence relation in \cref{eq:Sigmarecsingle2} and its relationship with \cref{eq:Rp},
\[
(\Sigma^{(1)}(\bx),\ldots,\Sigma^{(L+1)}(\bx)) = \Psi(S^{(1)},\ldots,S^{(L)}).
\]
Also, note that the $S^{(l)}(p_l)$ are independent and do not depend on the input $\bx$, and that the random variables
$\{\phi(\varepsilon_{j}^{(l)})^{2}\}_{j=1,\ldots p_{l}}$
are independent. As $|\phi(x)|\leq C_{\mathrm{Lip}}|x|$ for ${C_{\mathrm{Lip}}} := \max(|\phi(1)|, |\phi(-1)|)$, we have $\bfE[\phi(\varepsilon_{j}^{(l)})^{2}]<\infty$. Additionally, $\sum_{j=1}^{p_l}\lambda_{p_l,j}^{(l)}\overset{\cvd}{\rightarrow}\id(a^{(l)},\rho^{(l)})$ for each $l=1,\ldots,L$. It follows from~\cref{lemma:idlaplaceproduct}
in~\cref{app:limit thms} that
\begin{align}
        (S^{(1)}(p_1),\ldots,S^{(L)}(p_L))\overset{\cvd}{\to}\underset{l=1,\ldots,L}{\bigotimes}\id(c^{(l)},\eta^{(l)})
\end{align}
as $\min(p_1,\ldots,p_L)\to\infty$. By the continuous mapping theorem, this implies
\[
\Psi(S^{(1)}(p_1),\ldots,S^{(L)}(p_{L})) \overset{\cvd}{\to} \Psi(S^{(1)},\ldots,S^{(L)}).
\]
Thus, $(\Sigma^{(1)}(\bx;\bp),\ldots,\Sigma^{(L+1)}(\bx;\bp)) \overset{\cvd}{\to} (\Sigma^{(1)}(\bx),\ldots, \Sigma^{(L+1)}(\bx))$.
The final result now follows.

\subsection{Recursion for the Variance of the Limiting Outputs}
\label{sec:singleinputvariance}

Let $(\zeta^{(l)}_1(\bx),\ldots,\zeta^{(l)}_m(\bx))$ be the random variables that
are distributed as $\underset{k=1,\ldots,m}{\bigotimes}\mathcal N(0,\Sigma^{(l)}(\bx))$
when conditioned on $\Sigma^{(l)}$. Note that if
we do not condition on $\Sigma^{(l)}$,
these random variables $(\zeta^{(l)}_1(\bx),\ldots,\zeta^{(l)}_m(\bx))$ have the distribution
$\bfE[\underset{k=1,\ldots,m}{\bigotimes}\mathcal N(0,\Sigma^{(l)}(\bx))]$.
Thus, they are the infinite-width pre-activations/outputs in \cref{th:singleinput}.
Assume that, for any $l$, $M_1^{(l)}:=\int_0^\infty x \rho^{(l)}(dx)<\infty$, and $C_{\phi}:=\bbE [\phi(X)^2]<\infty$, where $X\sim\mathcal N(0,1)$. Then,
\begin{align}
\var(\zeta_k^{(l)}(\bx))=\bbE[\Sigma^{(l)}(\bx)]
\end{align}
where $\bbE[\Sigma^{(l)}(\bx)]$ follows the recursion
\begin{align}
\bbE[\Sigma^{(1)}(\bx)]&:=\sigma_b^2 + (\sigma_v^2 \| \bx\|^2/\din)\label{eq:Sigmarecsingle1_mean}\\
\bbE[\Sigma^{(l)}(\bx)]&:=\sigma_b^2 + \sigma_v^2 C_{\phi} (a^{(l-1)}+M_1^{(l-1)}) \bbE[\Sigma^{(l-1)}(\bx)]~~~\text{ for }~l=2,\ldots,L+1.\label{eq:Sigmarecsingle2_mean}
\end{align}

In the particular cases where $\sigma_b=0$, we obtain the simple expression
\begin{align}
\var(\zeta_k^{(l)}(\bx))=\sigma_v^2 \frac{ \| \bx\|^2}{\din}\prod_{l'=1}^{l-1} \sigma_v^2 C_{\phi}(a^{(l')}+M_1^{(l')}).\label{eq:asympvariance}
\end{align}
In order to avoid the variance of the pre-activations to explode/vanish as the depth increases, the pair $(a^{(l)},\rho^{(l)})$ should be chosen such that $\sigma_v^2 C_{\phi}(a^{(l)}+M_1^{(l)})=1$. In the ReLU case, $C_{\phi}=1/2$, and this reduces, if $\sigma_v=1$, to $a^{(l)}+M_1^{(l)}=2$. This is the configuration of the four examples presented in \cref{sec:intro}.

\subsection{Regularly Varying Properties of the Activations and Outputs}

We derive here results for the ReLU activation function. Similar results can be derived for other homogeneous activation functions. We have already shown in \cref{prop:regvarweightsinf} that if the tail L\'evy intensity decays polynomially at infinity, then the weights have power-law tails. The next proposition shows that if this is the case for all hidden layers, then the activations at each level and the outputs also have regularly varying tails, with an exponent which is twice the minimum of the exponents of the tail L\'evy intensities in the layers below that level.
\begin{proposition}\label{prop:powerlawactivations}
Let $L\geq 1$. Consider the same assumptions as in \cref{th:singleinput}. Also, for $l=1,\ldots,L$, assume that $\overline\rho^{(l)}$ has a power-law behaviour at infinity with exponent $\tau^{(l)}$, that is
\begin{align}
\overline\rho^{(l)}(x)\overset{x\to\infty}{\sim}c^{(l)}x^{-\tau^{(l)}}\label{eq:rhobarRVL}
\end{align}
for some positive constants $c^{(l)}>0$. Then, for any $l=1,\ldots,L$,
\begin{align}
\Pr(S^{(l)}>u)\overset{u\to\infty}{\sim}\overline\nu^{(l)}(u)/2 \overset{u\to\infty}{\sim} \widetilde c^{(l)} u^{-\tau^{(l)}}\label{eq:powerlawSl}
\end{align}
where
\begin{align*}
\overline\nu^{(l)}(u) & =\int_0^\infty \overline\rho^{(l)}(u/z)\gammadist(z;1/2,1/2)dz,
&
\widetilde c^{(l)} & = c^{(l)}\times 2^{(\tau^{(l)}-1)}\Gamma(\tau^{(l)}+1/2)/\sqrt{\pi}.
\end{align*}
Also, for all $k\geq 1$ and $2 \leq l \leq L+1$,
if we let $\zeta_k^{(l)}(\bx) := \varepsilon_k^{(l)}\sqrt{\Sigma^{(l)}(\bx)}$
for $\varepsilon_k^{(l)}\sim\mathcal N(0,1)$, then
\begin{align}
\Pr(\Sigma^{(l)}(\bx)>u)&\overset{u\to\infty}{\sim} u^{-\beta^{(l-1)}}L^{(l-1)}%
(u)\label{eq:corRV2}\\
\Pr((\zeta_{k}^{(l)}(\bx))^2   >u)&\overset{u\to\infty}{\sim}u^{-\beta^{(l-1)}}L^{(l-1)}(u)\times 2^{\beta^{(l-1)}}(\Gamma(\beta^{(l-1)}+1/2)/\Gamma(1/2))
\end{align}
where $\beta^{(l-1)}=\min(\tau^{(1)},\ldots,\tau^{(l-1)})$ and $L^{(l-1)}$ are some slowly varying functions.
\end{proposition}

In general, the slowly varying functions $L^{(l-1)}$ in \cref{prop:powerlawactivations} cannot be obtained analytically. An exception is when the hidden layers have the same asymptotic distribution and there is no bias, as we show in the next proposition.

\begin{proposition}\label{prop:powerlawactivations_nobias}
Let $L\geq 1$. Consider the same assumptions as in \cref{th:singleinput}. Additionally, assume that $\sigma_b=0$, $a^{(l)}=a\geq 0$ and $\rho^{(l)}=\rho$ for all $l=1,\ldots,L$ with
$\overline\rho(x)\overset{x\to\infty}{\sim}c x^{-\tau}$
for some positive constant $c>0$ and exponent $\tau>0$. For $k\geq 1$ and $l=2,\ldots,L+1$,
let $\zeta_k^{(l)}(\bx) := \varepsilon_k^{(l)}\sqrt{\Sigma^{(l)}(\bx)}$
for $\varepsilon_k^{(l)}\sim\mathcal N(0,1)$. Then, for $l=2,\ldots,L+1$,
\begin{align*}
\Pr(\Sigma^{(l)}(\bx)>u)&\overset{u\to\infty}{\sim}\left (\frac{ \| \bx\|^2}{\din}\sigma_v^{2l} \right)^\tau \frac{\tau^{l-2}(\widetilde c)^{l-1}}{(l-2)!}u^{-\tau}\log^{l-2} u\\
\Pr((\zeta_{k}^{(l)})^2   >u)&\overset{u\to\infty}{\sim}  \Pr(\Sigma^{(l)}(\bx)>u) \times (2^{\tau}\Gamma(\tau+1/2)/\Gamma(1/2))
\end{align*}
where $\widetilde c=
c \times (2^{\tau-1}\Gamma(\tau+1/2) / \sqrt{\pi}).$

\end{proposition}
Note that in this particular case, the tails of the activations have the same exponent $2\tau$, but an additional log factor is added for each additional hidden layer after the first one, and so, the tails become slightly heavier as the network gets deeper.

\subsection{Pruning of the Nodes of the Network}

Suppose that we want to prune the nodes of the neural network in order to reduce the computational cost. We consider two different strategies for node pruning, both based on the values of the per-node variances $\lambda_{p_{l},j}^{(l)}$. The first strategy, called {\it $\epsilon$-pruning}, prunes nodes such that $\lambda_{p_{l},j}^{(l)}\le \epsilon$, for some fixed threshold $\epsilon>0$. The second strategy, called {\it $\kappa$-pruning}, prunes nodes such that $\lambda_{p_{l},j}^{(l)}\le \lambda^{(l)}_{p_{l},(\lfloor \kappa p_l \rfloor )}$ where
the subscript $(\lfloor \kappa p_l \rfloor )$ denotes an order statistic:
at layer $l$, $\lambda_{p_{l},(1)}^{(l)}\geq \lambda_{p_{l},(2)}^{(l)}\geq\ldots\geq \lambda_{p_{l},(p_l)}^{(l)}$ denote the ordered values of $(\lambda^{(l)}_{p_{l},j})_{j=1,\ldots,p_l}$.
When there are no repeated values, $\kappa$-pruning is equivalent to pruning
a proportion $(1-\kappa)\in(0,1)$ of the $p_l$ nodes with lowest $\lambda_{p_{l},j}^{(l)}$ values in each layer. The pruning strategies we employ here are related to the compressibility of a network discussed in \cref{sec:compression}. This connection was noted in \citet{barsbey2021heavy} where similar pruning schemes were discussed.

We start with an error bound of the pruned network that holds for both strategies with $ \eps $ and $ \kappa $.
To this end, let $ \lambda^{*(l)}_{p_{l}}, l=1,2,\ldots, $ be nonnegative random variables. Consider the following pruned network:
\begin{equation}\label{eq: pruned network general}
	\begin{aligned}
		&Z^{*(1)}_{k}(\bx;\bp) := Z^{*(1)}_{k}(\bx) := \sum_{j=1}^{\din} \frac{1}{\sqrt{\din}} V^{(1)}_{jk} x_{j} + B^{(1)}_{k}, \\
		&Z^{*(l)}_{k}(\bx;\bp) := \sum_{j=1}^{{p}_{l-1}} \left(\sqrt{\lambda^{(l-1)}_{p_{l-1},j}} \ind_{\{\lambda^{(l-1)}_{p_{l-1},j} > \lambda^{*(l-1)}_{p_{l-1}}\}}\right) V^{(l)}_{jk} \phi(Z^{*(l-1)}_{j}(\bx;\bp)) + B^{(l)}_{k},~~l\geq 2 .
	\end{aligned}
\end{equation}
Namely, we prune a node if its node variance is less than or equal to the threshold $ \lambda^{*(l)}_{p_{l}} $. For $ \eps $-pruning, $ \lambda^{*(l)}_{p_{l}} = \eps $. In this case, we write $ Z^{*(l)}_{k}(\bx;\bp) = Z^{*(l)}_{k}(\bx;\bp,\eps) $ to emphasise the dependence of the network on $ \eps $. On the other hand, for $ \kappa $-pruning, $ \lambda^{*(l)}_{p_{l}} = \lambda^{(l)}_{p_{l},(\lfloor \kappa p_{l} \rfloor)} $. Similarly, we write $ Z^{*(l)}_{k}(\bx;\bp) = Z^{*(l)}_{k}(\bx;\bp,\kappa) $ to emphasise the dependence on $ \kappa $.

Set $N^{(l)}_{p_{l}} := \bfE[ \sum_{j=1}^{p_{l}} \lambda^{(l)}_{p_{l},j}]$.
A key assumption used throughout this subsection on pruning is:
\begin{itemize}
	\item[(UI)] For all layers $ l=1,\ldots,L $,
$$
\int_{0}^{\infty} u \rho^{(l)}(du) = M^{(l)}_{1} < \infty,
\quad
N^{(l)}_{p_l} < \infty \text{ for all $p_l$},
\quad
\text{and} \quad N^{(l)}_{p_l}\to a^{(l)}+M^{(l)}_{1}\text{ as } p_l \to \infty.$$
\end{itemize}
In our setting, the assumption (UI) is equivalent to the uniform integrability of the family $ \{ \sum_{j=1}^{p_{l}}\lambda^{(l)}_{p_{l},j} \}_{p_{l}}$
(see \cref{sec: proof pruning}).

We will also utilise the following assumptions in this subsection:
\begin{itemize}
	\item[(A1)] The activation function $ \phi $ is positive homogeneous.
	\item[(A2)] \cref{eq:lambdaconvergenceinfdiv} holds with $a^{(l)}=0$ for all hidden layers $ l=1,\ldots,L $.
	\item[(A3)] The L\'evy measures of all layers are equal, $ \rho^{(l)} = \rho $, and  $ \rho $ satisfies $ \overline\rho(u) \stackrel{u\to0}{\sim} u^{-\alpha} L(1/u) $ for some $ \alpha \in [0,1) $ and some slowly varying function $ L $. In this case, $ M_{1} := M^{(l)}_{1} $ does not depend on $ l $.
\end{itemize}
The following proposition gives a bound on the error of the above pruned network. The argument is a variant of the variance recursion given in \cref{sec:singleinputvariance}. To state the proposition, recall that ${C_{\mathrm{Lip}}} = \max(|\phi(1)|, |\phi(-1)|)$ and define
$U^{(l)} := \sup_{\bp} \bfE[( Z^{(l)}_{1}(\bx;\bp))^{2}]$.
We point out that $ U^{(l)} < \infty $ under (UI) and (A1); see \cref{lem: l2 norm preactivations} in the Appendix.

\begin{proposition}[Pruning error bound]\label{prop: error bound}
	If (A1) holds, then the $ L^{2} $-error between the pruned and unpruned networks satisfies
	\begin{align}
		&\bfE\left[\left( Z^{(l+1)}_{1}(\bx;\bp) - Z^{*(l+1)}_{1}(\bx;\bp) \right)^{2}\right]
		\label{eq:error bound} \\
		&\le \sigma_{v}^{2}{\Clip}U^{(l)} A^{(l)}_{p_{l}} + (\sigma_{v}^{2}{\Clip})^{2}N^{(l)}_{p_{l}}U^{(l-1)}A^{(l-1)}_{p_{l-1}} + \cdots + (\sigma_{v}^{2}{\Clip})^{l}N^{(l)}_{p_{l}}\cdots N^{(2)}_{p_{2}} U^{(1)}A^{(1)}_{p_{1}}, \notag
	\end{align}
	where $A^{(l)}_{p_{l}} := \bfE[ \sum_{j=1}^{p_{l}} \lambda^{(l)}_{p_{l},j} \ind_{\{ \lambda^{(l)}_{p_{l},j} \le \lambda^{*(l)}_{p_{l}} \}}]$.
\end{proposition}
\begin{remark}\label{rem: simple variance upper bound}
	To get a bound on $U^{(l)}$, note that the variance $ \bfE[( Z^{(l)}_{1}(\bx;\bp))^{2}] $ satisfies a similar recurrence relation to that described in \cref{sec:singleinputvariance}. Namely,
	\begin{align*}
		\bfE\left[ \left( Z^{(l+1)}_{1}(\bx;\bp) \right)^{2} \right]
		\leq \sigma_{v}^{2} C_{\phi}N^{(l)}_{p_{l}} \bfE \left[ \left( Z^{(l)}_{1}(\bx;\bp) \right)^{2} \right] + \sigma_{b}^{2}.
	\end{align*}
	Also, the bound in \cref{eq:error bound} holds when the supremum in $U^{(l)}$ for each $l$ is taken for $\bp'$ with $\min \bp' \geq \min \bp$. See the proofs of \cref{lem: l2 norm preactivations,prop: error bound} for details. In the particular case where $ \sigma_{b}=0$, $\var(\zeta^{(l)}_{1}(\bx)) > 0$, $\min \bp$ is sufficiently large and (A2) holds, if the supremum of $U^{(l)}$ is taken over $\bp'$ with $\min \bp' \geq \min \bp$ for every $l$, then $U^{(l)}$ satisfies
	\begin{align*}
		U^{(l)} \le 2\var(\zeta^{(l)}_{1}(\bx))=2\sigma_{v}^{2}\frac{\|\bx\|^{2}}{\din} \prod_{l'=1}^{l-1} \sigma_{v}^{2}C_{\phi} M^{(l')}_{1} .
	\end{align*}
\end{remark}

\subsubsection{$\epsilon$-Pruning}

Let $ \lambda^{*(l)}_{p_{l}} = \epsilon$ for some $\epsilon > 0$. At layer $l$, this means that we keep the hidden nodes $j$ such that $\lambda_{p_{l},j}^{(l)} > \epsilon$. (We do not let $ \epsilon $, the pruning level, depend on the layer here just to simplify presentation; lifting this restriction would not invalidate our results to be presented next.)
It should be noted that, when the limiting unpruned network is infinite, this pruning strategy produces a finite network.

To analyse the error between the unpruned network and the $ \eps $-pruned network in \cref{eq: pruned network general}, we investigate the limit of the $ L^{2} $ pruning error $\bfE[(Z^{*(l)}_{k}(\bx;\bp,0)- Z^{*(l)}_{k}(\bx;\bp,\epsilon))^2]$ and show that, under assumptions (UI) and (A1-A3), this error remains small in the limit as $ \min(p_{1},\ldots,p_{L}) \to \infty $. This comes as a corollary to \cref{prop: error bound}.

\begin{corollary}[Single input case, $\epsilon$-pruning]\label{thm: eps-pruning}
	Consider pruned FFNNs defined by \cref{eq: pruned network general,eq:B,eq:V,eq:lambda} with $ \lambda^{*(l)}_{p_{l}} = \eps $.
	Suppose (UI) and (A1-A3) hold.
	Then, for all $ \delta\in(0,1-\alpha)$, there exists $ \eps_0(\delta) > 0 $ such that
	if $\eps < \eps_0(\delta)$, we have, for each $ l=1,\ldots,L$ and any $k \ge 1$,
	\begin{align*}
		\lim_{\min\bp \to \infty} \bfE\left[ \left| Z^{(l+1)}_{k}(\bx;\bp)- Z^{*(l+1)}_{k}(\bx;\bp) \right|^{2} \right] \le D(l) \cdot \eps^{1-(\alpha + \delta)} ,
	\end{align*}
	where
	\begin{align*}
		D(l) = \frac{\sigma_{v}^{2}{\Clip}}{1-(\alpha +\delta)}( U^{(l)} + (\sigma_{v}^{2}{\Clip}M_{1})U^{(l-1)} + \cdots + (\sigma_{v}^{2}{\Clip}M_{1})^{l-1} U^{(1)} )
	\end{align*}
	is a constant not depending on $\eps$.
\end{corollary}

{
Although the pruning error is controlled mostly by the pruning level $ \epsilon $, the error can vary according to the constant $D(l)$ which depends on the number of previous layers $ l $. The deeper our network gets, the larger the pruning error becomes. In other words, the pruning error is small at shallow layers, but it accumulates and gets larger at deeper layers.}

In the particular case $\sigma_b=0$ (no bias), combining \cref{thm: eps-pruning} with \cref{rem: simple variance upper bound}, we obtain
\begin{equation}
	\begin{aligned}
		& \lim_{\min\bp \to \infty} \bfE\left[ \left| Z^{(l+1)}_{k}(\bx;\bp)- Z^{*(l+1)}_{k}(\bx;\bp) \right|^{2} \right]
		\\
		& \qquad {}
		\le \left( \frac{\sigma_v^2 {\Clip}}{1-(\alpha + \delta)}\right)
		\cdot \left( \sum_{l'=0}^{l-1}\left(\frac{{\Clip}}{C_{\phi}}\right)^{l'} \right) \cdot 2\var(\zeta^{(l)}_{1}(\bx)) \cdot \epsilon^{1-(\alpha + \delta)}
	\end{aligned}
	\notag
\end{equation}
where $C_{\phi}$ is as in \cref{sec:singleinputvariance}.

{
\begin{remark}\label{rem: trade-off pruning}
	In \cref{sec: proof pruning}, we prove \cref{thm: eps-pruning} in a slightly more general setting where we allow for different $ \rho^{(l)} $'s in different layers. The trade-off is that, if we confine $ \rho^{(l)} = \rho $ for some $ \rho $ as in (A3), then $ \eps_{0}(\delta)$ depends only on $ \delta $ and not on $ L $, thus one can possibly add more layers after $ L+1 $. On the contrary, if we allow for different $ \rho^{(l)} $'s as in the proof, then $ \eps_{0}(\delta,L)$ depends not only on $ \delta $ but also on $ L $, so adding more layers requires changing $ \eps_{0}$.
\end{remark}}

\subsubsection{$\kappa$-Pruning}

For fixed $\kappa\in(0,1)$, let $ \lambda^{*(l)}_{p_{l}} = \lambda^{(l)}_{p_{l},(\lfloor \kappa p_l \rfloor )} $. That is, $\kappa$-pruning discards nodes $j$ at layer $l$ with $\lambda^{(l)}_{p_{l},j} \le \lambda^{(l)}_{p_{l},(\lfloor \kappa p_l \rfloor )}$.

The next result shows that, under assumptions (UI) and (A1-A2) including the compressibility of layers (A2; see \cref{sec:compression}), the error between the unpruned output and the $\kappa$-pruned output in  \cref{eq: pruned network general}  converges to 0, no matter what the value $\kappa\in(0,1)$ is. Again, this comes as a corollary to \cref{prop: error bound}.

\begin{corollary}[Single input case, $\kappa$-pruning]\label{th:kappapruning}
	Consider pruned FFNNs defined by \cref{eq: pruned network general,eq:B,eq:V,eq:lambda} with $ \lambda^{*(l)}_{p_{l}} = \lambda^{(l)}_{p_{l},(\lfloor \kappa p_l \rfloor )} $.
	Suppose (UI) and (A1-A2) hold.
	Then, for each $ l=1,\ldots,L $ and for any $ \kappa\in(0,1)$ and any $k\geq 1$, %
	\begin{align*}
	\bfE\left[ \left| Z^{(l+1)}_{k}(\bx;\bp)-Z^{*(l+1)}_{k}(\bx;\bp)\right|^{2} \right]\to 0\text{ as }\min\bp\to\infty.\label{eq:limiterrorneuralnet}
	\end{align*}
\end{corollary}

This result states that, if $a^{(l)}=0$ for all $ l=1,\ldots,L $, the neural network is compressible: the difference between the output of the $\kappa$-pruned network and that of the unpruned network vanishes in probability as the width of the network goes to infinity.
This is not generally the case if $a^{(l)}>0$. If, in addition, almost surely no node variances are repeated (so
$\kappa$-pruning prunes a $(1-\kappa)$-proportion of nodes), we do not obtain the vanishing error, which occurs when  $a^{(l)}=0$.
For instance, consider a network with one hidden layer. Then, the $ L^{2} $-error is
\begin{align*}
	\bfE\left[ \left| Z^{(2)}_{1}(\bx;\bp)-Z^{*(2)}_{1}(\bx;\bp)\right|^{2} \right]
	= \sigma_{v}^{2} \bfE\left[ \sum_{j=1}^{p_{1}} \lambda^{(1)}_{p_{1},j} \ind_{\{ \lambda^{(1)}_{p_{1},j} \le \lambda^{(1)}_{p_{1},(\lfloor \kappa p_1 \rfloor)} \}} \right] \bfE\left[ \phi^{2}(Z^{(1)}(\bx)) \right]
\end{align*}
which is not guaranteed to converge to 0 for all $\kappa\in(0,1)$ when $a>0$. See \cref{prop:compressibilityID}.

In the iid Gaussian case, our $\kappa$-pruning strategy prunes every node due to the repeated node variance $\frac{c_1}{p_l}$, so that the pruning error trivially does not vanish. In practice, one prunes the iid Gaussian case by removing nodes using instead
$$ T^{(l)}_{p, j}:=\|W_{j,:}\|^2 = \lambda^{(l)}_{p,j} \sum\limits_{k=1}^{p_{l+1}} (V^{(l+1)}_{j,k})^2 $$
Denote by $Z^{**}$, the network defined in a similar way to \cref{eq: pruned network general} but where $T_{p,j}$ is used for pruning instead of $\lambda_{p,j}$. Then, it can be shown that in the iid Gaussian case, the error is non-vanishing, i.e.
$$ \limsup \bfE\left[ \left| Z^{(l+1)}_{k}(\bx;\bp)-Z^{*(l+1)}_{k}(\bx;\bp)\right|^{2} \right] > 0.$$

\section{Infinite-Width Limit for Multiple Inputs in the General Case}\label{sec: inf width limit multi inputs}

We prove the convergence theorem for multiple inputs under a more general assumption for the activation function $\phi$ than the positive homogeneity assumption in~\cref{th:singleinput}. Any positive homogeneous function such as ReLU satisfies this generalisation, as well as the classical $\tanh$ or sigmoid functions. This assumption is called a ``polynomial envelope'' condition by \cite{Matthews2018}, and commonly used in the context of analysing infinitely-wide neural networks either implicitly or explicitly. In \cite{Neal1996,Lee2018}, the authors are implicitly exploiting this assumption by considering $ \tanh $ and ReLU mainly, and in \cite{Favaro2020,Jung2021}, they explicitly considered a weaker version of this assumption.

\begin{theorem}[Multi-input case]\label{th:multipleinputs}
Consider the feedforward neural network model defined by \cref{eq:FFNN1,eq:B,eq:W,eq:V,eq:lambda}. Assume that the activation function $\phi$ is continuous and satisfies the so-called polynomial envelope condition: for all $z\in\mathbb R$, $|\phi(z)|\leq A+B|z|^C$ for some $A,B,C>0$. Assume that, for all hidden layers $l=1,\ldots,L$, we have
$$
        \sum_{j=1}^{p_l} \lambda_{p_l,j}^{(l)}\overset{\cvd}{\to} \id(a^{(l)},\rho^{(l)})\text{ as } p_l \to \infty
$$
for some $a^{(l)}\geq 0$ and some L\'evy measure $\rho^{(l)}$. Let $\mathbf{x}_{1},\ldots,\mathbf{x}_{n}$ be $n$ inputs, where
$\mathbf{x}_{i}\in\mathbb{R}^{\din}$.\ Define
$\vec{Z}_{k}^{(l)}(\mathbf{x}_{1},\ldots,\mathbf{x}_{n};\mathbf{p}):=(Z^{(l)}_{k}(\mathbf{x}_{1};\mathbf{p}),\ldots,Z^{(l)}_{k}(\mathbf{x}_{n};\mathbf{p}))^{T}\in\mathbb{R}%
^{n}$, the associated $k$-th outputs. Then,
for all $l=1,\ldots,L+1$ and all $m \geq 1$,
as $\bn\to\infty$ in the order $\lim_{p_L\to\infty}\ldots\lim_{p_1\to\infty}$,
\[
\left(  \vec{Z}_{k}^{(l)}(\mathbf{x}_{1},\ldots,\mathbf{x}_{n};\mathbf{p})\right)_{k=1,\ldots,m}\overset{\cvd}{\rightarrow}%
\bfE\left[\underset{k=1,\ldots,m}{\bigotimes}\mathcal{N}(0,{}{\Sigma}^{(l)})\right].
\]
Here, ${}{\Sigma}^{(l)}$ is a
random $n$-by-$n$ positive semi-definite matrix defined by ${}{\Sigma}%
_{ij}^{(l)}=K^{(l)}(\bx_{i},\bx_{j})$, for $1\leq i,j\leq n$, where $K^{(l)}:\mathbb{R}^\din\times
\mathbb{R}^\din\rightarrow\mathbb{R}$ is a random covariance kernel. The
sequence of random kernels $(K^{(1)},\ldots,K^{(L+1)})$ is a Markov sequence whose distribution  can be defined recursively,  for  $l=1,\ldots,L,$ by:
\begin{align}\label{eq: multi input kernel inductive def}
K^{(1)}(\mathbf{x},\mathbf{x}^{\prime})&:=\sigma_{b}^{2}+\sigma_{v}%
^{2}\frac{\mathbf{x}^{T}\mathbf{x}^{\prime}}{\din}\\
K^{(l+1)}(\mathbf{x},\mathbf{x}^{\prime})&:=\sigma_{b}^{2}+\sigma_{v}^{2}%
a^{(l)}\bfE\left[\left.\phi(\zeta_{1}^{(l)}(\mathbf{x}))\phi(\zeta_{1}%
^{(l)}(\mathbf{x}^{\prime}))\right| K^{(l)}\right]  +\sigma_{v}^{2}\sum_{j\geq1}^{{}%
}\widetilde{\lambda}_{j}^{(l)}\phi\left(  \zeta_{j}^{(l)}(\mathbf{x})\right)
\phi\left(  \zeta_{j}^{(l)}(\mathbf{x}^{\prime})\right) \notag
\end{align}
where $\{\widetilde{\lambda}_{j}^{(l)}\}_{j\geq1}$ are the points of a Poisson
point process on $(0,\infty)$ with mean measure $\rho^{(l)}$ and, for $j\geq
1$,%
\[
\zeta_{j}^{(l)}\mid K^{(l)}\overset{\text{iid}}{\sim}\GP(0,K^{(l)}).
\]
Here $\GP(\mu,K)$ denotes a Gaussian process on $\R^{\din}$, i.e., a random element of $\mathcal M=\{f:\R^\din\to\R \}$, with mean $\mu\in \mathcal M$ and covariance function $K:\R^\din\times\R^\din \to\mathbb R$.
\end{theorem}

\begin{remark}\label{rem:sequentiallimit}
	The limit in the above theorem is taken in sequential order from the first layer to the last layer. Extending the theorem to a different and more natural limiting scheme, such as  $\min(p_1,\ldots,p_L)\to\infty$, is non-trivial. For instance, although the proof of \Cref{th:singleinput} handles the case $\min(p_1,\ldots,p_L)\to\infty$, it heavily relies on positive homogeneity of the activation function so as to rephrase the outputs of hidden nodes in some layer $l$ as a vector of independent Gaussian random variables that is scaled by a random scalar (\cref{eq:singleinput:onelayer-representation}). Since the positive homogeneity of $\phi$ does not let us move a matrix $M$ from $\phi(Mv)$ to the outside in any form, It is difficult to obtain an analogous result in the case of multiple inputs. We expect that a different approach, such as the use of exchangeability \cite{Favaro2020,Matthews2018}, is needed for such extension of our result, and we leave this as one of the remaining future challenges.
	\end{remark}

When $\rho^{(l)}$ is trivial for all $l=1,\ldots,L$, the kernels are
deterministic, and one recovers a Gaussian process. Otherwise, we obtain a mixture of Gaussian processes, where the mixture comes from the randomness of the kernel $K^{(l)}$. We now discuss some of the properties of the random kernel.

The following proposition is an immediate consequence of the Campbell theorem for Poisson random measures, together with results regarding the ReLU activation function \cite{Cho2009}; see \cref{app:relu}.

\begin{proposition}[Conditional mean and variance of the kernel]
For any $l\geq 1$ and $n\geq 1$, let $M_n^{(l)}=\int_0^\infty x^n\rho^{(l)}(dx)$.
We have
\begin{align*}
        \bfE\left[\left.K^{(l+1)}(\mathbf{x},\mathbf{x}^{\prime})  \right| K^{(l)}\right]
        & = \sigma_{b}^{2}+\sigma_v^{2}(M_{1}^{(l)}+a^{(l)})
        \bfE\left[\left.\phi(\zeta_{1}^{(l)}(\bx))\phi (\zeta_{1}^{(l)}(\bx^{\prime})) \right| K^{(l)}\right]
        \\
        \var\left[\left. K^{(l+1)}(\bx,\bx^{\prime})\right| K^{(l)}\right]
        & = \sigma_v^{4}
                M_{2}^{(l)}\bfE\left[\left.\phi(\zeta_{1}^{(l)}(\bx))^{2}\phi(\zeta_{1}^{(l)}(\bx^{\prime}))^{2} \right| K^{(l)}\right]
\end{align*}
where
\begin{equation}
        \left.\left(
        \begin{array}{c}
                \zeta_{1}^{(l)}(\bx) \\
                \zeta_{1}^{(l)}(\bx')
        \end{array}
        \right) \right| K^{(l)}
        \overset{\text{iid}}{\sim}
        \mathcal N\left ( 0, \left(
        \begin{array}{cc}
                K^{(l)}(\bx,\bx) & K^{(l)}(\bx,\bx') \\
                K^{(l)}(\bx,\bx') & K^{(l)}(\bx',\bx')
        \end{array}\right)\right ).
\end{equation}
In the ReLU case, we have the analytic expressions
\begin{align*}
        \bfE\left[\left.K^{(l+1)}(\mathbf{x},\mathbf{x}^{\prime}) \right| K^{(l)}\right]
        & =
        \sigma_{b}^{2}
        + \sigma_v^{2}(M_{1}^{(l)}+a^{(l)})\frac{\sqrt{K^{(l)}(\bx,\bx)K^{(l)}(\bx',\bx')}}{2\pi} \kappa_1(\rho^{(l)}_{\bx,\bx'})
        \\
        \var\left[\left. K^{(l+1)}(\bx,\bx^{\prime}) \right| K^{(l)}\right]
        & =
        \sigma_v^{4}M_{2}^{(l)}\frac{K^{(l)}(\bx,\bx)K^{(l)}(\bx',\bx')}{2\pi}\kappa_2(\rho^{(l)}_{\bx,\bx'})
\end{align*}
where   $\rho^{(l)}_{\bx,\bx'} = K^{(l)}(\bx,\bx')/ \sqrt{K^{(l)}(\bx,\bx)K^{(l)}(\bx',\bx')}$ and
\begin{align}
        \label{eq:kappa_n}
        \kappa_n(\rho) & =
        \begin{cases}
                \sqrt{1-\rho^{2}}+\left(  \frac{\pi}{2}+\arcsin\rho\right)  \rho
                & \text{if } n=1
                \\
                3\sqrt{1-\rho^{2}}\rho+\left(  \frac{\pi}{2}+\arcsin\rho\right)  (1+2\rho ^{2})
                & \text{if } n=2.
        \end{cases}
\end{align}

\end{proposition}

\begin{figure}
  \centering
  \subfigure[$\beta=1$]{\includegraphics[width=.3\textwidth]{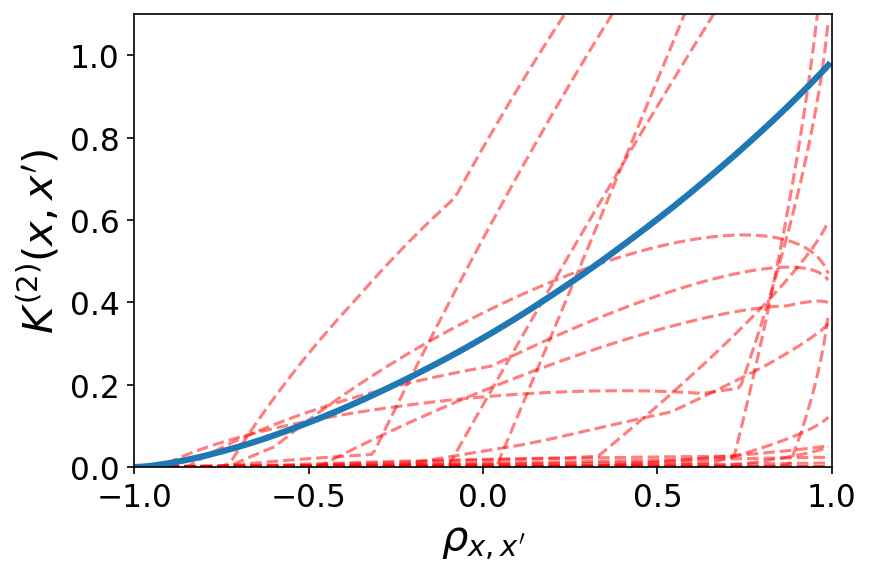}}
  \subfigure[$\beta=10$]{\includegraphics[width=.3\textwidth]{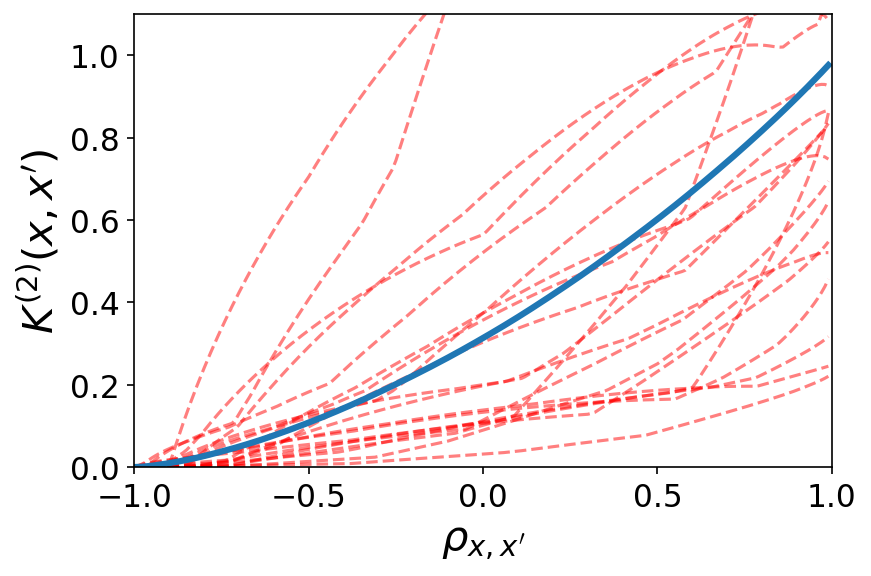}}
  \subfigure[$\beta=1000$]{\includegraphics[width=.3\textwidth]{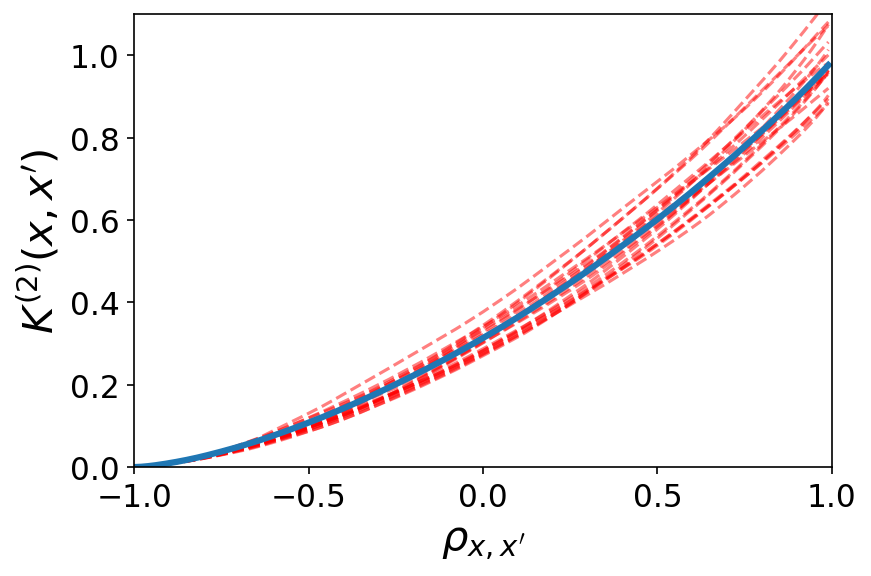}}\\
  \subfigure[$\rho_{\bx,\bx'}=0$]{\includegraphics[width=.4\textwidth]{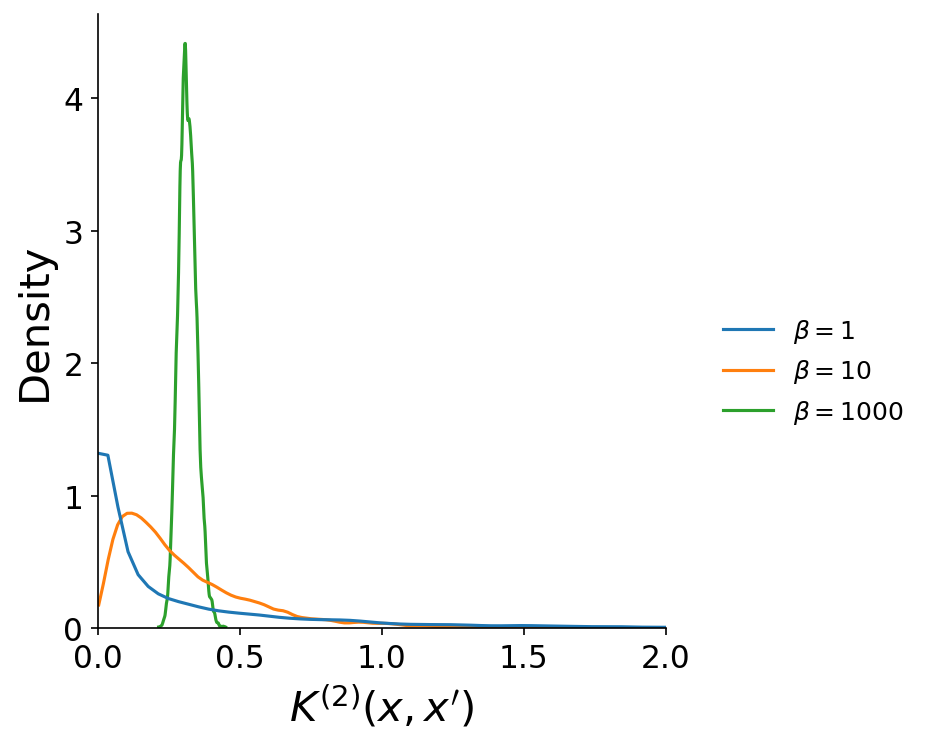}}
  \subfigure[$\rho_{\bx,\bx'}=0.5$]{\includegraphics[width=.4\textwidth]{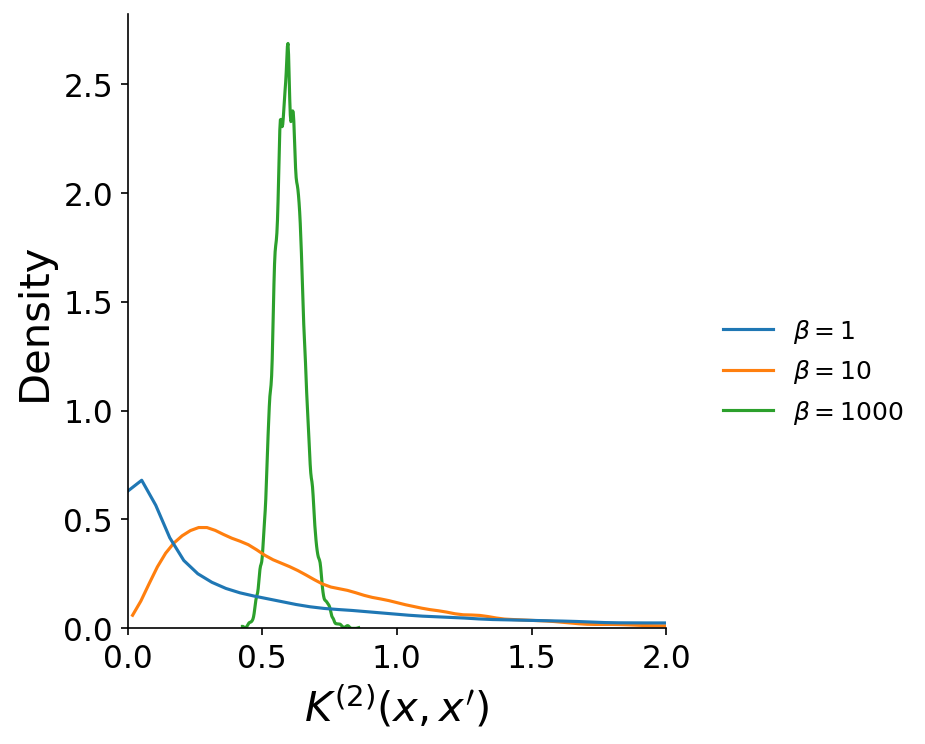}}
  \caption{(a-c) Dashed red lines represent 20 realisations of the kernel $K^{(2)}(\bx,\bx^{\prime})$, as a function of the correlation $\rho_{\bx,\bx'}=\frac{\bx^T\bx'}{\|\bx\|\|\bx'\|}$, when  $\|\bx\| \|\bx'\|/\din=1$, $\sigma_v=1$, $\sigma_b=0$, for the beta model in \cref{example:betamodel} with (a) $\beta=1$, (b) $\beta=10$ and (c) $\beta=1000$. The solid blue line represents the GP ReLU kernel in \cref{eq:relukernelnobias}. The random kernels $K^{(2)}$ are centred on the GP ReLU kernel, and the variance decreases with the tuning parameter $\beta$. (d-e) Distribution of $K^{(2)}(\bx,\bx')$ for different values of $\beta$, for (d) $\rho_{\bx,\bx'}=0$ and (e) $\rho_{\bx,\bx'}=0.5$.}\label{fig:betakernels}
\end{figure}
\begin{example}
\label{example:betamodel}
Assume that $\sigma_v=1$ and $\sigma_b=0$. Consider the model
$\lambda_{p,j}^{(l)}\sim\betadist(\beta/p,\beta/2)$ for some $\beta>0$.  This generalises the example (c) introduced in \cref{sec:intro}, with an additional parameter $\beta>0$. As will be shown later in \cref{sec:betaexample}, $\sum_{j} \lambda_{p,j}^{(l)}$ converges in distribution to a random variable $\Lambda^{(l)}\sim\id(0,\rho)$ where
$\rho(dx)=\beta x^{-1}(1-x)^{(\beta/2)-1}\ind_{\{x \in (0,1)\}}dx$. This is a beta L\'evy measure, with moments $M_k=\beta \frac{\Gamma(k)\Gamma(\beta/2)}{\Gamma(k+\beta/2)}$, so that $M_1=\bfE[\Lambda^{(l)}]=2$ and $M_2=\var(\Lambda^{(l)})=4/(2+\beta)$. It follows that
\begin{align*}
\bfE\left[K^{(2)}(\mathbf{x},\mathbf{x}^{\prime}) \right]
        & =\mathcal K^{(2)}(\mathbf{x},\mathbf{x}^{\prime}),
        &
        \var\left[ K^{(2)}(\bx,\bx^{\prime})\right]&=\frac{2}{\pi(2+\beta)}\frac{\|\bx\|^2\|\bx^{\prime}\|^2}{\din^2}\kappa_2(\rho^{(1)}_{\bx,\bx'}),
\end{align*}
where $\mathcal K^{(2)}$ is the GP ReLU kernel given in \cref{eq:relukernelnobias}. Thus, the random kernel $K^{(2)}$ is centred on $\mathcal K^{(2)}$, and the parameter $\beta$ controls the variance of the kernel.
Realisations of the kernel $K^{(2)}$ for different values of $\beta$ and with $\|\bx\| \|\bx'\|/\din=1$ are given in \cref{fig:betakernels}.
\end{example}

In \cref{app:multiple-input-results-stable-case},
we further discuss a special case of \cref{th:multipleinputs} when the limiting infinitely divisible distribution of $\sum_{j=1}^{p_l} \lambda_{p_l,j}^{(l)}$ is an $\alpha$-stable distribution.

\section{Examples}
\label{sec:examples}

In this section,  we provide examples of models used in the literature, and the associated parameters of the limiting infinitely divisible random variable of \cref{eq:convergencedist}. In some cases, we use a different scaling so that the limit exists, and is not degenerate at 0. \Cref{table:examples} summarises the properties of these models.
Further discussions on these and additional example models can be found in \cref{app:examples}.
To simplify notation, we often drop the layer index $l$ fully or partially in the rest of this section, writing e.g. $\lambda_{p,j}\sim\mu_p$.

\begin{table}
\begin{adjustwidth}{-5mm}{}
\scriptsize
  \begin{tabular}{@{}|@{\,}c@{\,}|@{\,}c@{\,}|@{\,}c@{\,}|@{\,}c@{\,}|@{\,}c@{\,}|@{\,}c@{\,}|@{\,}c@{\,}|@{\,}c@{\,}|@{\,}c@{\,}|@{}}
  \hline
        Name & Mixture's name & $\mu_p$ & $a$ &  L\'evy measure & Support & Finite? & Exp. $\alpha$ & Exp. $\tau$ \\
  \hline
  Determ. & Gaussian & $\delta_{c_1/p} $& $c_1$ & $0$ & -- & -- & -- & -- \\
  Bernoulli & Spike and Slab & $\left(1-\frac{c}{p}\right) \cdot \delta_0 + \frac{c}{p} \delta_1$ & $0$ & $c\delta_1$ & $\{1\}$& Yes & 0 & -- \\
        Gamma & Group lasso & $\gammadist\left (\frac{p_{l+1}+1}{2},\frac{p_l(p_{l+1}+{1})}{2{c_1}} \right)$ & $c_1$ & 0 & -- &-- & -- & --  \\
  Beta & Normal-beta & $\betadist\left(\frac{1}{p},\frac{1}{c}\right)$ &0&$x^{-1}(1-x)^{1/c-1}$&(0,1)&No& -- & --\\
  Inv.-Gamma & Multivariate t & $\IG\left(2, 2/p\right)$ & 2 & 0 & -- & --& --  & -- \\
  Beta prime & Horseshoe & $\frac{2p}{\pi^2}x^{-1/2}(1+\frac{4xp^2}{\pi^2})^{-1}$ & 0 & $\frac{1}{2}x^{-3/2}$ &$(0,\infty)$ & No & 1/2 & 1/2  \\
  Gen. BFRY &
  \begin{tabular}{@{}l@{}}
  Normal
  \\
  \ \ {}-gen. BFRY
  \end{tabular}
  & See \cref{eq:BFRYmodel}& 0 & $\frac{\eta x^{-1-\tau}}{\Gamma(1-\alpha)}\gamma(\tau-\alpha,x)$ & $(0,\infty)$ &No & $\alpha\in(0,1)$ & $\tau>\alpha$ \\
  \hline
\end{tabular}
\end{adjustwidth}
\caption{List of models and their limiting location parameter and L\'evy measure.
\label{table:examples}}
\end{table}

\subsection{Constant Variance (iid Gaussian/Weight Decay/L2 Regularisation)}
\label{ex:gaussian case}

The standard iid Gaussian model is obtained as a special case when $\lambda_{p,j}\sim\delta_{c_1/{p}}$
for some constant $c_1>0$ and so the weights $W_{jk}$ are iid $\mathcal{N}(0,(c_1{\sigma_v^{2}})/p)$.
In this case, $\sum_j {\lambda_{p,j}}=c_1$, so that $\sum_j{\lambda_{p,j}}\overset{\cvd}{\to}\id(c_1,0)$. The weights (and variances) converge uniformly to $0$, i.e. for any $k\geq 1$, $\max_{j=1,\ldots,p}(|W_{jk}|)\overset{\cvp}{\to} 0$.

\subsection{Bernoulli Prior}
\label{sec:bernoulliprior}
For some $c>0$, consider $\lambda^{(l)}_{p_l,j}\sim\Ber(c/p_l)$ for every $p_l\geq c$.
This corresponds to a marginal spike and slab distribution for $W^{(l)}_j=(W^{(l)}_{j1},\ldots,W^{(l)}_{jp_{l+1}})$, with
$$
W^{(l)}_j\sim \left(1-\frac{c}{p_l}\right) \cdot \delta_{0} + \frac{c}{p_l} \cdot \mathcal N(0,\sigma_v^2 I_{p_{l+1}}).
$$
Such a prior has been used by~\citet{Jantre2021} for pruning Bayesian neural networks. In that case,
$\sum_{j}\lambda^{(l)}_{p_l,j}\cvdto \id(0,c\delta_1)$.
That is, the location parameter $a$ is zero, and the L\'evy measure $\rho=c\delta_1$ is finite and discrete.

\subsection{Group Lasso Prior}
We consider that\footnote{Note that in this case, $ \lambda^{(l)}_{p_{l},j}$ depends on the size $p_{l+1}$ of the upper layer as well. However, we show here that, for a specific choice of $b_{p_l}$, at the infinite-width limit with respect to $p_l$, this dependency on $p_{l+1}$ disappears. For clarity, we keep the superscript/subscript $l$ in this subsection.}
$\lambda^{(l)}_{p_{l},j}\sim \gammadist((p_{l+1}+1)/2,\,b_{p_l}/2)$,
where $b_{p_l}$ is an inverse-scale parameter that depends on the layer's width. Such a distribution leads to the so-called group lasso distribution~\cite{Raman2009,Casella2010} over the weights $(W_{jk}^{(l+1)})_{jk}$, which have joint marginal density
\begin{equation}
        f(w)\propto \exp\left(-\frac{\sqrt{b_{p_l}}}{\sigma_v} \sum_{j=1}^{p_{l}} \sqrt{\sum_{k=1}^{p_{l+1}}w_{jk}^2}\right).
        \label{eq:grouplassoprior}
\end{equation}
The regularisation term
\begin{equation}
        -\log f(w)=\left(\frac{\sqrt{b_{p_l}}}{\sigma_v} \sum_{j=1}^{p_{l}} \sqrt{\sum_{k=1}^{p_{l+1}}w_{jk}^2 }\,\right) + {C}
        \label{eq:grouplassopenalty}
\end{equation}
is known as the group lasso penalty, introduced by \citet{Yuan2006} for regression models. This penalty has been used as a regulariser for neural networks by \citet{Scardapane2017} and \citet{Wang2017}. The group lasso distribution in \cref{eq:grouplassoprior} has been used as a sparsity-promoting prior in Bayesian learning of sparse neural networks by~\citet{Jong2018}.

\citet{Scardapane2017} suggested to set $b_{p_l}=p_l$.
However, this assumption implies that $\lim_{p_l\to\infty }\sum_j {\lambda^{(l)}_{p_{l},j}}=p_{l+1}+1$ almost surely, which diverges if $p_{l+1}\to\infty$. This fact has been noted by \citet{Wolinski2020} who suggested the different scaling $b_{p_l}=p_l(p_{l+1}+1)/c_1$ (with $c_1\sigma_v = 1$).
Setting $b_{p_l}=p_l(p_{l+1}+1)/c_1$, we obtain $\sum_j \lambda^{(l)}_{p_{l},j} \cvpto c_1$ as $p_l\to\infty$. Thus,$$\sum_j \lambda^{(l)}_{p_{l},j}\cvdto \id(c_1,0).$$

\subsection{Inverse Gamma Prior and Similar Models}
\label{sec:ex:inversegamma}

We consider here, as in~\cite{Ober2021}, that the variances follow an inverse gamma distribution
\begin{equation}
        \lambda_{p,j}\sim \IG(2, 2/p).\label{eq:inversegammamodel}
\end{equation}
Note that this is equivalent to
\begin{align}
\lambda_{p,j}=Y_j/p\label{eq: Yoverp}
\end{align}
where  $Y_1,Y_2,\ldots,$ are iid $\IG(2,2)$. By the law of large numbers, $\sum_j\lambda_{p,j}\cvpto 2$ or equivalently, $\sum_j\lambda_{p,j}\cvdto \id(2,0)$.
More generally, any model of the form in \cref{eq: Yoverp} where $Y_1,Y_2,\ldots$ are iid random variables with finite mean, satisfies $\sum_j\lambda_{p,j}\cvdto \id(\bfE[Y_1],0)$.

\subsection{Beta Model and Beta L\'evy Measure}
\label{sec:betaexample}
Consider
$\lambda_{p,j}\sim \betadist(\eta/p,b)$
where $\eta,b>0$.
An application of~\cref{th:convergenceid} in \cref{app:limit thms} yields
$\sum_{j}\lambda_{p,j}\cvdto \id(0,\rho)$,
where
$\rho(dx)=\eta x^{-1}(1-x)^{b-1}\ind_{\{x\in(0,1)\}}dx$\,
is a Beta L\'evy measure~\cite{Hjort1990}. The measure is infinite with bounded support.

\subsection{Horseshoe Model}\label{sec:horseshoe}

In the horseshoe model \cite{Carvalho2010}, we assume
the independent random variables $Y_1,Y_2,\ldots$ that have the same distribution as $Y=T^{2}$, where $T \sim \Cauchy(0,1)$ is a half-Cauchy random variable, with pdf given by~\cref{eq:halfcauchypdf}. The random variable $Y\sim\betaprime(1/2,1/2)$ is a beta prime random variable (with both shape parameters equal to $1/2$), with pdf
\[
        f_{Y}(y)=\frac{1}{\pi\sqrt{y}(1+y)}.%
\]
Its survival function satisfies
\[
\Pr(Y>y)\overset{y\to\infty}{\sim} (2 y^{-1/2})/\pi,%
\]
and therefore $Y$ has a power-law tail at infinity with exponent $\alpha=1/2$. Let $c>0$ be some scaling parameter. Setting
$$
\lambda_{p,j} = (c\pi^2 Y_j) / (4 p^2),
$$
we obtain
\begin{align}
        \sum_{j} \lambda_{p,j}\cvdto \id(0,\rho)=\IG(1/2,c\pi/4).
\end{align}
where $\rho(dx) = (\sqrt{c}/2) x^{-3/2} \ind_{\{x > 0\}} dx$. The tail L\'evy intensity $\overline\rho(x)$ in this case has power-law tails at $0$ and $\infty$, with exponent $1/2$.\footnote{We say that a L\'{e}vy measure $\rho$ on $(0,\infty)$ has a power-law tail at $\infty$ with exponent $\tau \in \R$ if its tail L\'{e}vy intensity $\overline{\rho}$ satisfies $\overline{\rho}(x) \overset{x\to\infty}{\sim} cx^{-\tau}$. Similarly, we say that $\rho$ has a power-law tail at $0$ with exponent $\alpha > 0$ if $\overline{\rho}(x) \overset{x \to 0}{\sim} cx^{-\alpha}$.}

\subsection{Generalised Gamma Pareto Model}
 The model described in \cref{sec:horseshoe} allows us to obtain a L\'evy measure which has power-law tails with the same exponent $\alpha$ at $0$ and $\infty$. We describe here a model that permits power-law tails with different exponents.
Let
$\lambda_{p,j}=\beta_{j}\zeta_{p,j}$ where
\begin{align}
\beta_{j} &  \sim \mathrm{Pareto}(\tau,1),
&
\zeta_{p,j} &  \sim \mathrm{etBFRY}\left(  \alpha,\left(  \frac{{{p}}\alpha\tau}%
{\eta(\tau-\alpha)}\right)^{1/\alpha},1\right)\label{eq:BFRYmodel}
\end{align}
for $\alpha\in(0,1)$, $\tau>\alpha$ and $\eta>0$. Here $\pareto(\tau,c)$ denotes the Pareto distribution with pdf $f(x)=\tau c^\tau x^{-\tau-1}\ind_{\{x>c\}}$.
Also, $\mathrm{etBFRY}(\alpha,t,\xi)$ denotes an exponentially tilted BFRY
distribution \cite{Lee2016,Bertoin2006}, with pdf
\[
g(s)=\frac{\alpha s^{-1-\alpha}e^{-\xi s}(1-e^{-ts})}{\Gamma(1-\alpha
	)((t+\xi)^{\alpha}-\xi^{\alpha})}.%
\]
We can sample easily from this distribution by inversion. The variances $\lambda_{p,j}$ follow a generalised BFRY distribution with density:
$$ f_p(x) = \frac{\tau \alpha x^{-\tau-1}}{\Gamma(1-\alpha)((t+1)^\alpha-1)} \left( \gamma(\tau-\alpha, x) - \frac{\gamma(\tau-\alpha, (t+1)x)}{(t+1)^{\tau-\alpha}}\right),$$
where $t = \left(  \frac{{{p}}\alpha\tau}%
{\eta(\tau-\alpha)}\right)  ^{1/\alpha}$ and  $\gamma(s,x)=\int_0^x t^{s-1}e^{-t}dt$ denotes the lower incomplete gamma function

Under this model, $\sum_j \lambda_{p,j}\cvdto \id(0,\rho)$ where the limiting L\'{e}vy measure $\rho$ is a generalised gamma Pareto measure, introduced by \citet{Ayed2019,Ayed2020}:
\[
\rho(dx)=\frac{\eta}{\Gamma(1-\alpha)}x^{-1-\tau}\gamma(\tau-\alpha,x) dx.
\]
As shown by \citet{Ayed2020}, the tail L\'evy intensity of this measure shows power-law behaviours at both $0$ and $\infty$:
\begin{align*}
\overline\rho(x) & \overset{x\to0}\sim c_1 x^{-\alpha},
&
\overline\rho(x) & \overset{x\to\infty}\sim c_2 x^{-\tau}
\end{align*}
for some constants $c_1,c_2>0$. The exponents $\alpha\in(0,1)$ and $\tau>\alpha$ here can take different values, allowing for different asymptotic behaviours for small and large weights.

\section{Illustrative Experiments}
\label{sec:experiments}

\subsection{MoGP at Initialisation}\label{sec:experiments_simu}
In this subsection, we illustrate the key benefits of the MoGP regime as well as our main results through simulations; we consider a FFNN model defined by \cref{eq:FFNN1,eq:B,eq:W,eq:V,eq:lambda}, with no bias, $\sigma_v = 1$, ReLU activation and univariate inputs. For the variance distributions, we consider five of the examples described in \cref{sec:examples}, namely the deterministic, inverse-gamma, beta, horseshoe and generalised BFRY models. For all models except the horseshoe, we set the parameters such that $\mathbb{E}[\sum_j \lambda_{p_1,j} ] \rightarrow 1$ as $p_1 \rightarrow \infty$. For the horseshoe model, we take $\lambda_{p_1,j}=(\frac{\pi}{2})^2\frac{U_j^2}{p_1^2}\text{ where }U_j\sim \Cauchy(0,1)$. Unless otherwise stated, the neural networks have a single hidden layer, recovering the illustrative example described in the introduction.

\begin{figure}
  \subfigure[Full distributions]{\includegraphics[width=.3\textwidth]{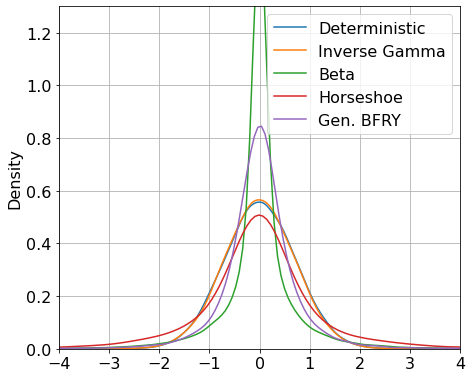}}
  \subfigure[Tails]{\includegraphics[width=.3\textwidth]{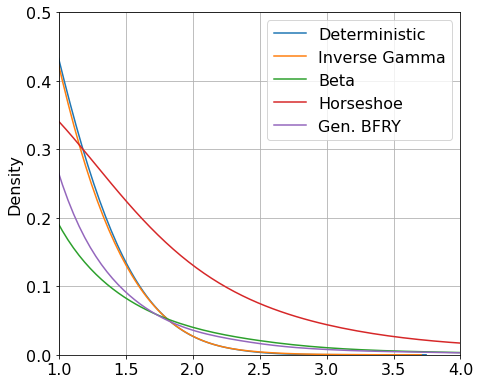}}
  \subfigure[Tails in log scale]{\includegraphics[width=.3\textwidth]{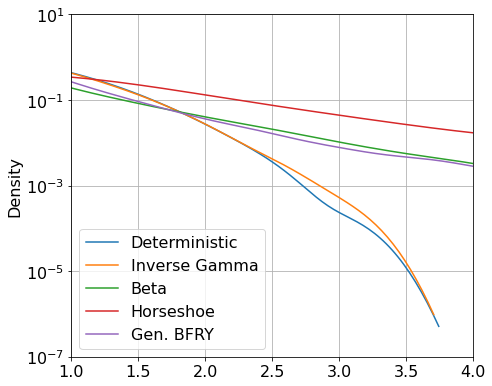}}
  \caption{MoGP output distribution}\label{fig:MoGP_output_dist}
\end{figure}

\paragraph{Output distribution.} \Cref{fig:MoGP_output_dist} shows the distribution of the output with a large width $p_1=2000$. We use 50000 samples from the model to draw the plots, each corresponding to a random realisation of the weights. The figure confirms the limiting
behaviour described in \cref{th:singleinput}: the deterministic and inverse-gamma converge to the same Gaussian Process (the orange and blue lines overlap), whereas MoGP regimes offer a wider class of output distributions. In particular, when we examine %
the densities in log-log scale, we can notice that the beta, horseshoe and generalised BFRY exhibit a density with a power-law tail (straight line in log-log scale), whereas the deterministic and inverse-gamma exhibit a light-tailed density.

\paragraph{Dependence of the dimensions on the output.} Another key consequence of \cref{th:singleinput} is that in the GP regime, the different dimensions of the output are asymptotically independent, while this is not the case in the MoGP regime. For a two-dimensional output FFNN, we report in \cref{tab:MoGP_output_corr} the empirical correlation between $(Z^{(2)}_{1}(x;p_1))^2$ and $(Z^{(2)}_{2}(x;p_1))^2$ when $p_1 \rightarrow \infty$ for the different models using 5000 random samples. The empirical results confirm the theoretical ones: we can see that for the deterministic and inverse-gamma models, the correlation converges to zero, while this is not the case for the other models.
\begin{table}
  \centering
  \begin{tabular}{c|cccccc}
     \hline
Width & Deterministic & Inverse Gamma & Beta & Horseshoe & Gen. BFRY \\
\hline
100 & 0.019994 & 0.113897 & 0.320444 & 0.691159 & 0.33325\\
500 & 0.00539 & 0.028584 & 0.281498 & 0.434425 & 0.219763\\
1000 & 0.005495 & 0.015217 & 0.279571 & 0.995462 & 0.316032\\
2000 & 0.001844 & 0.004522 & 0.297515 & 0.253737 & 0.235673\\
     \hline
   \end{tabular}
  \caption{MoGP output correlation.}\label{tab:MoGP_output_corr}
\end{table}

\paragraph{Distribution of the largest weight.} \Cref{prop: extremes} describes another benefit of the MoGP regime: when the Levy measure is trivial, i.e. in the GP regime, the largest weight in each layer converges in probability to zero, while this is not the case in the MoGP regime. \cref{fig:largest_weight_distrib} empirically validates this result; we show the evolution of the distribution of $\max_{1 \leq j,k \leq p} |W^{(2)}_{jk}|$ as the width $p_1$ grows.
This property can have a significant impact on the performance of the models since some weights remain non-negligible asymptotically and can be connected to nodes representing important hidden features. This, coupled with a heavy-tailed distribution of the nodes, can favour specialisation of the neurons, with benefits for pruning and feature learning. We refer the reader to \cref{sec:MoGP regularisation,sec:MoGP bayesian} for experiments with real data where our proposed framework is used either in a frequentist or Bayesian fashion.

\begin{figure}
  \centering
  \subfigure[Deterministic]{\includegraphics[width=.3
\textwidth]{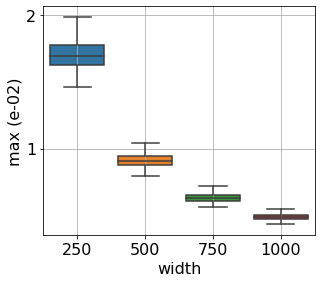}}
    \subfigure[Beta]{\includegraphics[width=.3\textwidth]{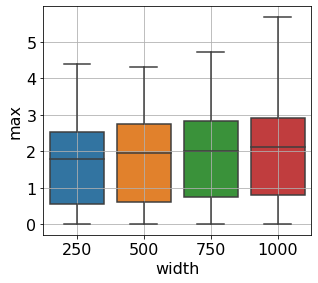}}
  \subfigure[Gen. BFRY]{\includegraphics[width=.3\textwidth]{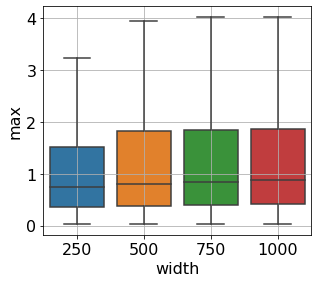}}
  \caption{Distribution of the largest weight when the width increases.}\label{fig:largest_weight_distrib}
\end{figure}

\begin{figure}
  \centering
  \includegraphics[width=.3\textwidth]{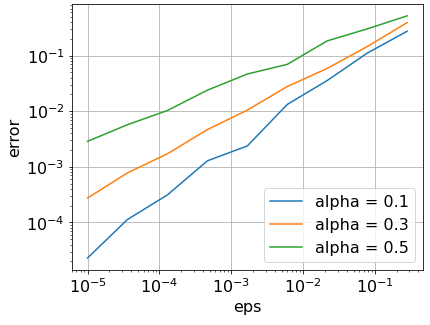}
  \caption{Expected truncation error as a function of $\epsilon$ (in log-log scale).}\label{fig:truncation error}
\end{figure}

\paragraph{Truncation error.} In \cref{fig:truncation error}, we illustrate \cref{thm: eps-pruning} with a generalised BFRY model with different values of $\alpha$ (and $\tau=5$ is fixed). The expectation is estimated using $1000$ simulations with width $p_1=2000$ and depth $L=3$. The empirical results match well the theoretical bound. In particular, in log-log scale, we get an empirical slope of $0.492$ for $\alpha = 0.5$, an empirical slope of $0.691$ for $\alpha = 0.3$ and an empirical slope of $0.920$ for $\alpha=0.1$. Therefore, the slope is approximately equal to $1-\alpha$, which confirms the theoretical rate of decay of the expected pruning error as a function of the truncation level $\eps$. We get similar results with different depths $L$.

In \cref{app:Bayesian-experiment}, we report further experiments analysing the vanishing/exploding gradient phenomenon in the MoGP context for deep networks (up to 20 layers). We also show how it can be alleviated with the right choice of model parameters.

The following two subsections describe how one can use our proposed framework with real data, either as a regularisation term or as a prior for a Bayesian Neural Network. We illustrate the discussed benefits of the proposed framework on compressibility and feature learning. The datasets considered in our experiments are MNIST and Fashion MNIST. Both datasets correspond to an image classification problem with 10 classes (digits for MNIST and clothes type for Fashion MNIST). The images are grey-scale and split between training and test sets, composed of 60000 and 10000 examples.

\begin{figure}
  \centering
  \subfigure[Deterministic]{\includegraphics[width=.35\textwidth]{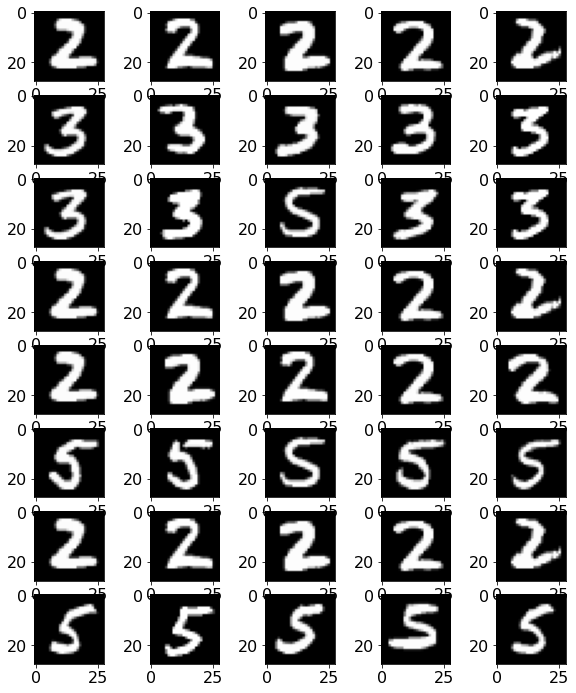}\hspace{0.05\textwidth}
  \includegraphics[width=.35\textwidth]{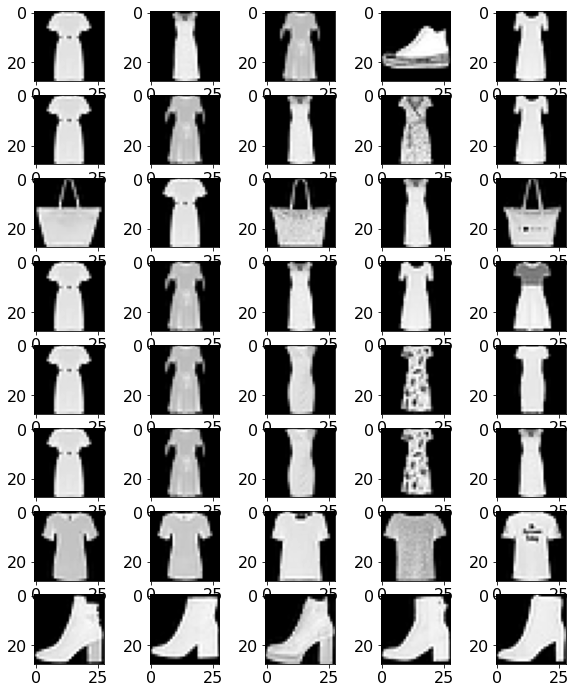}}
  \subfigure[Horseshoe]{\includegraphics[width=.35\textwidth]{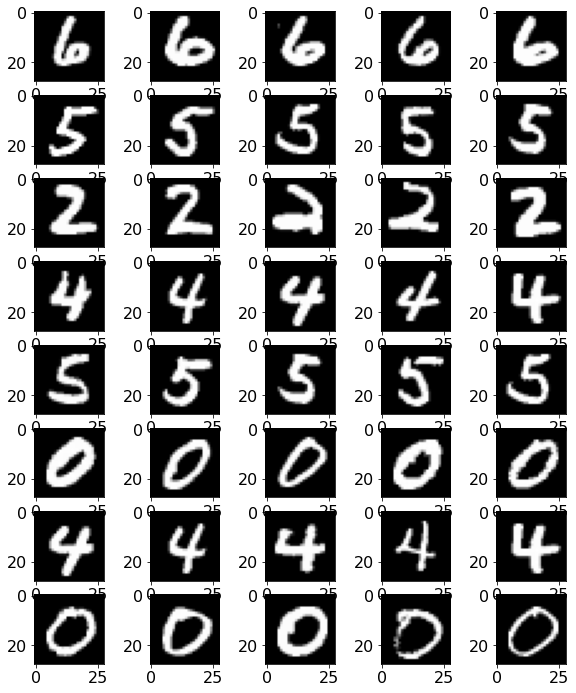}\hspace{0.05\textwidth}
  \includegraphics[width=.35\textwidth]{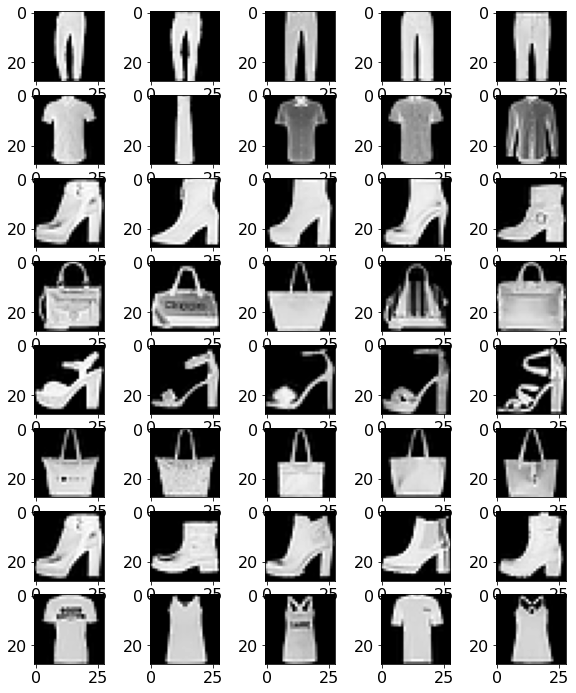}}
  \caption{Visualisation of the top-8 neurons of the first hidden layer of models trained on MNIST (left) and Fashion MNIST (right) --- deterministic and horseshoe cases. Each row corresponds to a neuron. The elements of the row correspond to the ordered 5 images that maximise the neuron output.}\label{fig:vizualise_features1}
\end{figure}

\subsection{MoGP as a Regularisation}\label{sec:MoGP regularisation}
The most straightforward application of the MoGP framework is to use the prior as a regularisation term to add to the loss. We consider FFNN models $f_\theta $ with ReLU activation and three hidden layers, all having the same width $p=2000$. The parameters $\theta = (\lambda^{(l)}_{p,j}, V^{(l)}_{jk}, B^{(l)}_k)$ are trained using Adam optimisation for 50 epochs to minimise the objective function:
\begin{equation}
\mathcal{L} = \frac{1}{n}\sum\limits_{i=1}^n \ell( y_i, f_\theta(x_i) ) + \gamma \left(\sum\limits_{jkl} \log \pi_V(V^{(l)}_{jk}) + \sum\limits_{lk} \log \pi_B(B^{(l)}_k) + \sum\limits_{jl} \log \pi_\lambda(\lambda^{(l)}_{p,j}) \right) \label{eq:obj_regularisation}
\end{equation}
where $\ell$ is a loss function,  $\pi_V$ and $\pi_B$ are, respectively, the densities of zero-mean Gaussian distributions with variance $\sigma_v^2$ and $\sigma_b^2$, and $\pi_\lambda$ is the density of the finite-dimensional approximation of the limiting infinitely divisible distribution. The parameter $\gamma$ controls the weight of the prior. Notice that when $\gamma=1$, we recover the maximum a posteriori estimator when $\ell$ is a log likelihood.
In our experiments, we take $\sigma_v=\sigma_b=1$ and $\gamma=0.2$, and $\ell$ is the cross-entropy loss. We consider three examples detailed in \cref{sec:examples}, namely, the deterministic, the horseshoe and the generalised BFRY models. The deterministic and the horseshoe parameters are as in the simulated experiments. For the generalised BFRY, we set $\alpha=0.8$ and $\tau=5$. We bring to the reader's attention that for the deterministic model, the variance distribution is a Dirac at $1/p$; therefore, the network is trained with a similar parameterisation as the Neural Tangent Kernel framework \cite{Jacot2018}.

\begin{figure}
  \centering
  \subfigure[Gen. BFRY]{\includegraphics[width=.35\textwidth]{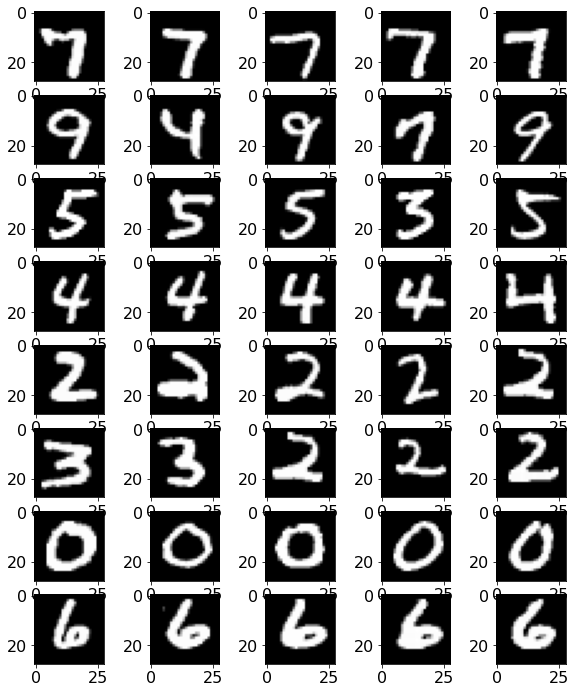}\hspace{0.05\textwidth}
  \includegraphics[width=.35\textwidth]{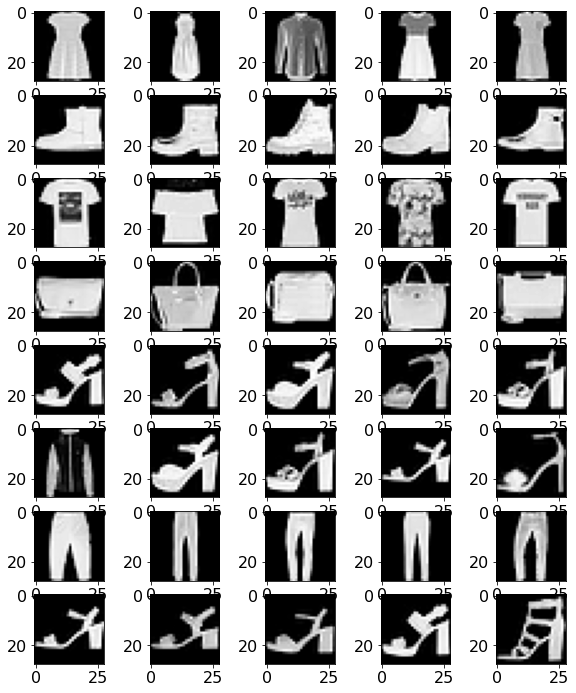}}
  \caption{Visualisation of the top-8 neurons of the first hidden layer of models trained on MNIST (left) and Fashion MNIST (right) --- generalised BFRY case. Each row corresponds to a neuron. The elements of the row correspond to the ordered 5 images that maximise the neuron output.}\label{fig:vizualise_features2}
\end{figure}

\paragraph{Feature learning.} For each variance distribution, we train a network on MNIST and another on Fashion MNIST to minimise \cref{eq:obj_regularisation}. For all the models, we reach a test accuracy of approximately $98\%$ on MNIST and $88\%$ on Fashion MNIST, which are standard performances for feedforward models. Exact numbers can be found in the compressibility paragraph hereafter. We visualise the top-8 neurons of the first hidden layer of each model by plotting the 5 input images that maximise the neuron output. For the horseshoe and generalised BFRY, the top neurons are the ones with the highest variance $\lambda^{(1)}_{k}$. For the deterministic, since all the variances are equal, the top neurons are selected according to $\sum_k (W_{jk}^{(1)})^2$ (using this metric for the horseshoe and generalised BFRY leads to similar results). \Cref{fig:vizualise_features1,fig:vizualise_features2} reveal a key distinction between the MoGP regime (horseshoe and generalised BFRY) and the standard GP regime (deterministic). In the former regime, the top neurons tend to be more specialised: each neuron learns a different feature. In the latter regime, several top neurons learn the same features, which is materialised by almost equal lines in \cref{fig:vizualise_features1,fig:vizualise_features2}, such as the neurons 1, 4, 5, and 7 of the network trained on MNIST with the deterministic model.
To validate this phenomenon, we repeat the training of each model five times. In \cref{fig:unique_imgs}, we plot the evolution of the average total number of unique images among the representative ones as a function of the number of top neurons, and also as a function of the number of images considered per neuron. The total number of unique images is interpreted as a simple metric to quantify the diversity of the learned features. The curves validate our hypothesis; in the MoGP regime, the top neurons learn more heterogeneous features.

\begin{figure}
  \centering
  \subfigure[]{\includegraphics[width=.3
  \textwidth]{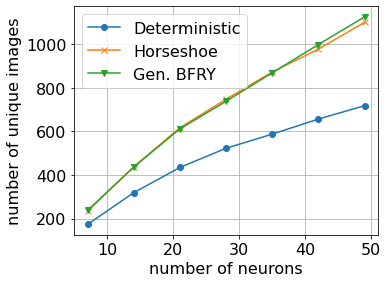}}
    \subfigure[]{\includegraphics[width=.3\textwidth]{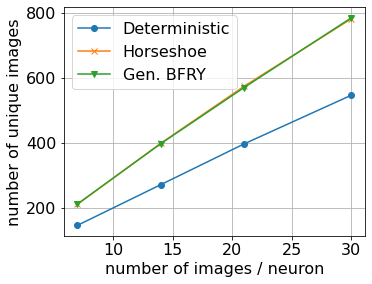}}
  \caption{Number of unique images when a) representing each neuron using the top-30 images that maximise the output of the top neurons, while varying the number of top neurons, and b) representing each neuron by a varying number of images, while fixing the number of neurons to thirty. Average of five runs for models trained on the MNIST dataset.}\label{fig:unique_imgs}
\end{figure}

\begin{table}
\begin{center}
\begin{tabular}{lrrr}
\toprule
\pbox{4cm}{Truncation \\(i.e., $1-\kappa$)} &  Deterministic &  Horseshoe &  Gen. BFRY \\
\midrule
0.0$\%$        &   97.44 ($\pm 0.05$) &      97.94 ($\pm 0.09$) &        98.00 ($\pm 0.07$) \\
80.0$\%$       &   95.58 ($\pm 0.70$) &      97.94 ($\pm 0.09$) &        98.00 ($\pm 0.07$)\\
90.0$\%$       &   71.70 $(\pm 11.2)$ &      97.94 ($\pm 0.09$) &        98.00 ($\pm 0.07$) \\
95.0$\%$       &   23.90 $(\pm 12.0)$&      97.94 ($\pm 0.09$) &        98.00 ($\pm 0.07$) \\
98.0$\%$       &   12.12 $(\pm 3.95)$ &      64.22 $(\pm 13.1)$ &        65.74 $(\pm 6.98)$ \\
98.5$\%$       &   10.36 $(\pm 0.63)$ &      44.14 $(\pm 10.9)$ &        50.76 $(\pm 2.90)$ \\
\bottomrule
\end{tabular}
\caption{MNIST truncation accuracy: Average accuracy and standard deviation (between parenthesis) using five independent runs.}\label{tab:MNIST}
\end{center}
\end{table}

\paragraph{Compressibility.} We expect the higher diversity of the features learnt by the top neurons to affect the compressibility of the networks. We compare the degradation of the accuracy of the models with node pruning. For the horseshoe and generalised BFRY, we prune the nodes as described in \cref{sec:compression} using the node variances $\lambda^{(l)}_j$. For the deterministic model, we use $\sum_k (W_{jk}^{(1)})^2$. For each layer,  a given fraction $\kappa$ of the nodes is kept. The mean and standard deviation of the accuracies are reported in \cref{tab:MNIST} for MNIST and \cref{tab:FashionMNIST} for Fashion MNIST. As expected, the horseshoe and generalised BFRY
outperform the deterministic model, with a slight advantage for the latter. What is even more interesting is that the accuracies of the pruned generalised BFRY models have a smaller variance. Though both the horseshoe and generalised BFRY have a power-law tail,  the tail of the horseshoe is heavier; in particular, the distribution has an infinite expectation, which is not the case for the generalised BFRY. This can explain the difference between the models in terms of variances. We believe this simple experiment serves as motivation to further explore the MoGP regime beyond the horseshoe model, as different limiting distributions can offer valuable practical advantages.

\begin{table}
\begin{center}
\begin{tabular}{lrrr}
\toprule
\pbox{4cm}{Truncation \\(i.e., $1-\kappa$)} &  Deterministic &  Horseshoe &  Gen. BFRY \\
\midrule
0.0$\%$        &   87.98 $(\pm 0.19)$&      88.70 $(\pm 0.20)$ &        88.54 $(\pm 0.19)$\\
80.0$\%$       &   86.24 $(\pm 0.80)$&      88.70 $(\pm 0.20)$&        88.54 $(\pm 0.19)$ \\
90.0$\%$       &   60.24 $(\pm 5.14)$&      88.68 $(\pm 0.19)$&        88.56 $(\pm 0.18)$ \\
95.0$\%$       &   19.64 $(\pm 7.01)$&      88.50 $(\pm 0.12)$&        88.40 $(\pm 0.25)$\\
98.0$\%$       &   10.84 $(\pm 1.17)$&      76.56 $(\pm 3.35)$&        77.24 $(\pm 2.32)$\\
98.5$\%$       &   10.26 $(\pm 0.58)$&      58.26 $(\pm 14.2)$&        60.44 $(\pm 3.42)$\\
\bottomrule
\end{tabular}
\caption{Fashion MNIST truncation accuracy: Average accuracy and standard deviation (between parenthesis) using five independent runs.}\label{tab:FashionMNIST}
\end{center}
\end{table}

In \cref{app:Bayesian-experiment}, we empirically verify on the Cifar10 dataset that using the MoGP framework as a regularisation also improves the compressibility of convolutional neural networks.

\subsection{MoGP in a Fully Bayesian Setting}\label{sec:MoGP bayesian}
We further demonstrate the MoGP in a fully Bayesian setting, where we simulate the posterior distribution of a FFNN with MoGP priors on the weights. Let $f_\theta$ be a FFNN with ReLU activation and three hidden layers of width $p=2000$. The log joint-density for classification with this FFNN is then given as:
\begin{align}
\log g(\theta) &= \sum_{i} \log h(y_i, f_\theta(x_i)) + \sum_{j,k,l} \log \pi_V(v_{jk}^{(l)}) + \sum_{l,k} \log \pi_B(b_k^{(l)}) + \sum_{j,l} \log \pi_\lambda(\lambda_j^{(l)}).
\end{align}
We consider the $C$-way classification problem where $y_i \in \{1, \dots, C\}$ and $h(y, f_\theta(x))$ is the categorical likelihood, i.e.,
$h(y, f_\theta(x) ) = \mathrm{softmax}(f_\theta(x))_y$,
with $\mathrm{softmax}(f_\theta(x))_c = \frac{\exp({f_\theta(x)}_c)}{\sum_{c'=1}^C \exp({f_\theta(x)}_{c'})}$ to get proper probability vectors.

We compare the deterministic, the horseshoe and the generalised BFRY models on MNIST and Fashion MNIST datasets. We infer the posteriors of the network weights via Stochastic Gradient Hamiltonian Monte-Carlo (SGHMC)~\citep{Chen2014} with batch size set to 100. We run the samplers for 100 epochs through datasets and collect samples every 2 epochs after 50 burn-in epochs. Following \citet{Zhang2020}, we adopt a simple cosine-annealed step size with a single cycle and set the first half of the epochs as an exploration stage (updating without noise for quick convergence to a local minimum). For the generalised BFRY, considering the importance of the hyperparameter $\alpha$, we introduce a uniform prior on it and infer its value along with the model parameters. We run every experiment five times and averaged results.

\paragraph{Compressibility.}
As in \cref{sec:MoGP regularisation}, we first compare the predictive classification accuracies of the FFNN models under a varying truncation ratio. For all models, we collect 25 samples after 50 burn-in epochs (collecting a sample at the end of every 2 epochs after the burn-in), and the test accuracy is measured with Monte-Carlo estimates of predictive distributions,
\begin{align}
    p(y_*|x_*, \mathcal{D}) \approx \frac{1}{S} \sum_{s=1}^S \mathrm{softmax}(f_{\theta^{(s)}}(x_*))_{y_*},
\end{align}
where $\mathcal{D}$ is the training set and $\theta_1,\dots, \theta_S$ are samples collected from SGHMC.
As in \cref{sec:MoGP regularisation}, we prune the nodes with respect to the magnitude of the node variances $\lambda_j^{(l)}$ and measure the test accuracies of the pruned networks. \cref{tab:MNIST_Bayes,tab:Fashion_MNIST_Bayes} summarise the results. FFNNs with horseshoe or generalised BFRY priors are more robust to truncation; both maintain decent classification accuracies even when 95$\%$ of the neurons are truncated.   For the generalised BFRY, we present the posterior samples of the hyperparameter $\alpha$ in \cref{fig:gbfry_alphas}. The posteriors are well concentrated around the range $[0.7, 0.8]$.

\begin{table}
\begin{center}
\begin{tabular}{lrrr}
\toprule
\pbox{4cm}{Truncation\\ (i.e., $1-\kappa$)} &  Deterministic &  Horseshoe &  Gen. BFRY \\
\midrule
0.0$\%$         &   90.23 $(\pm 0.12)$&      97.83 $(\pm 0.07)$ &        97.78 $(\pm 0.10)$\\
80.0$\%$       &   13.14 $(\pm 3.18)$&      97.93 $(\pm 0.06)$&        97.71 $(\pm 0.10)$ \\
90.0$\%$       &   9.98 $(\pm 0.82)$&      97.73 $(\pm 0.04)$&        97.72 $(\pm 0.05)$\\
95.0$\%$       &   9.70 $(\pm 0.89)$&      97.59 $(\pm 0.06)$&        97.68 $(\pm 0.03)$\\
98.0$\%$       &   9.46 $(\pm 0.52)$&      87.97 $(\pm 3.74)$&        54.90 $(\pm 4.13)$\\
98.5$\%$       &   9.46 $(\pm 0.52)$&      89.29 $(\pm 3.58)$&        57.69 $(\pm 11.15)$\\
\bottomrule
\end{tabular}
\end{center}
\caption{Predictive classification accuracy on MNIST dataset under various truncation ratio.}\label{tab:MNIST_Bayes}
\end{table}

\begin{table}
\begin{center}
\begin{tabular}{lrrr}
\toprule
\pbox{4cm}{Truncation\\ (i.e., $1-\kappa$)} &  Deterministic &  Horseshoe &  Gen. BFRY \\
\midrule
0.0$\%$         &   80.65 $(\pm 0.12)$&      87.72 $(\pm 0.12)$ &        87.75 $(\pm 0.05)$\\
80.0$\%$       &   10.00 $(\pm 0.00)$&      87.57 $(\pm 0.28)$&        87.28 $(\pm 0.29)$ \\
90.0$\%$       &   10.00 $(\pm 0.00)$&      87.43 $(\pm 0.28)$&        87.01 $(\pm 0.29)$\\
95.0$\%$       &   10.00 $(\pm 0.00)$&      87.27 $(\pm 0.32)$&        86.64 $(\pm 0.54)$\\
98.0$\%$       &   10.00 $(\pm 0.00)$&      80.65 $(\pm 3.73)$&        81.85 $(\pm 3.78)$\\
98.5$\%$       &   10.00 $(\pm 0.00)$&      59.34 $(\pm 6.02)$&        68.34 $(\pm 6.71)$\\
\bottomrule
\end{tabular}
\end{center}
\caption{Predictive classification accuracy on Fashion MNIST %
under various truncation ratio.}\label{tab:Fashion_MNIST_Bayes}
\end{table}

\begin{figure}
    \centering
    \includegraphics[width=0.3\linewidth]{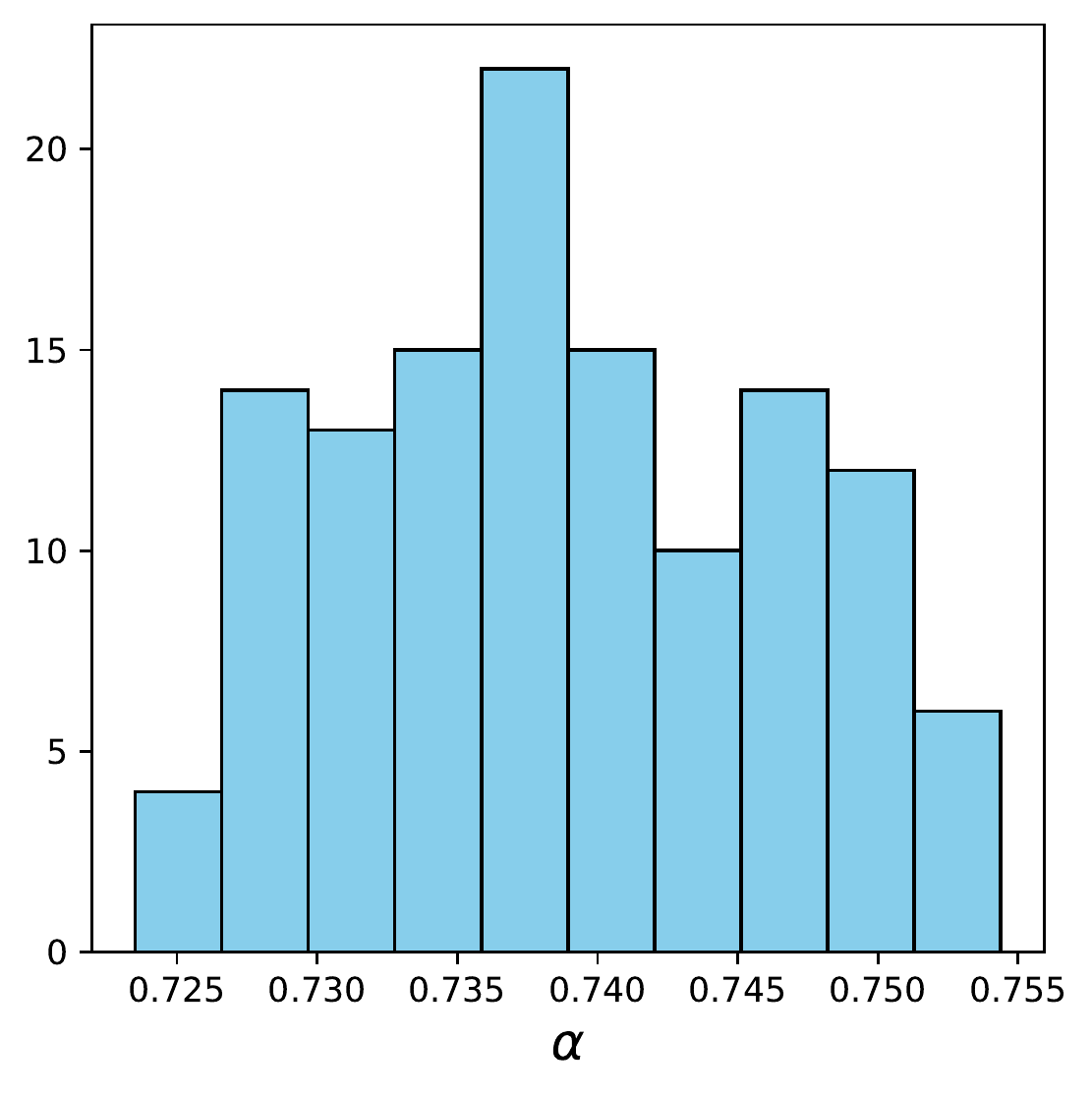}
    \includegraphics[width=0.3\linewidth]{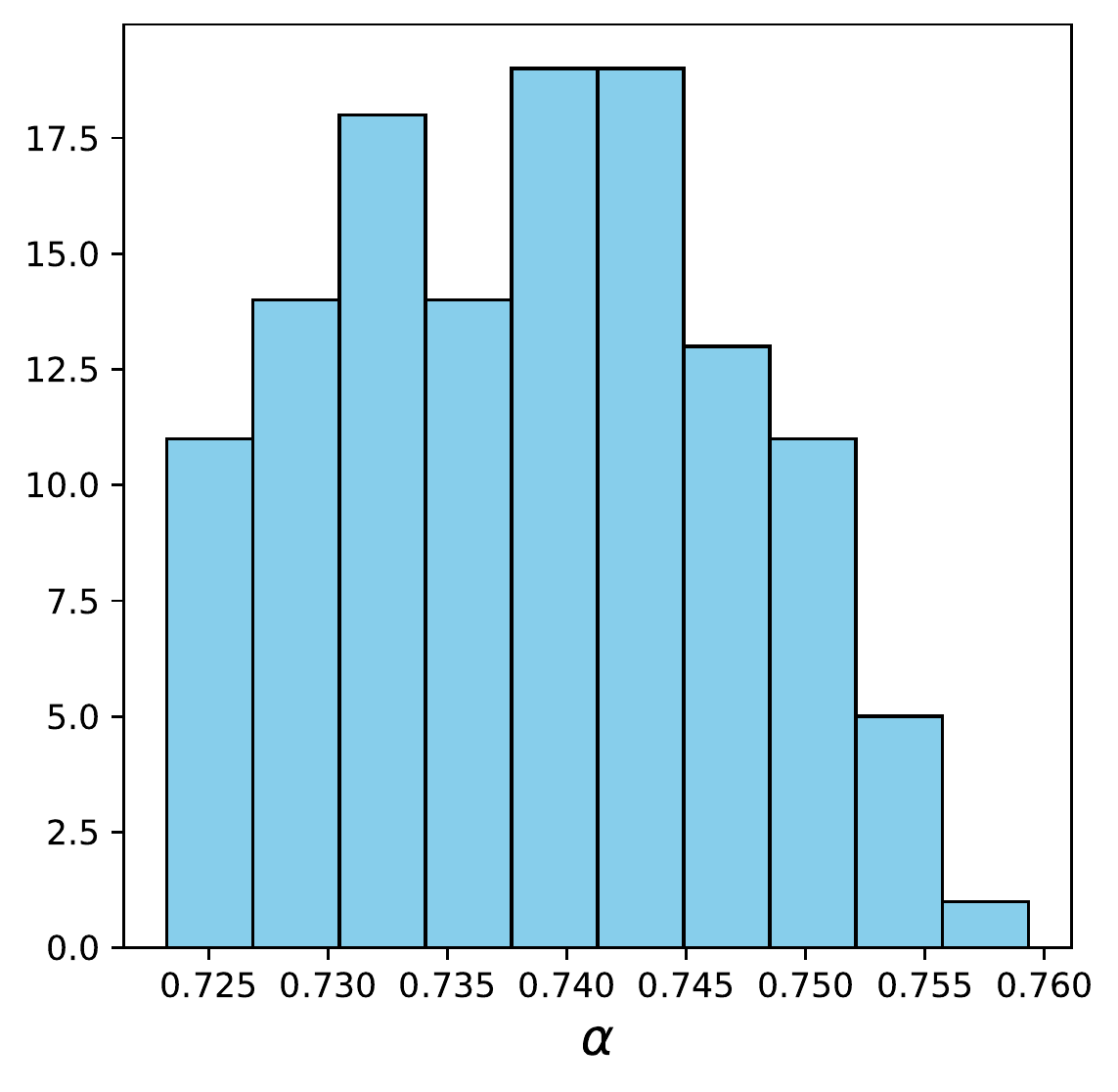}
    \caption{Inferred $\alpha$ values for gen. BFRY on MNIST (left) and Fashion MNIST (right).}
    \label{fig:gbfry_alphas}
\end{figure}

\paragraph{Feature learning.}
Finally, to demonstrate the feature learning aspect of MoGP priors, we conduct a transfer learning experiment.
We start by taking an external dataset, not used during training, and split it into two halves $D_1$ and $D_2$. Then, we sort the neurons of the first hidden layers of the trained FFNNs with respect to their average magnitudes of activations over the images in $D_1$, and select the top-$k$ activated neurons. Next, for each of these  top-$k$ activated neurons, we select $m$ images from $D_1$ which most strongly activate this neuron, and use the combined collection of these images (over all the top-$k$ activated neurons) to form a subset that will then be used for transfer learning. Each image in this subset is then represented by a vector consisting of the activations of the top-activated neurons. For instance, choosing the top-30 activated neurons and the top-10 activating images per neuron, results in 300 images in total in this subset  (assuming no overlaps in top-activating images). Each image in the subset is represented as a 30-dimensional vector which is just the concatenation of activations from the top-30 activated neurons. Our hypothesis here is that if a trained neural network exhibits feature learning, the subset of images selected in this way is representative enough to express the important features in the set $D_1$, and so,   if we train a new classifier based on this subset, the resulting model should generalise well to $D_2$. To validate this hypothesis, we train a light-weight FFNN with one hidden layer using the selected subset with the selected vectors of activations. Then we evaluate the test accuracy of the trained light-weight FFNN using the vector activation form of $D_2$ computed from the neurons selected with $D_1$. Since we are working with a Bayesian model, for each configuration, we have multiple samples of parameters. We first compute the Monte-Carlo estimates of the average activations using those samples, and then use the estimates to sort neurons, select images and form vectorized activations. We perform this transfer learning experiment with varying subset sizes, repeating all experiments five times per configuration and take average results. \cref{fig:bayes_transfer} summarises the result. As we expect, the horseshoe and generalised BFRY models transfer well, while the deterministic model fails to transfer.  In particular, the test accuracy of the deterministic model does not increase as the size of the training set increases, demonstrating that the model does not exhibit feature learning.

We comment that unlike our results in \cref{sec:MoGP regularisation}, the top-5 activating images for each of the top-8 activated neurons
do not show a noticeable difference between the deterministic case and the rest in this fully Bayesian setting.
See \cref{fig:bayes_neurons} in \cref{app:Bayesian-experiment} for the visualisation of those images in the deterministic, the horseshoe
and the generalised BFRY cases.

\begin{figure}
    \centering
    \includegraphics[width=0.3\linewidth]{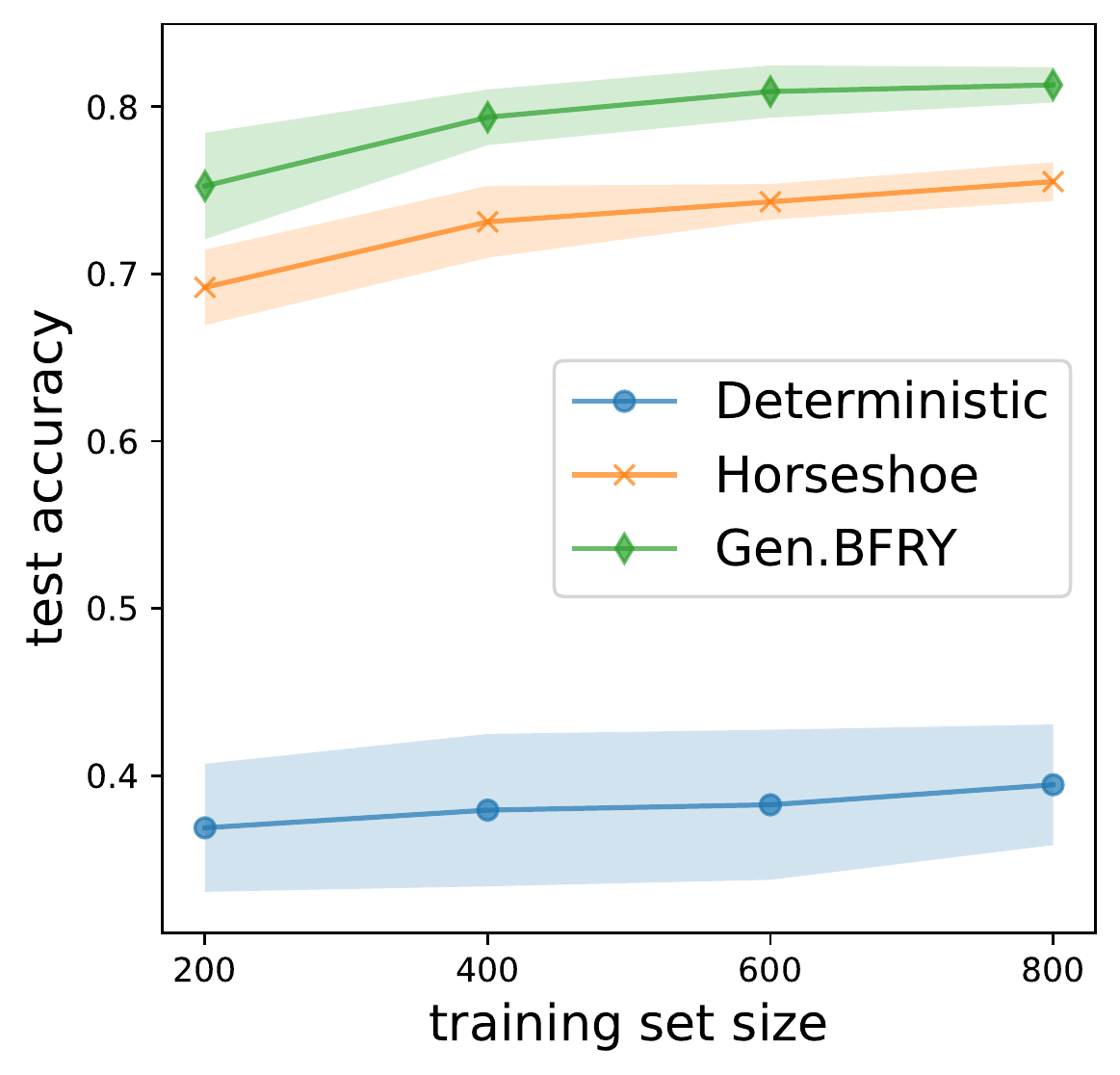}
    \includegraphics[width=0.3\linewidth]{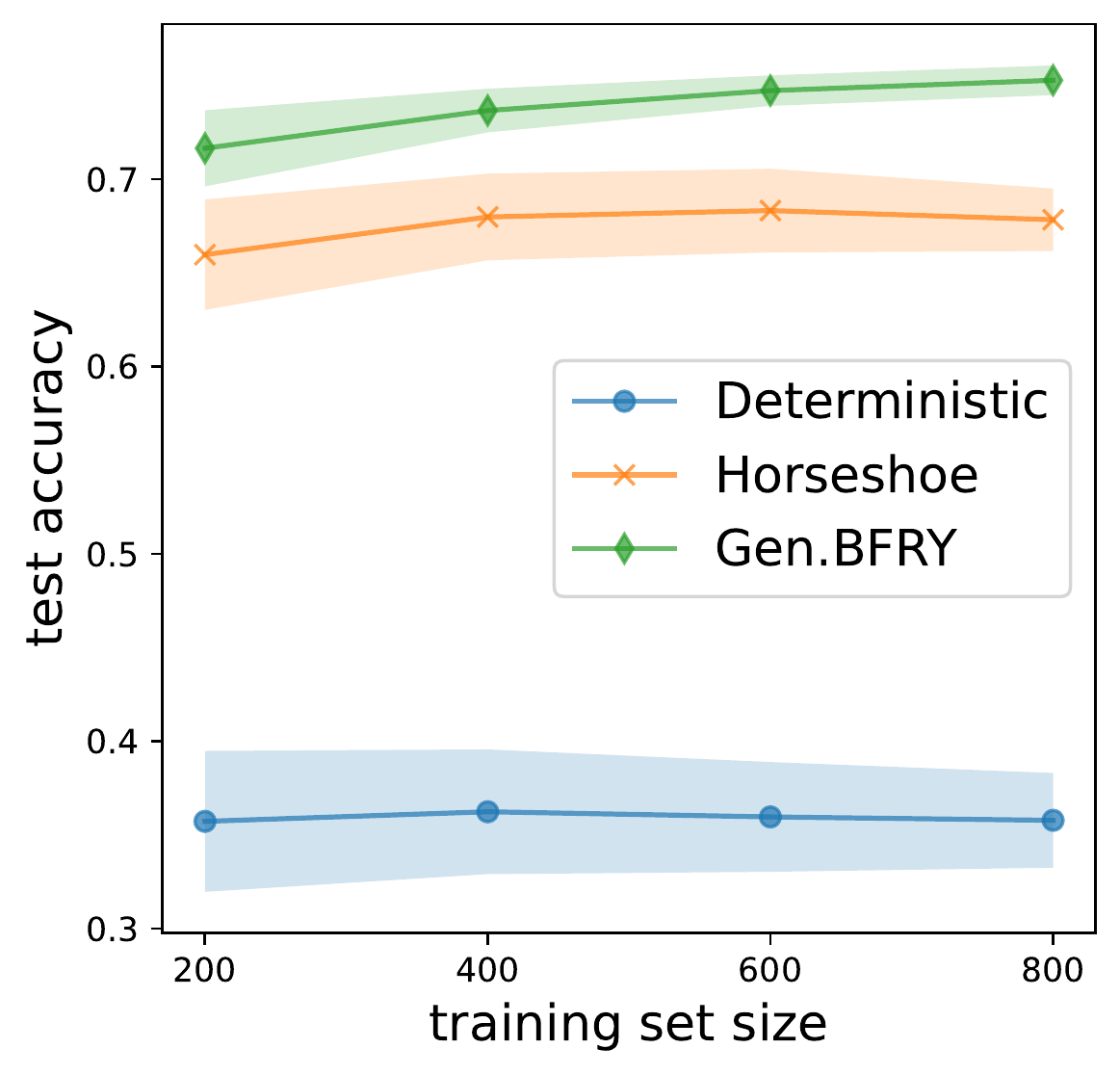}
    \caption{Comparing feature learning aspect of FFNN models via transfer learning. Results for MNIST (left) and Fashion MNIST (right).}
    \label{fig:bayes_transfer}
\end{figure}

\section{Discussion and Further Directions}
\label{sec:discussion}

\paragraph{Models with iid non-Gaussian weights.} \citet{Neal1996}'s seminal work on infinite-width limits of shallow neural networks includes the case where weights are initialised by an iid symmetric stable distribution. A subsequent thorough analysis \cite{Der2006} showed that such networks converge to stable processes as the widths increase, and this analysis was generalised from shallow networks to deep networks~\cite{Favaro2020,Bracale2021,Jung2021}. This line of work (the iid model) and our dependent-weight model can both lead to infinite-width limits that are heavy-tailed non-Gaussian processes; however, the sources of heavy-tails in the two cases are different. In the former (the iid model), the source is the use of a non-Gaussian, heavy-tailed distribution for initialising weights, which corresponds to assuming in our setup that the $V_{jk}^{(l)}$ are sampled independently from a heavy-tailed distribution, instead of a Gaussian distribution. In the latter (dependent-weight model), on the other hand, the source is the use of a per-node random variance that is shared by the weights of all outgoing edges from a given node. As a result, the limiting networks of these two classes of models have different properties. In the former, nodes in a layer of a limiting network are always independent, while in the latter, they are usually dependent. Also, in the former, the limiting networks exist largely due to a normalisation adjusted for the sum of heavy-tailed non-Gaussian random variables, while in the latter, the limiting networks exist because the random variances of nodes in a layer remain summable at the infinite-width limit. One interesting future direction is to study the generalisation of our model class where the $V_{jk}^{(l)}$ are sampled independently from a stable or heavy-tailed distribution using techniques developed in \cite{Neal1996,Der2006,Favaro2020,Bracale2021,Jung2021}. Such a generalisation would be a good starting point for analysing the pruning of not only network nodes but also edges. Also, our tools for handling the limiting random variables with infinitely divisible distributions, and for analysing the tail behaviour and pruning of the limiting neural networks, may help extend results of that line of work. Finally, we point out that \citet{LeeYYL22} studied the infinite-width limits of neural networks where weights are initialised by iid Gaussians but the shared variance of weights in the last readout layer is made random. This simple change caused the limiting networks to be mixtures of Gaussian processes, increasing the expressivity of those networks without sacrificing the tractability of Gaussian processes for inference much. Our work differs from \citet{LeeYYL22} in that our models use per-node random variances, while theirs use per-layer random variances.

\paragraph{Row-column exchangeable models.} \citet{Tsuchida2019} introduced a general class of deep FFNNs with dependent weights, and analysed their infinite-width limits. They consider\footnote{Under the assumtion that $\bfE[F_{jk}]=0$} that, at a given layer, the weights take the form
$W_{jk}=F_{jk}/\sqrt{p}$,
where the (infinite) array of random variables $(F_{jk})$ is assumed to be row-column exchangeable (RCE), that is
$(F_{\pi_1(j)\pi_2(k)})\overset{\cvd}{=}(F_{jk})$
for any permutations $\pi_1$ and $\pi_2$ of $\mathbb N$. By the Aldous-Hoover theorem, any RCE array admits the representation $F_{jk}=f(A, B_j,C_k,D_{jk})$ for some measurable function $f$ and some iid random variables $A, B_j,C_k,D_{jk}$. A crucial thing to note is that here none of the random variables depends on the width $p$. In contrast, our model assumes that the rows are independent, and their distributions may depend on $p$ in a non-trivial way. The models presented in \cref{sec:ex:inversegamma}, which converge to a GP, fall into this RCE class, as they can be expressed as $W_{jk}=\frac{\sqrt{Y_j}}{\sqrt{p}} V_{jk}$. More generally, \citet{Tsuchida2019} show that RCE models converge to a MoGP. The properties of the infinite-width neural network are, however, very different from those obtained in this article in the MoGP regime. A key difference between the infinitely divisible models considered here and RCE models is that in the latter, the weights still converge uniformly to 0 in the infinite-width limit:
$\max_{j=1,\ldots,p} |W_{jk}|=\frac{1}{\sqrt{p}}\max_{j=1,\ldots,p} |F_{jk}|\cvpto 0$
as $p\to\infty$. This follows from the fact that $(F_{jk})$ are exchangeable in $j$, and thus conditionally iid. RCE therefore cannot exhibit the behaviour described in \cref{prop: extremes}.

\paragraph{L\'evy adaptive regression kernels.} In the case of a single hidden layer with location parameter $a=0$, the infinite-width model falls in the class of L\'evy adaptive regression kernels \cite{Wolpert2011}, which is explored by \citet{Jang2017} for kernel learning.

\paragraph{L\'evy measures and sparsity-promoting priors.} The link between L\'evy measures and sparsity promoting models has been explored in the high-dimensional (sparse) linear regression setting, see e.g. the works of \citet{Caron2008,Polson2012,McCullagh2018}.

\paragraph{MoGP as a stationary distribution for trained networks with SGD.} Understanding the statistical properties of trained neural networks is a complex endeavor due to the non-convexity of deep learning problems. The past decade has seen considerable efforts to provide theoretical tools to understand neural networks trained with SGD. A fruitful line of research proposes to view the SGD dynamics with a fixed step size as a stationary stochastic process \citep{mandt2016variational,hu2017diffusion,chaudhari2018stochastic,zhu2018anisotropic}. The intuition is that with a constant step size, SGD first moves towards a valley of the objective function and then bounces around because of sampling noise in the gradient estimate. This stationary distribution governs the statistical properties of the trained networks. The core question is then how to model the stationary distribution. In \citet{shin2021compressing,barsbey2021heavy}, the authors base their models on the fact that when the networks are trained with large learning rates or small batch sizes, heavy-tailed stationary distributions emerge \citep{hodgkinson2021multiplicative,gurbuzbalaban2021heavy}. Under such a hypothesis, the authors derive a generalisation bound as a function of the tail index of the weights. Both papers assume that the weights of the trained network are asymptotically independent when the width is large enough. In \citet{shin2021compressing}, the weights follow a Pareto distribution, whereas in \citet{barsbey2021heavy}, the weights follow a generic heavy-tailed distribution. One might explore whether our MoGP framework might be used as a possible stationary distribution of a trained network, so that similar generalisation bounds could then be derived (see \cref{sec:compression} for more details). Our proposed framework has the additional benefit of readily taking into account the summability of the weights as the width goes to infinity, whereas other models would need to impose an additional scaling factor.

\paragraph{Infinite networks with bottlenecks.} In order to allow for feature/representation learning in infinitely-wide neural networks, \citet{Aitchison2020} recently proposed the use of infinite neural networks with bottlenecks. The idea is to consider the same iid Gaussian assumptions on the weights as for NNGP, but to set one layer (the bottleneck layer) to have a fixed number of hidden nodes, while taking the number of nodes in other layers to infinity. As shown by \citet{Aitchison2020}, the resulting model no longer converges to a Gaussian process, but rather to a Gaussian process with a random kernel, and this then allows for representation learning. The training dynamics of such a model is also partially analysed~\cite{Littwin2021}.

Our general approach allows to construct neural networks with bottlenecks similar to \cite{Aitchison2020}; the difference is that in our case, the limiting model is obtained by taking the widths of all layers to infinity, including the bottleneck layer, but the number of active nodes in the bottleneck layer remains finite in this limit, and converges to a Poisson distribution. For example, one way to achieve this is to use, in the bottleneck layer, the Bernoulli prior described in \cref{sec:bernoulliprior} where, for $p_l\geq c$, $\lambda^{(l)}_{p_l,j}\sim\Ber(c/p_l)$, for some $c>0$. The limit has parameters $a^{(l)}=0$ and $\rho^{(l)}=c\delta_1$. As $p_l$ tends to infinity, the number of active hidden nodes of the bottleneck layer, that is the number of nodes with $\lambda_{p_l,j}>0$, converges to a Poisson random variable:
$\sum_{j=1}^{p_l} \ind_{\{\lambda_{p_l,j}>0\}}\cvdto \Poi(c)$.
Using $\lambda^{(l)}_{p_l,j} = c_2/p_l$ for all other layers, we obtain a model where the bottleneck layer has a random number of hidden nodes.

\paragraph{Deep Gaussian processes.} Deep Gaussian processes~\cite{Damianou2013,Bui2016,Dunlop2018} are hierarchical models where each layer is described by a latent variable Gaussian process. That is, they denote a random function $g :\R^{p_0} \to \R^{p_{L+1}}$ of the form
$g(\bx) =(\bold f^{(L+1)}\circ\ldots \circ \bold f^{(1)})(\bx)$, where
 for each $l=1,\ldots,L+1$, the $l$-th GP layer $\bold f^{(l)}$ has width $p_l$ and is defined as follows:
\begin{align*}
\bold f^{(l)} & : \R^{p_{l-1}} \to \R^{p_l},
&
\bold f^{(l)}(\bold z) & =(f_1^{(l)}(\bold z),\ldots,f_{p_l}^{(l)}(\bold z))
\end{align*}
with $f_{j}^{(l)}\overset{\text{iid}}{\sim} \text{GP}(0,k^{(l)})$ for some kernel $k^{(l)}:\mathbb R^{p_{l-1}}\times\mathbb R^{p_{l-1}}\to\mathbb R$ of the $l$-th layer. As noted by a number of authors, a finite neural network with iid Gaussian weights is a deep Gaussian process. Similarly, for our finite model, \textit{conditionally on the set of variances} $(\lambda_{p_l,j}^{(l)})_{l=1,\ldots,L;j=1,\ldots,p_l}$, the model can be seen as a deep Gaussian process, whose kernels $k^{(l)}$ depends on $(\lambda_{p_l,j}^{(l)})_{j=1,\ldots,p_l}$. Assume $a^{(l)}=0$. In the infinite limit, as is apparent from \cref{th:multipleinputs}, we still have a deep Gaussian process, \textit{conditionally on the point processes} $\{\widetilde{\lambda}_{j}^{(l)}\}_{j\geq1}$ with mean measure $\rho^{(l)}$. The Gaussian process kernels are given by
\[
k^{(1)}(\mathbf{x},\mathbf{x}^{\prime}):=\sigma_{b}^{2}+\sigma_{v}^{2}\frac{\mathbf{x}^{T}\mathbf{x}^{\prime}}{\din},
\qquad
k^{(l+1)}(\mathbf{z},\mathbf{z}^{\prime}):=\sigma_{b}^{2}+\sigma_{v}^{2}\sum_{j\geq1}^{{}}\widetilde{\lambda}_{j}^{(l)}\phi(  z_{j})
\phi(  z_{j}')\ \text{for}\  l=1,\ldots,L.
\]

\paragraph{Neural tangent kernels.} An interesting direction would be to investigate the behaviour of the neural network during training by gradient descent when initialised with the MoGP model. In the iid Gaussian case, it has been shown that the evolution of the neural network can be described by a fixed kernel (neural tangent kernel) in the infinite-width limit~\cite{Jacot2018}, and that this kernel remains constant over training. We conjecture that in the MoGP case, when the L\'evy measure is non-trivial, we would obtain a random kernel in the infinite-width limit, where this kernel evolves over training.

\acks{We would like to extend our gratitude to the Action Editor, Dan Roy, and the two anonymous reviewers for their detailed and insightful feedback. In particular, their contributions significantly enhanced the clarity and rigour of \cref{th:multipleinputs}'s statement and proof. HL and PJ were funded in part by the National Research Foundation of Korea (NRF) grants NRF-2017R1A2B200195215 and NRF-2019R1A5A102832421. HL was funded in part by NRF grant NRF-2023R1A2C100584311. HY was supported by the Engineering Research Center Program through the National Research Foundation of Korea funded by the Korean government MSIT (NRF-2018R1A5A1059921) and also by the Institute for Basic Science (IBS-R029-C1). JL was supported by an Institute of Information \& communications Technology Planning
\& Evaluation (IITP) grant funded by the Korea government MSIT (No.2019-0-00075), Artificial Intelligence Graduate School Program (KAIST).}


\newpage

\appendix

\begin{center}
\huge \bf Appendices
\end{center}

\paragraph{Organisation of the appendices.} \Cref{app:background-all} contains background material on positive homogeneous functions, ReLU kernels, regularly varying functions, stable distributions, conditions for the convergence to infinitely divisible random variables and L\'evy measures.   \Cref{app:limit thms} gives and proves some limit theorems on independent triangular arrays. These results form the building blocks of the proofs of the main theorems and propositions, which are given in \Cref{app:proofs-all}.   \Cref{app:additionalresults} provides additional secondary theoretical results on the properties of small weights in our model, and on the infinite-width limit for multiple inputs in the symmetric $\alpha$-stable case. \Cref{app:examples_all} provides details of the derivations for the examples given in \cref{sec:intro,sec:examples} and additional properties concerning these examples. It also provides a general recipe to construct novel examples, and describes new concrete examples. Finally, \cref{app:Bayesian-experiment} provides additional experimental results, including stability with respect to depth and our experimental findings on convolutional neural networks.

\section{Background Material}
\label{app:background-all}

\subsection{Background on Positive Homogeneous Functions}
\label{app:homogeneous}
Let $\phi:\mathbb R\to\mathbb R$ be a positive homogeneous function. That is, for all $\alpha>0$,
\[
\phi(\alpha x)=\alpha\phi(x).
\]
Define $ C_{\mathrm{Lip}}:= \max(|\phi(1)|,|\phi(-1)|) $.
Then, $\phi$ is $C_{\mathrm{Lip}}$-Lipschitz continuous and satisfies $ |\phi(x)|\leq C_{\mathrm{Lip}} |x| $ for all $x$.
\begin{proof}
We first show that $|\phi(x)| \leq C_{\mathrm{Lip}} |x|$ for all $x$. When $x$ is 0, $\phi(0)=\alpha\phi(0)$ for all $\alpha>0$ and so $\phi(0)=0$. When $x$ is not 0,
$$
        |\phi(x)| = |x|\left|\phi\left(\frac{x}{|x|}\right)\right| \leq |x|\max(|\phi(1)|,|\phi(-1)|) = |x| C_{\mathrm{Lip}}.
$$
Next we show the Lipschitz continuity. Let $x,y \in \R$. If both $x$ and $y$ are nonnegative, then
\[
|\phi(x) - \phi(y)| = |x \phi(1) - y \phi(1)| = |\phi(1)||x-y| \leq C_{\mathrm{Lip}} |x - y|.
\]
If both $x$ and $y$ are nonpositive, then
\[
|\phi(x) - \phi(y)| = |(-x) \phi(-1) - (-y) \phi(-1)| = |\phi(-1)||x-y| \leq C_{\mathrm{Lip}} |x - y|.
\]
The remaining case is that one of $x$ and $y$ is positive and the other is negative. Without loss of generality, we may assume that $x$ is positive.
Then,
\begin{align*}
|\phi(x) - \phi(y)|
= \big| |x|\phi(1) - |y|\phi(-1)\big|
\leq |x||\phi(1)| + |y||\phi(-1)|
\leq C_{\mathrm{Lip}}(|x|+|y|) = C_{\mathrm{Lip}}|x-y|.
\end{align*}
\end{proof}

\subsection{Background on ReLU Kernels}\label{app:relu}

Following \cite{Cho2009}, we have, for $\Sigma=
\begin{pmatrix}
\Sigma_{11} & \Sigma_{12}\\
\Sigma_{21} & \Sigma_{22}%
\end{pmatrix} $, $
\begin{pmatrix}
X\\
Y
\end{pmatrix}
\sim\mathcal{N}(0,\Sigma)$ and
$\alpha \geq 0$,
\[
\bfE\left[  \max(0,X)^{\alpha}\max(0,Y)^{\alpha}\right]  =\frac{1}{2\pi}\left(
\Sigma_{11}\Sigma_{22}\right)  ^{\alpha/2}J_{\alpha}(\theta)
\]
where
\[
J_{\alpha}(\theta)=\Gamma(\alpha+1)\sin^{2\alpha+1}(\theta)\int_0^{\pi/2} \frac{\cos^\alpha(x)}{(1-\cos(\theta)\cos(x))^{\alpha+1}}dx
\]
with $$\theta=\arccos(\rho)\quad\text{and}\quad \rho=\frac{\Sigma_{12}}{\sqrt{\Sigma
_{11}\Sigma_{22}}}.$$
Furthermore, if $\alpha\in\N$,
\[
J_{\alpha}(\theta)
  = (-1)^{\alpha}\sin^{2\alpha+1}(\theta)
    \left(\frac{1}{\sin(\theta)}\frac{\partial}{\partial\theta}\right)^{\alpha}
    \left(\frac{\pi-\theta}{\sin(\theta)}\right)
\]
In particular, we have%
\begin{align*}
        J_{0}(\theta)  & =\pi-\theta\\
        J_{1}(\theta)  & =\sin(\theta)+(\pi-\theta)\cos(\theta)\\
        J_{2}(\theta)  & =3\sin(\theta)\cos(\theta)+(\pi-\theta)(1+2\cos^{2}(\theta)).
\end{align*}
For $\alpha=1/2$, we show that
\[
        J_{1/2}(\theta)=\sqrt{\frac{\pi}{2}}\left (2\ellipE\left(\frac{\cos(\theta)+1}{2}\right)-(1-\rho)\ellipK\left(\frac{\cos(\theta)+1}{2}\right)   \right )
\]
where $\ellipK$ and  $\ellipE$ are respectively the complete elliptical integrals of the first and second kind, defined by
\begin{align}
\ellipK(m)&=\int_0^{\pi/2}(1-m\sin^2(t))^{-1/2}dt=\int_0^1 \frac{dt}{\sqrt{(1-t^2)(1-mt^2)}}\label{eq:ellipK}\\
\ellipE(m)&=\int_0^{\pi/2}(1-m\sin^2(t))^{1/2}dt=\int_0^1 \frac{\sqrt{1-mt^2}}{\sqrt{1-t^2}}dt\label{eq:ellipE}
\end{align}
which can be computed efficiently using the arithmetic-geometric mean. Writing the above expressions in terms of the correlation $\rho$, we obtain
\begin{align*}
&\bfE\left[  \max(0,X)^{\alpha}\max(0,Y)^{\alpha}\right] =\frac{1}{2\pi}\left(
\Sigma_{11}\Sigma_{22}\right)  ^{\alpha/2} \kappa_\alpha(\rho)
\end{align*}
where
\begin{align}
        &\kappa_\alpha(\rho)=J_\alpha(\arccos(\rho)) =\left\{
\begin{tabular}
[c]{ll}%
        $\frac{\pi}{2}+\arcsin(\rho)$ & if $\alpha=0$\\
$
\sqrt{\frac{\pi}{2}}\left (2\ellipE(\frac{\rho+1}{2})-(1-\rho)\ellipK(\frac{\rho+1}{2})   \right )  $ &
if $\alpha=1/2$\\
$
        \sqrt{1-\rho^{2}}+\left(  \frac{\pi}{2}+\arcsin(\rho)\right)  \rho  $ &
if $\alpha=1$\\
$  3\sqrt
        {1-\rho^{2}}\rho+\left(  \frac{\pi}{2}+\arcsin(\rho)\right)  (1+2\rho
^{2})  $ & if $\alpha=2$%
\end{tabular}
\right.\label{eq:kappa_alpha}
\end{align}
using the identities $\pi-\arccos(x)=\frac{\pi}{2}+\arcsin(x)$ and $\sin(\arccos
(x))=\cos(\arcsin(x))=\sqrt{1-x^{2}}$.

\bigskip

We also have%
\begin{align*}
\bfE\left[  \max(0,X)^{2\alpha}\right]  &=\frac{1}{2}\bfE\left[  X^{2\alpha}\right]=\left(  \Sigma
_{11}\right)  ^{\alpha}\frac{2^{\alpha-1}\Gamma(\alpha+1/2)}{\Gamma(1/2)}=\left\{
\begin{tabular}
[c]{ll}%
$\frac{1}{2}$ & if $\alpha=0$\\
$\sqrt{\frac{\Sigma_{11}}{2\pi}}$ & if $\alpha=1/2$\\
$\frac{\Sigma_{11}}{2}$ & if $\alpha=1$\\
$\frac{3}{2}\Sigma_{11}^{2}$ & if $\alpha=2$%
\end{tabular}
\right. .
\end{align*}

\begin{proof}
All of the above results are from  \cite{Cho2009}, except for $J_{1/2}$ (or equivalently $\kappa_{1/2}$). Write $\kappa_{1/2}(\rho)=\Gamma(3/2)(1-\rho^2)f(\rho)$, where
\begin{align*}
        f(\rho)&:=\int_0^{\pi/2} \frac{\sqrt{\cos(x)}}{(1-\rho \cos(x))^{3/2}}dx\\
&=\int_0^1 \frac{\sqrt{v}}{(1-\rho v)^{3/2}\sqrt{1-v^2}}dv.
\end{align*}
Integrating with respect to $\rho$ and using Fubini's theorem,
\begin{align}
\int_a^\rho f(u)du=2\int_{0}^{1}\frac{1}{\sqrt{v}(1-\rho v)^{1/2}\sqrt{1-v^{2}}}dv+const.\label{eq:int_f}
\end{align}

Using the change of variables $u = \sqrt{\frac{1+\frac{1}{v}}{2}}$, we have that
\begin{eqnarray*}
\int_{0}^{1}\frac{1}{\sqrt{v}(1-\rho v)^{1/2}\sqrt{1-v^{2}}}dv &=& \int_1^\infty \frac{4 u}{(2u^2-1)^2\sqrt{\frac{1}{2u^2-1}} \sqrt{1-\frac{\rho}{2u^2-1}} \sqrt{1-\frac{1}{(2u^2-1)^2}}} du\\
&=& \int_1^\infty \frac{4 u}{\sqrt{2} \sqrt{u^2-\frac{\rho+1}{2}} \sqrt{2u^2-2} \sqrt{2u^2}}du \\
&=& \sqrt{2} \int_1^\infty \frac{du}{\sqrt{u^2-\frac{\rho+1}{2}} \sqrt{u^2-1}} \\
&=& \sqrt{2} \ellipK\left(\frac{\rho+1}{2}\right).
\end{eqnarray*}

Differentiating \cref{eq:int_f} with respect to $\rho$ gives
$$
f(\rho)=\frac{\sqrt{2}}{1-\rho^2}\left (2\ellipE\left(\frac{\rho+1}{2}\right) - (1-\rho) \ellipK\left(\frac{\rho+1}{2}\right) \right ).
$$
\end{proof}

\subsection{Useful Lemmas on Regularly Varying Random Variables}
\label{app:backgroundRV}

A nonnegative random variable $X$ is said to be regularly varying with index
$\alpha\geq 0$ if and only if
\[
        \Pr(X>x)\overset{x\to\infty}{\sim} x^{-\alpha} L(x)
\]
where $L$ is a slowly varying function.
That is, its survival
function (a.k.a. complementary cumulative distribution function) $\overline
F(x):=\Pr(X>x)$ is regularly varying with index $-\alpha$.

\begin{lemma}\label{lem:Jessen4.1i}
\cite[Lemma 4.1.(i)]{Jessen2006} Assume that $X_{1}$ and $X_{2}$ are
independent nonnegative random variables and that $X_{1}$ is regularly
varying with index $\alpha>0$. If either $X_{2}$ is regularly varying with
index $\alpha>0$ or $\Pr(X_{2}>x)=o(\Pr(X_{1}>x))$, then $X_{1}X_{2}$ is
regularly varying with index $\alpha>0$.
\end{lemma}

\begin{lemma}\label{lem:Jessen4.1iv}
\cite[Lemma 4.1.(iv)]{Jessen2006} Let $X_{1}$, \ldots, $X_{{p}}$ be iid
nonnegative random variables with survival function satisfying $\Pr(X_{1}>x) \overset{x\to\infty}{\sim} cx^{-\alpha}$ for some $c>0$. Then,
\[
        \Pr(X_{1}\ldots X_{{p}}>x) \overset{x\to\infty}{\sim}\frac{\alpha^{{p}-1}c^{{p}}}{({p}-1)!}x^{-\alpha}\log^{{p}-1}x.
\]
\end{lemma}

\begin{lemma}\label{lem:Jessen4.2}
\cite[Lemma 4.2]{Jessen2006} Assume that $X_{1}$ and $X_{2}$ are nonnegative
independent random variables and that $X_{1}$ is regularly varying with index
$\alpha>0$.

\begin{enumerate}
\item If there exists $\epsilon>0$ such that $\bfE[X_{2}^{\alpha+\epsilon}]<\infty$, then
\begin{align}
        \label{eq:RVprop}\Pr(X_{1}X_{2}>x)\overset{x\to\infty}{\sim} \bfE[X_{2}^{\alpha}]\Pr(X_{1}>x).
\end{align}
\item If $\Pr(X_{1}>x)\overset{x\to\infty}{\sim} cx^{-\alpha}$ and $\bfE[X_{2}^{\alpha}]<\infty$, then
\cref{eq:RVprop} holds.
\end{enumerate}
\end{lemma}

\begin{proposition}\label{prop:idRVregvar}
        \cite[Theorem 8.2.1 p.341]{Bingham1989} Let $X\sim \operatorname{ID}(a,\rho)$ be a
nonnegative infinitely divisible random variable. Let $\overline F(x)=1-F(x)$
be its survival function, where $F$ is the cdf of $X$. Then, for all $\alpha
\geq0$,  the tail L\'evy intensity $\overline\rho$ is regularly varying with index $-\alpha$ if and only if $\overline F$ is also. Furthermore,
in that case,
\[
        \overline\rho(x) \overset{x\to\infty}{\sim} \overline F(x).
\]

\end{proposition}

\begin{proposition}\label{prop:resnickProp5}
\cite[Proposition 5(iii)]{Resnick2005} Let $G$ be a regularly varying function with index $\alpha \in \R$,
and $(a_{{p}})_{p}$ and $(b_{{p}})_{p}$ be two sequences that satisfy
        $0<a_{{p}}\rightarrow\infty$, $0<b_{{p}}\rightarrow\infty$ and $a_{{p}}\overset{p \to \infty}{\sim} cb_{{p}}$
for some $0<c<\infty$. Then,
\[
        G(a_{{p}})\overset{p\to\infty}{\sim} c^{\alpha}G(b_{{p}}).
\]

\end{proposition}

\begin{lemma}\cite[Lemma 2, VIII.8]{Feller1971}\label{lem: feller slowly varying bound}
	If L is slowly varying at infinity, then for any $ \delta > 0$, there exists $ x_{0} $ such that $ x^{-\delta} < L(x) < x^{\delta} $ for all $ x > x_0 $.
\end{lemma}

\subsection{Background on Positive Stable Random Variables}\label{app:positive stable}

A (possibly degenerate) positive strictly stable random variable $X\sim \stable(\alpha,\gamma)$ with stability exponent $\alpha\in(0,1]$ and scale parameter $\gamma>0$ has Laplace transform
$$
\bfE[e^{-tX}]=e^{-(\gamma t)^\alpha},\text{ for }t\geq0.
$$
It satisfies, for any $n\geq1$, $\sum_{i=1}^n X_i \overset{\cvd}{=} n^{1/\alpha}X$ where $X_1$, $\ldots$, $X_n$ are iid copies of $X$, and is an important example of an infinitely divisible random variable. If $\alpha=1$, $X=\gamma$ is degenerate at $\gamma$; for $\alpha\in(0,1)$, it is non-degenerate with support $(0,\infty)$. In general, $X$ is infinitely divisible with
$$
X\sim\left \{
\begin{array}{ll}
        \id(0, \rho_{\operatorname{stable}}(\,\cdot\,;\alpha,\gamma/\Gamma(1-\alpha)^{1/\alpha})) & \text{if }\alpha\in(0,1) \\
  \id(\gamma, 0) & \text{if }\alpha=1
\end{array}\right .
$$
where $\rho_{\operatorname{stable}}(dx;\alpha,c)$ denotes the following alpha--stable L\'evy measure on $(0,\infty)$ with stability exponent $\alpha\in(0,1)$ and parameter $c>0$:
\begin{align}
\rho_{\operatorname{stable}}(dx;\alpha,c):=\alpha c^\alpha x^{-\alpha-1}\ind_{\{x>0\}} dx.\label{eq:stablemeasure1}
\end{align}
In the special case $\alpha=\frac{1}{2}$, we have $\stable(1/2,\gamma) =\IG(\frac{1}{2},\frac{\gamma}{4})$, that is, the stable distribution with scale parameter $\gamma$ corresponds to the inverse gamma distribution with shape $1/2$ and scale $\gamma/4$.

\begin{remark}\label{stable remark}
Standard definitions of positive stable random variables are given for non-degenerate random variables, with $\alpha\in(0,1)$. See e.g. \cite{samorodnitsky1994stable,Janson2011}. We include here the degenerate case $\alpha=1$, as this case relates to the Gaussian process limit, as we show in~\cref{sec: inf width limit multi inputs}. For $\alpha\in(0,1)$, our parameterisation $\stable(\alpha,\gamma)$ corresponds to the standard four-parameters parameterisation~\cite[Theorem 3.3]{Janson2011} $S_\alpha(\gamma_0,\beta_0,\delta_0)$, with $\beta_0=1$, $\delta_0=0$ and $\gamma_0=\gamma (\cos(\pi\alpha/2))^{1/\alpha}$.
\end{remark}

\subsection{Background on L\'evy Measures on $(0,\infty)$}

The generalised gamma L\'evy measure~\cite{Hougaard1986,Brix1999} has three parameters $\eta>0$, $\alpha\in(-\infty,1)$ and $\tau>0$ if $\alpha\leq 0$ and $\tau\geq0$ if $\alpha\in(0,1)$. It is defined by
\begin{equation}
\rho_{\operatorname{gg}}(dx;\eta,\alpha,\tau)=\eta\frac{1}{\Gamma(1-\alpha)}x^{-\alpha-1}e^{-\tau x}dx.
\label{eq:generalisedgammameasure}
\end{equation}
It is finite when $\alpha<0$, and infinite otherwise. It admits as special cases the gamma measure when $\alpha=0$ and the positive stable measure if $\alpha\in(0,1)$ and $\tau=0$. We denote the gamma measure and (scaled) positive stable measures as
\begin{align}
\rho_{\operatorname{gamma}}(dx;\eta,\tau):=\rho_{\operatorname{gg}}(dx;\eta,0,\tau)&=\eta x^{-1}e^{-\tau x}dx,\label{eq:gammameasure}\\
\rho_{\operatorname{stable}}(dx;\alpha,c):=\rho_{\operatorname{gg}}(dx;\alpha c^\alpha\Gamma(1-\alpha),\alpha,0)&=\alpha c^\alpha x^{-\alpha-1}dx.\label{eq:stablemeasure2}
\end{align}
Here $\eta,\tau > 0$ for the gamma measure, and
$\alpha\in(0,1)$ and $c>0$ for the stable L\'evy measure. The $c$ is called a scaling parameter. Note that $\id(0,\rho_{\operatorname{gamma}}({} \cdot {};\eta,\tau))=\gammadist(\eta,\tau)$. Additionally, if $X\sim \id(0,\rho_{\operatorname{stable}}({} \cdot {};\alpha,c))$, then $X$ is a positive stable random variable with parameter $\alpha\in(0,1)$, with Laplace transform $\bbE[e^{-tX}]=e^{-(\gamma t)^\alpha}$ for $\gamma := c \cdot \Gamma(1-\alpha)^{1/\alpha}$.

The stable beta L\'evy measure with parameter $\eta>0$, $\alpha\in(-\infty,1)$, $\phi>-\alpha$ is defined as~\cite{Hjort1990,Thibaux2007,Teh2009}
\begin{equation}
\rho_{\operatorname{sb}}(dx;\eta,\alpha,\phi)=\eta\frac{\Gamma(1+\phi)}{\Gamma(1-\alpha)\Gamma(\phi+\alpha)}x^{-\alpha-1}(1-x)^{\phi+\alpha-1}\ind_{\{x\in(0,1)\}}dx.
\label{eq:stablebetameasure}
\end{equation}
It is infinite if $\alpha\geq 0$ and finite otherwise. If $\alpha=0$, this is known as the beta L\'evy measure.

The scaled stable beta measure has the additional scale parameter $c>0$, and is defined as
\begin{equation}
\rho_{\operatorname{ssb}}(dx;\eta,\alpha,\phi,c)=\eta\frac{\Gamma(1+\phi)c^\alpha}{\Gamma(1-\alpha)\Gamma(\phi+\alpha)}x^{-\alpha-1}(1-x/c)^{\phi+\alpha-1}\ind_{\{x\in(0,c)\}}dx.
\label{eq:scaledstablebetameasure}
\end{equation}

The following proposition derives some connections between the scaled stable beta measure and the generalised gamma measure, and some invariance property of the scaled stable measure. Similar expressions were obtained by \citet{Griffin2017a} for constructing dependent completely random measures.

\begin{proposition}[Gamma-function integral formulas]\label{prop:gamma formulas}
Let $\gammadist(x;a,b)$ denote the pdf of a Gamma random variable with shape parameter $a$ and inverse scale parameter $b$. Let $\kappa>0$, $\eta>0$, $\alpha\in(-\infty,1)$, $c>0$, $\phi>\max(0,-\alpha)$ and $b>0$. Then,
\begin{align}
\rho_{\operatorname{gg}}\left(dx;\frac{\eta\kappa \phi}{(bc)^\alpha},\alpha,bc\right)&=dx\times\kappa\int_0^\infty \frac{1}{z}\gammadist\left(\frac{x}{z};\phi,b\right)\rho_{\operatorname{ssb}}(dz;\eta,\alpha,\phi,1/c),\\
\rho_{\operatorname{stable}}\left(dx;\alpha, \frac{c}{b}\left ( \frac{\kappa\Gamma(\phi+\alpha)}{\Gamma(\phi)} \right )^{1/\alpha}\right)&=dx\times\kappa\int_0^\infty \frac{1}{z}\gammadist\left(\frac{x}{z};\phi,b\right)\rho_{\operatorname{stable}}(dz;\alpha,c).
\end{align}
\end{proposition}

\begin{proof}
Let
\begin{align*}
\nu(x)
& {} := \kappa\int_0^\infty \frac{1}{z}\gammadist\left(\frac{x}{z};\phi,b\right)\rho_{\operatorname{ssb}}(dz;\eta,\alpha,\phi,1/c)\\
& {} =\eta\kappa\frac{c^{-\alpha}b^{\phi}\Gamma(1+\phi)}{\Gamma(\phi)\Gamma(\phi+\alpha)\Gamma(1-\alpha)}%
\int_{0}^{1/c}z^{-1}(x/z)^{\phi-1}e^{-xb/z}z^{-\alpha-1}(1-cz)^{\alpha+\phi
-1}dz\\
& {} =\eta\kappa\frac{c^{-\alpha}b^{\phi}\phi}{\Gamma(\phi+\alpha)\Gamma(1-\alpha
)}x^{\phi-1}\int_{0}^{1/c}z^{-\phi-\alpha-1}e^{-xb/z}(1-cz)^{\alpha+\phi-1}dz.
\end{align*}

Using the change of variable $u=\frac{1}{z}-c$ so that $z=\frac{1}{u+c}$
and $dz=\frac{-du}{(u+c)^{2}}$, we obtain%
\begin{align*}
\nu(x) &  =\eta\kappa\frac{c^{-\alpha}b^{\phi}\phi}{\Gamma(\phi+\alpha)\Gamma(1-\alpha
)}x^{\phi-1}   \int_{0}^{\infty}(u+c)^{\phi+\alpha-1}e^{-xb(u+c)}\left(
\frac{u}{u+c}\right)  ^{\alpha+\phi-1}du\\
&  =\eta\kappa\frac{c^{-\alpha}b^{\phi}\phi}{\Gamma(\phi+\alpha)\Gamma(1-\alpha
)}x^{\phi-1} \int_{0}^{\infty}e^{-xb(u+c)}u^{\alpha+\phi-1}du\\
&  =\eta\kappa\frac{c^{-\alpha}b^{\phi}\phi}{\Gamma(\phi+\alpha)\Gamma(1-\alpha
)}x^{\phi-1}e^{-bcx}\frac{\Gamma(\alpha+\phi)}{(bx)^{\alpha+\phi}}\\
&  =\eta\kappa\frac{c^{-\alpha}b^{-\alpha}\phi}{\Gamma(1-\alpha)}x^{-\alpha-1}e^{-bcx},%
\end{align*}
which is the density of the generalised gamma L\'evy measure with parameters $(\eta\kappa (bc)^{-\alpha}\phi,\alpha,bc)$ at $x$.

Similarly,
\begin{align*}
\nu_2(x)
& {} := \kappa\int_0^\infty \frac{1}{z}\gammadist\left(\frac{x}{z};\phi,b\right)\rho_{\operatorname{stable}}(dz;\alpha,c)\\
& {} =\kappa\frac{c^\alpha \alpha b^{\phi}}{\Gamma(\phi)}%
\int_{0}^{\infty}z^{-1}(x/z)^{\phi-1}e^{-xb/z}z^{-\alpha-1}dz\\
& {} =\kappa\frac{c^\alpha \alpha b^{\phi}}{\Gamma(\phi)}%
x^{\phi-1}\int_{0}^{\infty}e^{-xbu}u^{\phi+\alpha-1}du\\
& {} =\kappa\frac{c^\alpha\alpha\Gamma(\phi+\alpha) }{\Gamma(\phi)b^{\alpha}}%
x^{-\alpha-1}=\left ( \frac{\kappa^{1/\alpha} c \Gamma(\phi+\alpha)^{1/\alpha} }{\Gamma(\phi)^{1/\alpha}b} \right )^\alpha \alpha x^{-\alpha-1}.
\end{align*}

\end{proof}

The following are corollaries of the above proposition, with $\kappa=\phi=b=\frac{1}{2}$, in combination with \cref{lemma:idlaplaceproduct}.
\begin{corollary}\label{cor:scaledstablebeta_sumxy}
Let $X_1,X_2,\ldots,$ be iid standard normal random variables, and $(\xi_i)_i$ be the points of a Poisson point process with mean measure $\rho_{\operatorname{ssb}}(dz;\eta,\alpha,1/2,1/c)$ for $\eta>0$, $c>0$ and $\alpha>-\frac{1}{2}$.  Then,
\begin{align*}
\sum_{i\geq 1} \xi_i&\sim \id(0,\rho_{\operatorname{ssb}}({} \cdot {};\eta,\alpha,1/2,1/c)),\\
\sum_{i\geq 1} \xi_i\max(0,X_i)^2&\sim \id\left(0,\rho_{\operatorname{gg}}\left({} \cdot {};\eta 2^{\alpha-2}c^{-\alpha},\alpha,c/2\right)\right).
\end{align*}
In particular, if additionally $\alpha=0$,
\begin{align*}
\sum_{i\geq 1} \xi_i\max(0,X_i)^2&\sim \gammadist\left(\frac{\eta}{4},\frac{c}{2}\right).
\end{align*}
\end{corollary}

\begin{corollary}\label{cor:stable_sumxy}
Let $X_1,X_2,\ldots,$ be iid standard normal random variables, and $(\xi_i)_i$ be the points of a Poisson point process with mean measure $\rho_{\operatorname{stable}}(dz;\alpha,c)$ for some $\alpha\in(0,1)$ and $c>0$.  Then,
\begin{align*}
\sum_{i\geq 1} \xi_i&\sim \id(0,\rho_{\operatorname{stable}}({} \cdot {};\alpha,c)),\\
\sum_{i\geq 1} \xi_i\max(0,X_i)^2&\sim \id\left(0,\rho_{\operatorname{stable}}\left({} \cdot {};\alpha,2c\left ( \frac{\Gamma(1/2+\alpha)}{2\sqrt{\pi}} \right )^{1/\alpha}\right) \right),
\end{align*}
so that
$$
\sum_{i\geq 1} \xi_i\max(0,X_i)^2 \overset{\cvd}{=} 2\left ( \frac{\Gamma(1/2+\alpha)}{2\sqrt{\pi}} \right )^{1/\alpha}   \sum_{i\geq 1} \xi_i .
$$
\end{corollary}

\newpage

\section{Some Limit Theorems on Independent Triangular Arrays}\label{app:limit thms}

Throughout the paper, we use a necessary and sufficient condition for sums of random variables to  converge to an infinitely divisible random variable, and also a sufficient condition for such convergence when all the random variables involved have densities. These two conditions are summarised in the following theorem. All the proofs of the examples in \cref{sec:examples,app:examples} rely on these conditions; details are given in~\cref{app:additionalderivationsExamples}.

\begin{theorem}[Necessary and sufficient conditions for convergence to $\id(a,\rho)$]\label{th:convergenceid}~\\
Let $(X_{p,j})_{p\geq 1,j=1,\ldots,p}$ be a triangular array of nonnegative real random variables, where for each $p \geq 1$,  the random variables $X_{p,1}, \ldots, X_{p,p}$ are iid. Let $a\geq 0$ and $\rho$ be a L\'evy measure on $(0,\infty)$.
Then, $\sum_{j=1}^p X_{p,j}\overset{\cvd}{\to}\id(a,\rho)$ if and only if the following two conditions hold:
\begin{itemize}
\item[(i)] $p\Pr(X_{p,1}>x)\to \overline\rho(x)$ for all $x>0$ such that $\rho(\{x\})=0$, and
\item[(ii)] $p\bbE [X_{p,1}\ind_{\{X_{p,1}\leq h\}}]\to a + \int_0^h x\rho(dx)$ for any $h>0$ with $\rho(\{h\}) = 0$.
\end{itemize}
If every $X_{p,1}$ is an absolutely continuous random variable with density $f_p$, and $\rho$ is absolutely continuous with density $\varrho$ and support $S$, then condition (i) is implied by the following three conditions:
\begin{enumerate}
\item[(a)] $pf_p(x)\to \varrho(x)$ for all $x>0$,
\item[(b)] for any $x_0>0$, there exists $C_{x_0}$ such that $\frac{pf_p(x)}{\varrho(x)}\leq C_{x_0}$ for all $x\in[x_0,\infty)\cap S$, and
\item[(c)] for any $x_0>0$, $\int_{[x_0,\infty)\backslash S}f_p(x)dx=o(1/p)$.
\end{enumerate}
\end{theorem}
In \cref{th:convergenceid}, we have included the second part since in practice, for continuous random variables, conditions (a-c) will be easier to check than the condition (i).
\begin{proof}
The first part of the theorem is  a corollary of Theorem 15.28 in \cite{Kallenberg2002}. We focus on the second part for absolutely continuous random variables.

Let $(0,\infty]=(0,\infty)\cup\{\infty\}$ denote the set of positive reals with the addition of $\infty$, which is called the set of extended positive reals. Note that for any $a>0$, $[a,\infty]$ is a compact set of $(0,\infty]$. Let $C^+_K((0,\infty])$ denote the set of continuous functions $f:(0,\infty]\to \mathbb R_+$ with compact support. Note that the functions are necessarily bounded as $f(\infty)\in\mathbb R_+$.

Let $(X_{{p},j})$ be a triangular array of random variables such that for every $p$, $(X_{{p},j})_{j=1,\ldots,p}$ is an iid sequence
of random variables from $\mu_{p}$. By \cite[Theorem 15.29]{Kallenberg2002}, $p\Pr(X_{p,1}>x)\to \overline\rho(x)$ for all $x>0$ such that $\rho(\{x\})=0$ is
equivalent to
\[
\eta_{{p}}:=\sum_{j=1}^{{p}}\delta_{X_{{p},j}}\overset{\cvd}{\rightarrow}\eta\text{ on }(0,\infty]
\]
where $\eta$ is a Poisson random measure with mean measure $\rho$. This is
equivalent to showing that
\[
\bbE [e^{-\eta_{{p}}(g)}]\rightarrow\bbE [e^{-\eta(g)}]
\]
for all $g\in C_{K}^{+}((0,\infty])$. Pick $g\in C_{K}^{+}((0,\infty])$. Let
$S\subseteq(0,\infty)$ be the support of $\rho$. We have
\begin{align*}
\bbE [e^{-\eta_{{p}}(g)}]
& =\bbE \left[e^{-\sum_{j=1}^{{p}}g(X_{{p},j})}\right] \\
& =\bbE \left[e^{-g(X_{{p},1})}\right]^{{p}} \\
& =\left(  \int_{0}^{\infty}e^{-g(x)}f_{{p}}(x)dx\right)  ^{{p}}\\
& =\left(  1-\int_{0}^{\infty}(1-e^{-g(x)})f_{{p}}(x)dx\right)  ^{{p}}\\
& =\left(  1-\frac{1}{{p}}\int_{0}^{\infty}(1-e^{-g(x)}){p}f_{{p}}(x)dx\right)
^{p}\\
& =\left(  1-\frac{1}{{p}}\left[  \int_{S}(1-e^{-g(x)})\frac{{p}f_{{p}}%
(x)}{\varrho(x)}\varrho(x)dx+\int_{\mathbb{R}_{+}\backslash S%
}(1-e^{-g(x)}){p}f_{{p}}(x)dx\right]  \right)  ^{{p}}.
\end{align*}
Since $g$ has compact support on $(0,\infty]$, there
exists $x_{0}>0$ such that $g(x)=0$ for $x<x_{0}$. Then, by assumption,
$\frac{{p}f_{{p}}(x)}{\varrho(x)}\leq C_{x_{0}}$ for all $x\in S \cap[x_{0},\infty)$ and ${p} \geq 1$. Also, again by assumption, $\frac{{p}f_{{p}}(x)}{\varrho(x)}\rightarrow 1$ as ${p} \rightarrow \infty$. What we have proved so far lets us use the dominated convergence theorem and derive the following convergence:
\[
\int_{S\cap [x_{0},\infty)}(1-e^{-g(x)})\frac{{p}f_{{p}}(x)}%
{\varrho(x)}\varrho(x)dx\rightarrow\int_{S\cap[x_{0},\infty
)}(1-e^{-g(x)})\varrho(x)dx.
\]
Additionally, $\int_{\mathbb{R}_{+}\backslash S}(1-e^{-g(x)}){p}f_{{p}}(x)dx\leq\int_{\mathbb{R}_{+}\backslash S}{p}f_{{p}}(x)dx=o(1)$.
Hence,
\begin{align*}
& \int_{S}(1-e^{-g(x)})\frac{{p}f_{{p}}(x)}{\varrho(x)}\varrho
(x)dx+\int_{\mathbb{R}_{+}\backslash S}(1-e^{-g(x)}){p}f_{{p}}(x)dx\\
& \ \ \ \  {} = \int_{S \cap [x_0,\infty)}(1-e^{-g(x)})\frac{{p}f_{{p}}(x)}{\varrho(x)}\varrho
(x)dx+\int_{\mathbb{R}_{+}\backslash S}(1-e^{-g(x)}){p}f_{{p}}(x)dx\\
& \ \ \ \ {} \rightarrow\int_{S}(1-e^{-g(x)})\varrho(x)dx+0=\int%
_{0}^{\infty}(1-e^{-g(x)})\varrho(x)dx.
\end{align*}
Recall that for any real sequence $(a_{{p}})_{{p}\geq1}$ converging to $a$, we
have $(1-\frac{a_{{p}}}{{p}})^{{p}}\rightarrow e^{-a}$. Thus,
\[
\left(  1-\frac{1}{{p}}\int_{S}(1-e^{-g(x)})\frac{{p}f_{{p}}(x)}%
{\varrho(x)}\varrho(x)dx\right)  ^{{p}}\rightarrow e^{-\int_{0}^{\infty
}(1-e^{-g(x)})\varrho(x)dx}=\bfE[e^{-\eta(g)}].
\]
\end{proof}

\begin{proposition}[Extremes of triangular arrays and infinite divisibility]\label{prop:extremesID}
	Let
	\[
	(X_{p,j})_{p\geq 1,j=1,\ldots,p}
	\]
	be a triangular array of independent nonnegative real random variables such that for each $p$, $(X_{p,j})_{j=1,\ldots,p}$ are iid. Assume $\sum_{j=1}^p X_{p,j}\overset{\cvd}{\to}\id(a,\rho)$ for some $a\geq 0$ and some L\'evy measure $\rho$ on $(0,\infty)$. Let $$\overline\rho^{-1}(u):=\inf\{x > 0 : \overline\rho(x) < u\},$$ the inverse tail L\'evy intensity of $\rho$.
	For each $p\geq 1$, let $X_{p,(1)}\geq X_{p,(2)}\geq \ldots\geq X_{p,(p)}$ denote the order statistics of $(X_{p,j})_j$.
	Then, the asymptotic behaviour of $X_{p,(k)}$, as $p\to\infty$, is solely characterised by $\rho$ (not $a$) with
	$ X_{p,(1)}\overset{\cvp}{\to} 0$ if $\rho$ is trivial, and if   $\rho$ is non-trivial,
	\begin{align*}
		\Bigr(X_{p,(k)}\Bigr)_{k\ge 1}\overset{\cvd}{\to} \Bigr(\overline\rho^{-1}(G_k)\Bigr)_{k\ge 1}
	\end{align*}
	with $(G_k)_{k\ge 1}$ being ordered points of a standard rate one Poisson process on $(0,\infty)$ with $G_k\sim\gammadist(k, 1)$. In particular,
	for non-trivial $\rho$, $\overline\rho^{-1}(G_k)$ is a nonnegative random variable, non-degenerate at $0$, with cumulative density function $F_k$  defined by
	\[
	F_k(x)= e^{-\overline\rho(x)}\sum_{i={0}}^{{k-1}} \frac{\overline\rho(x)^i   }{i!}
	\quad \text{for any $x>0$ with $\rho(\{x\})=0$}
	\]
	and
	\[
	F_k(0)=
	\begin{cases}
		e^{-\overline\rho(0)}\sum_{i={0}}^{{k-1}} \frac{\overline\rho(0)^i   }{i!} & \text{if $\rho$ is finite}, \\
		0 & \text{if $\rho$ is infinite}.
	\end{cases}
	\]
\end{proposition}

\begin{proof}
	From \cite[Theorem 15.29]{Kallenberg2002}, $\sum_{j=1}^p X_{p,j}\overset{\cvd}{\to}\id(a,\rho)$ implies that for any precompact Borel subsets $B_1,\ldots,B_k$ of the extended real half-line $(0,\infty]=(0,\infty)\cap\{\infty\}$,
	$$
	\Bigr(\#\{j\mid X_{p,j}\in B_i \}\Bigr)_{i=1}^k \cvdto \Bigr(\eta(B_i)\Bigr)_{i=1}^k
	$$
	where $\eta$ is a Poisson random measure with mean measure $\rho$. In particular, $\bigr(X_{p,(j)}\bigr)_{j= 1}^k$ converges in distribution to the joint distribution of the first $k$ arrival times, going backwards in time from $\infty$, of a Poisson process with intensity $\rho$. This is because
	for $$0<x_k<y_{k}<x_{k-1}<\cdots <x_1<y_1<x_0=\infty$$ such that $\rho(\{x_k,y_k,\ldots,x_1,y_1\})=0$,
	\begin{align*}
                & \Pr\left(\bigcap_{j=1}^k \{X_{p,(j)}\in (x_j,y_{j})\}\right)
                \\
                &\qquad {} \to\Pr\left(\{\eta(x_k,y_{k})\ge 1\}\cap\bigcap_{j=1}^{k-1}\{\eta(x_j,y_{j})=1\}\cap\bigcap_{j=1}^{k}\{\eta(y_j,x_{j-1})=0\}\right)\\
                &\qquad {} = \bigr(1-e^{-[\overline\rho(x_k)-\overline\rho(y_k)]}\bigr)\prod_{j=1}^{k-1}e^{-[\overline\rho(x_j)-\overline\rho(y_j)]}[\overline\rho(x_j)-\overline\rho(y_j)]\prod_{j=1}^{k}e^{-[\overline\rho(y_j)-\overline\rho(x_{j-1})]}\\
                &\qquad {} = \bigr(1-e^{-[\overline\rho(x_k)-\overline\rho(y_k)]}\bigr)e^{-\overline\rho(y_k)}\prod_{j=1}^{k-1}[\overline\rho(x_j)-\overline\rho(y_j)]\\
		&\qquad {} =\bigr(e^{-\overline\rho(y_k)}-e^{-\overline\rho(x_k)}\bigr)\prod_{j=1}^{k-1}[\overline\rho(x_j)-\overline\rho(y_j)]
	\end{align*}
and the final expression on the right-hand side above can be seen to be the joint distribution of those $k$ arrival times.
	It remains to calculate the limiting distribution of the marginal $X_{p,(k)}$.
	
	For any $x>0$ such that $\rho(\{x\})=0$,
	\begin{align*}
		\Pr(X_{p,(k)}\leq x)&=\sum_{i=0}^{k-1} \Pr(\#\{j\mid X_{p,j}> x \}=i)\\
		& {} \to\sum_{i=0}^{k-1} \Pr(\eta(x,\infty)=i)= \sum_{i=0}^{k-1} \frac{\overline\rho(x)^i e^{-\overline\rho(x)}}{i!}.
	\end{align*}
	The value at $0$ follows due to the right continuity of the cdf.
	Finally, using the identity, for any $\lambda>0$,
	$$
	\sum_{i=0}^{k-1} \frac{\lambda^i e^{-\lambda}}{i!}
	=\frac{\lambda^k}{\Gamma(k)}\int_1^\infty u^{k-1}e^{-u\lambda}du
	$$
	we obtain
	$$
	\sum_{i=0}^{k-1} \frac{\overline\rho(x)^i e^{-\overline\rho(x)}}{i!}
	=\frac{\overline\rho(x)^k}{\Gamma(k)}\int_1^\infty u^{k-1}e^{-u\overline\rho(x)}du
	=\Pr(G_k \geq \overline\rho(x))
	=\Pr(x \geq \overline\rho^{-1}(G_k))
	$$
	where the last equality follows from the definition of the inverse tail intensity $\overline\rho^{-1}$ and the absolute continuity of $G_k$.
\end{proof}

\begin{lemma}[L\'evy continuity theorem for triangular arrays]\label{lemma:idlaplace}
Let $(X_{p,i})_{p\geq 1,i=1,\ldots,p}$ be a triangular array of nonnegative scalar random
variables such that for every $p$, $(X_{p,i})_{i=1,\ldots,p}$ is an iid sequence of random variables
from a probability distribution $\mu_{p}$ on $[0,\infty)$. Let $a\geq 0$ and $\rho$ be a L\'evy measure on $(0,\infty)$. Then,
\[
        \sum_{i=1}^{p}X_{p,i}\overset{\cvd}{\rightarrow} \id(a,\rho)\text{ as }p\to\infty
\]
if and only if, for any $t\geq0$,%
\[
\int_{0}^{\infty}(1-e^{-tx})p\mu_{p}(dx)\rightarrow at+\int_{0}^{\infty
}(1-e^{-wt})\rho(dw)
\]
pointwise as $p\to\infty$.
\end{lemma}

\begin{proof}
Recall that if $S\sim \id(a,\rho)$, then $\bfE[e^{-tS}]=e^{-at-\psi(t)}$
where $\psi(t)=\int_{0}^{\infty}(1-e^{-wt})\rho(dw)$. By the L\'evy continuity theorem for Laplace transforms of nonnegative random variables, $\sum_{i=1}^{{p}} X_{{p},i}\cvdto \id(a,\rho)$ if and only if, for any $t\geq0$,
\begin{align*}
\bfE\left[  e^{-t\sum_{i=1}^{{p}}X_{{p},i}}\right]    & =\bfE\left[
e^{-tX_{{p},1}}\right]  ^{{p}}\\
& =\left(  \int_{0}^{\infty}e^{-tx}\mu_{{p}}(dx)\right)  ^{{p}}\\
& =\left(  1-\frac{1}{{p}}\int_{0}^{\infty}(1-e^{-tx}){p}\mu_{{p}}(dx)\right)
^{{p}}\rightarrow e^{-at-\psi(t)},
\end{align*}
which holds if and only if $\int_{0}^{\infty}(1-e^{-tx}){p}\mu_{{p}}(dx)\rightarrow
at+\psi(t).$
\end{proof}

\begin{proposition}[Compressibility of triangular arrays]\label{prop:compressibilityID}
	Let
	\[
	(X_{p,j})_{p\geq 1,j=1,\ldots,p}
	\]
	be a triangular array of nonnegative real random variables such that for each $p$, $(X_{p,j})_{j=1,\ldots,p}$ are iid. Assume $\sum_{j=1}^p X_{p,j}\overset{\cvd}{\to}\id(a,\rho)$ for some $a\geq 0$ and some L\'evy measure $\rho$ on $(0,\infty)$. For each $p\geq 1$, let $X_{p,(1)}\geq X_{p,(2)}\geq \ldots\geq X_{p,(p)}$ be the ordered values. Then, for every $\kappa\in(0,1)$,
	\begin{align}
		X_{p,(\lfloor\kappa p\rfloor)}\cvpto 0\ \text{ as $p\to\infty$}. \label{eq:compressibilityID_eq1}
	\end{align}
	Moreover, if $a=0$, then for each $ \kappa \in (0,1) $,
	\begin{align}
		\sum_{j=1}^p \ind_{\{X_{p,j} \leq X_{p,(\lfloor \kappa p\rfloor)}\}}X_{p,j}\cvpto 0 \ \text{ as $p\to\infty$.}\label{eq:compressibilityID_eq2}
	\end{align}
\end{proposition}

\begin{proof}
	We first prove \cref{eq:compressibilityID_eq1}. Suppose to the contrary that $X_{p,(\lfloor \kappa p\rfloor)}$ does not converge to 0. Then, there exist $\epsilon>0$ and $\eta\in(0,1)$ such that $\epsilon$ is a continuity point of $\rho$ (i.e. $\rho(\{\epsilon\}) = 0$) and for any $p_0$, there exists $p>p_0$ such that
	$$
	\Pr(X_{p,(\lfloor \kappa p \rfloor)}>\epsilon)>\eta.
	$$
	Hence,
	\begin{align}
		p\Pr(X_{p,1}>\epsilon)= \bbE\left[\sum_{j}\ind_{\{X_{p,j}>\epsilon\}}\right ]>\eta \lfloor \kappa p \rfloor.
	\end{align}
	But $p\Pr(X_{p,1}>\epsilon)\to \overline\rho(\epsilon)<\infty$, which gives a contradiction.
	
	Now, suppose $ a = 0 $ and choose $\kappa\in(0,1)$.
	We prove \cref{eq:compressibilityID_eq2}
	by showing that, for any $\epsilon>0$, as $p\to\infty$,
	\[
	\Pr\left(\sum_{j=1}^p \ind_{\{X_{p,j} \leq X_{p,(\lfloor \kappa p\rfloor)}\}}X_{p,j}>\epsilon\right)\to 0.
	\]
	Since $a=0$, it follows that for any continuity point $h>0$ of $\rho$,
	$$
	p\bbE[X_{p,1}\ind_{\{X_{p,1}\leq h\}}]\to \int_0^h x\rho(dx)
	$$
	as $p\to\infty$. (See the first part of \cref{th:convergenceid}, which is a corollary of Theorem 15.28 in \cite{Kallenberg2002}.) Thus, for any $\eta>0$, there exist a continuity point $h_0>0$ of $\rho$ and $p_0\in\mathbb N$ such that for all $p\geq p_0$,
	\begin{equation}
		\label{eqn:compression1}
		p\bbE[X_{p,1}\ind_{\{X_{p,1}\leq h_0\}}]<\eta.
	\end{equation}
	
	Consider $\epsilon>0$ and $\gamma > 0$. Define $\eta := (\gamma\epsilon)/2$. Let $h_0$ and $p_0$ be such that \cref{eqn:compression1} holds for $\eta$. Note that
	\begin{equation}
		\label{eqn:compression2}
		\begin{aligned}
			\Pr\left(\sum_{j=1}^p \ind_{\{X_{p,j} \leq X_{p,(\lfloor \kappa p\rfloor)}\}}X_{p,j}>\epsilon\right)
			&{} =\Pr\left(\sum_{j=1}^p \ind_{\{X_{p,j} \leq X_{p,(\lfloor \kappa p\rfloor)}\}}X_{p,j}>\epsilon\ \text{and}\ X_{p,(\lfloor \kappa p \rfloor)}\leq h_0\right)\\
			&{} +\Pr\left(\sum_{j=1}^p \ind_{\{X_{p,j} \leq X_{p,(\lfloor \kappa p\rfloor)}\}}X_{p,j}>\epsilon\ \text{and}\ X_{p,(\lfloor \kappa p \rfloor)} > h_0\right).
		\end{aligned}
	\end{equation}
	We will prove that for all sufficiently large $p$, each summand on the right-hand side above is bounded by $\gamma/2$.
	The next derivation uses Markov's inequality and bounds the first summand in \cref{eqn:compression2} for all $p \geq p_0$:
	\begin{align*}
		\Pr\left(\sum_{j=1}^p \ind_{\{X_{p,j} \leq X_{p,(\lfloor \kappa p\rfloor)}\}}X_{p,j}>\epsilon\ \text{and}\ X_{p,(\lfloor \kappa p\rfloor)}\leq h_0\right)
		& {} \leq
		\Pr\left(\sum_{j}  X_{p,j}\ind_{\{X_{p,j}\leq h_0\}}>\epsilon\right)
		\\
		& {} \leq
		\frac{p\bbE[X_{p,1}\ind_{\{X_{p,1}\leq h_0\}}]}{\epsilon}
		< \frac{\eta}{\epsilon} = \frac{\gamma\epsilon}{2\epsilon} = \frac{\gamma}{2}.
	\end{align*}
	The bound for the second summand in \cref{eqn:compression2} follows from \cref{eq:compressibilityID_eq1}.
	There is $p_1 \in \N$ such that for all $p\geq p_1$,
	$$
	\Pr\left(\sum_{j=1}^p \ind_{\{X_{p,j} \leq X_{p,(\lfloor \kappa p\rfloor)}\}}X_{p,j}>\epsilon\ \text{and}\ X_{p,(\lfloor \kappa p\rfloor)} > h_0\right)
	\leq
	\Pr\left(X_{p,(\lfloor \kappa p\rfloor)}>h_0\right)< \frac{\gamma}{2}.
	$$
	Bringing these two bounds together, we can conclude that for all $p\geq \max(p_0,p_1)$,
	\[
	\Pr\left(\sum_{j=1}^p \ind_{\{X_{p,j} \leq X_{p,(\lfloor \kappa p\rfloor)}\}}X_{p,j}>\epsilon\right) <\gamma,
	\]
	as desired.

\end{proof}

For the rest of this section, we consider the space $\K_{d}$ of positive semi-definite $ n\times n $ matrices.
To state the results, we need to recall a few definitions regarding cones.

Let $\bbH$ be a finite-dimensional Hilbert space, and $\left\Vert
\,\cdot\,\right\Vert$ and $\left\langle \,\cdot\,,\,\cdot\,\right\rangle$ be its norm and  inner product.
A nonempty convex subset $\bbS$ of $\bbH$ is a \textit{cone} if $\lambda\geq0$ and $x\in \bbS$ implies $\lambda x\in \bbS$. A cone is \textit{proper} if
$x=0$ whenever $x$ and $-x$ are in $\bbS$. The \textit{dual cone} $\bbS^{\prime}$ of $\bbS$ is
defined as $\bbS^{\prime}=\{y\in \bbH:\left\langle y,x\right\rangle \geq0$
for all $x\in \bbS\}$. Examples of proper cones include $[0,\infty)$,
$[0,\infty)^{d}$, and the set $\K_{d}$
of positive semi-definite $d$-by-$d$ matrices with real
entries.

Let $\bbS$ be a proper convex cone of $\bbH$. Denote by $\bbS\cup\{\Delta\}$ the
one-point compactification of the cone. Let $C(\bbS\cup\{\Delta\})$ be the set of
continuous functions $f:\bbS\cup\{\Delta\}\rightarrow\mathbb{R}$. For a function
$f\in C(\bbS\cup\{\Delta\})$, define $\left\Vert f\right\Vert =\sup_{x\in
      	\bbS\cup\{\Delta\}}|f(x)|$. Similarly, for a function $g : \bbS \rightarrow \mathbb{R}$,
define  $\left\Vert g\right\Vert =\sup_{x\in \bbS}|g(x)|$.

The next proposition expresses L\'evy's continuity theorem for $\K_{d}$, the cone
of positive semi-definite $d$-by-$d$ matrices with real entries. In this case, the ambient space $\bbH$ is
that of symmetric $d$-by-$d$ matrices with real entries, and its inner product and norm are inherited
from the Euclidean space $\mathbb{R}^{d\times d}$. We point out that with respect to this ambient space
$\bbH$, the cone $\K_{d}$ is self-dual: $(\K_{d})' = \K_{d}$.

\begin{lemma}[L\'evy continuity theorem on $\K_{d}$]\label{lem: levy continuity theorem on cones}
	Let  $\mu,\mu_{1},\mu_{2},\ldots$ be probability measures on $\K_{d}$.  If $\widetilde{\mu}_{n}(\theta)=\int e^{-\left\langle x,\theta\right\rangle }\mu_{n}(dx)$ converges pointwise to $\widetilde{\mu}(\theta)=\int e^{-\left\langle x,\theta\right\rangle}\mu(dx)$ for every $\theta\in \K_{d}$, then $\mu_{n} \cvdto \mu$.
\end{lemma}

\begin{proof}
	The proof follows that of Theorem 5.3 in \cite{Kallenberg2002}.
	
	Assume that $\widetilde{\mu}_{n}(\theta)$ converges to $\widetilde{\mu}(\theta)$ for every $\theta\in \K_{d}$.
	We have, using $e^{-t}\leq\frac{1}{2}$ for all
	$t\geq1$,
	\[
	1-\widetilde{\mu}_{n}(\theta)=\int_{\K_d}(1-e^{-\left\langle x,\theta
		\right\rangle }) \, \mu_{n}(dx)\geq\frac{1}{2}\mu_{n}(\{x:\left\langle
	x,\theta\right\rangle \geq1\}).
	\]
	Hence, for all $\theta\in \K_{d}$ and $r>0$, we have
	$\mu_{n}(\{x:\left\langle x,\theta\right\rangle \geq r\})\leq
	2(1-\widetilde{\mu}_{n}(\theta/r))$, which then implies
	\[
	\lim\sup_{n}\mu_{n}(\{x:\left\langle x,\theta\right\rangle \geq r\})\leq
	\lim_{n\rightarrow\infty}2(1-\widetilde{\mu}_{n}(\theta/r))=2(1-\widetilde
	{\mu}(\theta/r)).
	\]
	Taking $r\rightarrow\infty$ and using the continuity of $\widetilde{\mu}$ at
	$0$, we obtain that
	\[
	\lim_{r\rightarrow\infty}\lim\sup_{n}\mu_{n}(\{x:\left\langle x,\theta\right\rangle \geq r\})=0.
	\]
	From this, straightforward calculations show that
	\[ \lim_{r\rightarrow\infty}\lim\sup_{n}\mu_{n}(\{x: \left\Vert x\right\Vert \geq r\})=0, \] that is, the sequence $(\mu_{n})_n$ is tight. As a result,
	for any $\varepsilon>0$, we may choose large $r$
	that $\mu_{n}(\{x:\left\Vert x\right\Vert \geq r\})\leq\varepsilon$ for all $n$
	and $\mu(\{x:\left\Vert x\right\Vert \geq r\})\leq\varepsilon$.
	
	Now fix a bounded continuous function $f: \K_{d} \rightarrow\mathbb{R}$, with
	$\left\Vert f\right\Vert \leq m<\infty$. Pick an arbitrary $\varepsilon > 0$. Let $r > 0$ be the real such that
	$\mu_{n}(\{x:\left\Vert x\right\Vert \geq r\})\leq\varepsilon$ for all $n$
	and $\mu(\{x:\left\Vert x\right\Vert \geq r\})\leq\varepsilon$.
	Define $f_{r}$ to be the restriction of $f$
	to the ball $\{\left\Vert x\right\Vert \leq r\}$, and extend $f_{r}$ to a
	continuous function $\widetilde{f}$ on $\K_{d} \cup \{\Delta\}$ with $\left\Vert \widetilde
	{f}\right\Vert \leq m.$ For instance, one can set
	\[
	\widetilde{f}(x)=f\left(  \frac{rx}{\left\Vert x\right\Vert }\right)
	\times\max(r+1-\left\Vert x\right\Vert ,0)\text{ for }\left\Vert x\right\Vert
	\geq r.
	\]
	By \cref{lem: stone weierstrass on cones}, there exists some function
	$g(x)=\sum_{j=1}^{p}\lambda_{j}e^{-\left\langle x,\theta_{j}\right\rangle }$
	with $\theta_{j}\in \K_{d}$ and $\lambda_{j}\in\mathbb{R}$ such that
	$\left\Vert \widetilde{f}-g \right\Vert \leq\varepsilon$. We obtain%
	\begin{align*}
		\left\vert \mu_{n}f-\mu_{n}g\right\vert
		& {} \leq \left\vert \mu_n f - \mu_n \widetilde{f} \right\vert + \left\vert \mu_n \widetilde{f} - \mu_n g \right\vert
		\\
		& {} \leq
		\mu_{n}(\{x:\left\Vert x\right\Vert \geq r\})
		\left\Vert f-\widetilde{f}\right\Vert
		+ \left\Vert \widetilde{f}- g\right\Vert
		\\
		& {} \leq(2m+1)\varepsilon,
	\end{align*}
	and similarly for $\mu$. Thus,
	\begin{align*}
		\left\vert \mu_{n}f-\mu f\right\vert
		& {}\leq
		\left\vert \mu_n f - \mu_n g\right\vert +
		\left\vert \mu_n g - \mu g\right\vert +
		\left\vert \mu f - \mu g \right\vert
		\\
		& {} \leq\left\vert \mu_{n}g-\mu g\right\vert
		+2(2m+1)\varepsilon.
	\end{align*}
	By the pointwise convergence of $\widetilde{\mu}_{n}$,
	\begin{align*}
		\left\vert \mu_{n}g-\mu g\right\vert
		& {} \leq \sum_{j=1}^{p}|\lambda_{j}|
		|\widetilde{\mu}_{n}(\theta_{j}) - \widetilde{\mu}(\theta_{j})|
		\\
		& {} \rightarrow 0
	\end{align*}
	as $n\rightarrow\infty$. Letting $n\rightarrow\infty$ and then
	$\varepsilon\rightarrow0$, we obtain $\mu_{n}f\rightarrow\mu f$. As $f$ was chosen arbitrarily, this proves $\mu_{n}\cvdto \mu$.
\end{proof}

We state the following version of Lemma 5.4 in \cite{Kallenberg2002} for the cone of $d$-by-$d$ positive semi-definite matrices.
\begin{lemma}[Stone-Weierstrass theorem on $ \K_{d} $]\label{lem: stone weierstrass on cones}
	Every continuous function $g:\K_d \cup\{\Delta\}\rightarrow\mathbb{R}$ can be
	approximated uniformly by linear combinations of the functions
	$x \longmapsto e^{-\left\langle x,\theta\right\rangle}$ for $\theta\in \K_d$.
\end{lemma}

We omit the proof of \cref{lem: stone weierstrass on cones} as it directly follows from the general Stone-Weierstrass theorem.

\begin{lemma}\label{thm:tensoring with cones}
	Let $(X_{pi})$ be a triangular array of nonnegative scalar random
	variables such that for every $p$, $(X_{pi})_{i = 1,\ldots,p}$ is an iid sequence of random variables from $\mu_{p}$. Suppose also that  $(Y_{i})$ is an iid sequence in the cone $\K_n$ of positive semi-definite $n$-by-$n$ matrices with $Y_{i}\sim F$ and $\E[\left\Vert
	Y_{1}\right\Vert ]<\infty.$ If %
	$\sum_{i=1}^{p}X_{pi}$ converges in distribution to $\id(a,\rho)$ as $p\to\infty$, then as $p\to\infty$,
	\begin{equation}
		\sum_{i=1}^{p}X_{pi}Y_{i} \cvdto \id({A},\nu)\label{eq:sumxy2}%
	\end{equation}
	where ${A}\in \K_n$ and the L\'{e}vy measure $\nu$ on $\K_n\setminus\{0\}$ are defined by
	\begin{align*}
		{A}  & :=a\E[Y_{1}],
		&
		\nu(dz)  & :=\int_{0}^{\infty}F\left(  \frac{dz}{x}\right) \rho(dx) .
	\end{align*}
\end{lemma}
\begin{remark}
	The measure $ \nu $ defined above is indeed a L\'evy measure. Obviously, $ \nu(\{ 0\}) = 0 $. To see that $ \int \min(\|z\|,1) \, \nu(dz) < \infty $, note that
	\begin{align*}
		\int \min(\|z\|,1) \, \nu(dz)
		= \int_{\K_n}\int \min(x\|y\|,1) \ind_{\{xy \ne  0\}} \, \rho(dx) \, F(dy).
	\end{align*}
	Here, note that, for nonnegative $ a $ and $ b $, $ \min( 1,ab ) \le \max(1,a)\min(1,b) $. It follows that
	\begin{align*}
		\int_{\K_n}\int \min(x\|y\|,1) \ind_{\{xy \ne  0\}} \, \rho(dx) \, F(dy)
		&\le \int_{\K_n}\int \max( 1,\|y\| )\min(1,x) \ind_{\{xy \ne  0\}} \, \rho(dx) \, F(dy) \\
		&\le \int_{\K_n} (1+\|y\|)  \, F(dy) \int \min(1,x)\, \rho(dx) < \infty.
	\end{align*}
\end{remark}
\begin{proof}
	Let $ \theta \in \K_{n}' = \K_{n} $. We calculate the following Laplace transform at $ \theta $:
	\begin{align*}
		\bfE\left[ \exp\left( -\left\langle \theta, \sum_{i=1}^{p} X_{pi}Y_{i} \right\rangle \right) \right]
		&= \left( \bfE\left[ \exp\left( -\left\langle \theta, X_{p1}Y_{1} \right\rangle \right) \right] \right)^{p} \\
		&= \left( \int_{\K_{n}} \int e^{-\langle \theta,y \rangle x} \, \mu_{p}(dx) \, F(dy) \right)^{p} \\
		&= \left( 1- \frac{1}{p}\int_{\K_{n}} \int \left( 1-e^{-\langle \theta,y \rangle x} \right) \, p\mu_{p}(dx) \, F(dy) \right)^{p} .
	\end{align*}
	Now, for a fixed $ \theta, y \in \K_{n} $, \cref{lemma:idlaplace} implies that
	\begin{align*}
		\int \left( 1-e^{-\langle \theta,y \rangle x} \right) \, p\mu_{p}(dx) \to a\langle \theta, y \rangle + \int \left( 1-e^{-\langle \theta,y \rangle x} \right) \, \rho(dx) .
	\end{align*}
	Note that, as a function of $ y $, the above integral on the left-hand side is dominated by an integrable function. Specifically, we get, for large enough $ p $'s,
	\begin{align*}
		\int \left( 1-e^{-\langle \theta,y \rangle x} \right) \, p\mu_{p}(dx)
		&\le \int \min( 1,\langle \theta,y \rangle x ) \, p\mu_{p}(dx) \\
		&\le \max( 1,\langle \theta,y \rangle ) \int \min( 1,x ) \, p\mu_{p}(dx) \\
		&\le \left( 1+\|\theta\|\|y\| \right) \left( a+\int \min( 1,x ) \, \rho(dx) + \eps \right),
	\end{align*}
	where we used the fact that $ \min(1,ab) \le \max(1,a)\min(1,b) $ for nonnegative $ a $ and $ b $. As $ \bfE[\|Y_{1}\|] < \infty $, the last expression is integrable with respect to $ F(dy) $.
	By dominated convergence, it follows that
	\begin{align*}
		&\bfE\left[ \exp\left( -\left\langle \theta, \sum_{i=1}^{p} X_{pi}Y_{i} \right\rangle \right) \right] \\
		&\to \exp\left( -\left( \langle \theta,a\bfE[Y_{1}] \rangle + \int_{K_{n}} \int \left( 1-e^{-\langle \theta,xy \rangle } \right) \, \rho(dx) \, F(dy) \right) \right) \\
		&= \exp\left( -\left( \langle \theta,A \rangle + \int \left( 1-e^{-\langle \theta,z \rangle} \right) \, \nu(dz) \right) \right),
	\end{align*}
	which is the Laplace transform of $ \operatorname{ID}(A,\nu) $. By \cref{lem: levy continuity theorem on cones}, (see \cite[Theorem 5.4]{davydov2008strictly}), the Laplace transform uniquely determines the distribution, so the proof is completed.
\end{proof}

In the special case when $ n=1 $, we have the following.
\begin{corollary}\label{lemma:idlaplaceproduct}
	Let $(X_{{p}i})$ be a triangular array of nonnegative scalar random
	variables such that for every $p$, $(X_{{p}i})_{i=1,\ldots,p}$ is an iid sequence of random variables from $\mu_{{p}}$. Let $Y_{1}%
	,Y_{2},\ldots.$ be a sequence of iid nonnegative random variables $Y_{i}%
	\in\lbrack0,\infty)$ with $Y_{i}\sim F$ and $\bfE[Y_{1}]<\infty.$ Let $a\geq0$ and $\rho$ be a L\'evy measure on $(0,\infty)$. Then,%
	\begin{equation}
		\sum_{i=1}^{{p}}X_{{p}i}\cvdto \id(a,\rho)\label{eq:sumx}%
	\end{equation}
	implies%
	\begin{equation}
		\sum_{i=1}^{{p}}X_{{p}i}Y_{i}\cvdto \id(c,\nu)\label{eq:sumxy}%
	\end{equation}
	where $c\ge 0$ and the L\'{e}vy measure $\nu$ on $(0,\infty)$ are defined by
	\begin{align*}
		c  & :=a\bfE[Y_{1}], &  \nu(dz)  & :=\int_{0}^{\infty} F\left(\frac{dz}{x}\right) \rho(dx).
	\end{align*}
\end{corollary}
When $F$ has a density $f$ with respect to the Lebesgue measure, the Levy measure $\nu$ can be expressed as
\[
\nu(dz)   :=\int_0^\infty \rho(dz/x) f(x)dx.
\]
In this paper, we often used this equivalent form of $\nu$.

\newpage

\section{Proofs of the Main Theorems and Propositions}
\label{app:proofs-all}

\subsection{Proof of \cref{th:compressibility}}

Assume $ a^{(l)} = 0 $. Then, \cref{eq:compressibility_lambda} follows directly from \cref{prop:compressibilityID} in the Appendix and Slutsky's theorem.

Note that $T^{(l+1)}_j=\lambda^{(l)}_{p_l,j} Y_j$,
where the random variables $Y_{j}:=\sum_{k=1}^{p_{l+1}}(V_{j,k}^{(l+1)})^2$ are iid, and do not depend on $p_l$. If $ a^{(l)} = 0 $, then \cref{lemma:idlaplaceproduct} implies that $\sum_{j=1}^{p_l}T^{(l+1)}_j\to \id(0,\nu)$ for some L\'evy measure $\nu$. \Cref{eq:compressibility_w} then follows similarly from \cref{prop:compressibilityID} and Slutsky's theorem.

\subsection{Proof of \cref{prop:regvarweightsinf}}

We assume here more generally that $\overline\rho^{(l)}(x)\overset{x\to\infty}{\sim}x^{-\tau}L(x)$, where $L$ is a slowly varying function.
Using \cref{prop: extremes} and \cref{prop:extremesID} in the Appendix with the fact that $1-\sum_{i=0}^{k-1}\frac{x^i e^{-x}}{i!}\overset{x\to 0}{\sim}\frac{x^k}{k!}$, we obtain the first equivalence relations in \cref{eq:01,eq:02}.  We have
\begin{align*}
\overline\nu^{(l)}(x)&=\int_0^\infty \overline\rho^{(l)}(x/z)\gammadist\left (z;\frac{1}{2},\frac{1}{2}\right)dz\\
&=x\int_0^\infty \overline\rho^{(l)}(1/ u)\gammadist\left (ux;\frac{1}{2},\frac{1}{2}\right)du\\
&=\frac{1}{\sqrt{2\pi}}x^{1/2}\int_0^\infty \overline\rho^{(l)}(1/ u)u^{-1/2}e^{-ux/2}du.
\end{align*}
We have $\overline\rho^{(l)}(1/u)u^{-1/2}\overset{u\to 0}{\sim}  u^{\tau-1/2}L(1/u)$. Using a Tauberian theorem \cite[Chapter 13, Theorem 2]{Feller1971}, we obtain
\begin{align*}
\overline\nu^{(l)}(x)&\overset{x\to\infty}{\sim}\frac{1}{\sqrt{2\pi}}x^{1/2} \times \Gamma(\tau+1/2) (x/2)^{-\tau-1/2}L(x/2)\\
&\overset{x\to\infty}{\sim}\frac{2^\tau\Gamma(\tau+1/2)}{\sqrt{\pi}}x^{-\tau}L(x).
\end{align*}
The second equivalence relations in \cref{eq:01,eq:02} now follow directly.

\subsection{Proofs of \cref{prop:powerlawactivations,prop:powerlawactivations_nobias}}

\paragraph{Proof of \cref{prop:powerlawactivations}.}
First consider \cref{eq:powerlawSl}. We have $S^{(l)}\sim\id(\frac{a^{(l)}}{2},\frac{\nu^{(l)}}{2})$ as mentioned in \cref{eq:relulevy}. Also, we have shown in the proof of \cref{prop:regvarweightsinf} that \cref{eq:rhobarRVL} implies
$$
\overline\nu^{(l)}(x)\overset{x\to\infty}{\sim}
\frac{2^{\tau^{(l)}}\Gamma(\tau^{(l)}+1/2)}{\sqrt{\pi}}c^{(l)}  x^{-\tau^{(l)}}.
$$
\cref{eq:powerlawSl} then follows from
\cref{prop:idRVregvar}.
For $l=2$, we have
\begin{align}
\Sigma^{(2)}(\bx)&=\sigma_b^2 + \sigma_v^2 S^{(1)} \Sigma^{(1)}(\bx)
\end{align}
where $\Sigma^{(1)}(\bx)=\sigma_b^2 + \sigma_v^2 \frac{ \| \bx\|^2}{\din}$ is fixed. Thus,
\begin{align}
\Pr\left(\Sigma^{(2)}(\bx)>u\right)\overset{u\to\infty}{\sim}\Pr(S^{(1)}>u)\cdot (\sigma_v^2  \Sigma^{(1)}(\bx))^{\tau^{(1)}}
\end{align}
and the random variable $\Sigma^{(2)}(\bx)$ has a power-law tail with exponent $\beta^{(1)}=\tau^{(1)}$.
We can now proceed by induction. If $\Sigma^{(l)}(\bx)$ has a regularly varying tail with exponent $\beta^{(l-1)}$ and $S^{(l)}$ has a regularly varying tail with exponent $\tau^{(l)}$, then $\Sigma^{(l+1)}(\bx)=\sigma_b^2 + \sigma_v^2 S^{(l)} \Sigma^{(l)}(\bx)$ also has regularly varying tail with exponent $\beta^{(l)}=\min(\beta^{(l-1)},\tau^{(l)})$ by
\cref{lem:Jessen4.1i}.
Finally, we have, using
\cref{lem:Jessen4.2},
\begin{align*}
\Pr\left( (\zeta_{k}^{(l)})^2   >u\right)&\overset{u\to\infty}{\sim}  \Pr\left(\Sigma^{(l)}(\bx)>u\right) \bfE\left [(\varepsilon_k^{(l)})^{2\beta^{(l-1)}} \right ]\\
&\overset{u\to\infty}{\sim}  \Pr\left(\Sigma^{(l)}(\bx)>u\right) \times 2^{\beta^{(l-1)}}\frac{\Gamma(\beta^{(l-1)}+1/2)}{\Gamma(1/2)}.
\end{align*}

\paragraph{Proof of \cref{prop:powerlawactivations_nobias}.}
If $\sigma_b=0$, then
$$\Sigma^{(l+1)}(\bx)=\frac{ \| \bx\|^2}{\din}\sigma_v^{2(l+1)} \prod_{j=1}^{l} S^{(j)}$$
where $S^{(1)},\ldots,S^{(L)}$ are iid and
\[
\Pr\left(S^{(j)} > u\right) \overset{u\to\infty}{\sim} \widetilde{c} u^{-\tau}.
\]
Here
\begin{equation*}
\widetilde c=
\frac{2^{\tau-1}\Gamma(\tau+1/2)}{\sqrt{\pi}}c \label{eq:c2tilde}.
\end{equation*}
It
follows from \cref{lem:Jessen4.1iv} that
$$
\Pr\left(\prod_{j=1}^l S^{(j)}>u\right)\overset{u\to\infty}{\sim}\frac{\tau^{l-1}(\widetilde c)^l}{(l-1)!}u^{-\tau}\log^{l-1} u.
$$
Therefore, for $l\geq 1$,
$$
\Pr\left(\Sigma^{(l+1)}(\bx)>u\right)\overset{u\to\infty}{\sim}\left (\frac{ \| \bx\|^2}{\din}\sigma_v^{2(l+1)} \right)^\tau \frac{\tau^{l-1}(\widetilde c)^l}{(l-1)!}u^{-\tau}\log^{l-1} u.
$$
Finally, we have, using \cref{lem:Jessen4.2},
\begin{align*}
\Pr\left( (\zeta_{k}^{(l+1)})^2   >u\right)
&\overset{u\to\infty}{\sim}
\Pr\left(\Sigma^{(l+1)}(\bx)>u\right) \times 2^{\tau}\frac{\Gamma(\tau+1/2)}{\Gamma(1/2)}.
\end{align*}

\subsection{Proofs of \cref{prop: error bound,thm: eps-pruning,th:kappapruning}}\label{sec: proof pruning}

Recall the assumption (UI):
For all layers $ l=1,\ldots,L $,
$$\int_{0}^{\infty} u \rho^{(l)}(du) = M^{(l)}_{1} < \infty,
\quad
N^{(l)}_{p_l} < \infty \text{ for all $p_l$},
\quad
\text{and} \quad N^{(l)}_{p_l}\to a^{(l)}+M^{(l)}_{1}\text{ as } p_l \to \infty.$$
As mentioned earlier, this is equivalent to the uniform integrability of the family
\begin{align}\label{eq:family sum lambda is ui}
	\left\{ \sum_{j=1}^{p_{l}} \lambda^{(l)}_{p_{l},j} \right\}_{p_{l}} .
\end{align}
This can be seen
by \citet[Theorem~4.6.3]{durrett2019probability} and Skorokhod representation.

We start by introducing a lemma that results from this assumption.
\begin{lemma}\label{lem: l2 norm preactivations}
	Suppose (UI) and (A1) hold. For all $ l=1,\ldots,L+1, $ we have
	\begin{align*}
		\sup_{\bp} \bfE\left[ \left( Z^{(l)}_{1}(\bx;\bp) \right)^{2} \right] < \infty.
	\end{align*}
\end{lemma}
\begin{proof}
        Recall $C_\phi = \bfE[\phi(\eps)^2]$ for $\eps \sim \mathcal{N}(0,1)$.
	Note that
	\begin{align*}
		\bfE\left[ \left( Z^{(l+1)}_{1}(\bx;\bp) \right)^{2} \right]
		&= \sigma_{v}^{2}\bfE\left[ \sum_{j=1}^{p_{l}} \lambda^{(l)}_{p_{l},j} \phi^{2}(Z^{(l)}_{j}(\bx;\bp)) \right] + \sigma_{b}^{2} \\
		&= \sigma_{v}^{2}\bfE\left[ \sum_{j=1}^{p_{l}} \lambda^{(l)}_{p_{l},j} \right] \bfE\left[ \phi^{2}(Z^{(l)}_{1}(\bx;\bp)) \right] + \sigma_{b}^{2} \\
		&= \sigma_{v}^{2}C_{\phi} N^{(l)}_{p_{l}} \bfE\left[ \left(Z^{(l)}_{1}(\bx;\bp)\right)^{2} \right] + \sigma_{b}^{2}.
	\end{align*}
	Apply this recurrence repeatedly and note that $ \bfE[( Z^{(1)}_{1}(\bx) )^{2}] = \Sigma^{(1)}(\bx) = \sigma_{v}^{2}\|\bx\|^{2}/\din + \sigma_{b}^{2} $ is a constant. Then, we get
	\begin{align*}
		&\bfE\left[ \left( Z^{(l)}_{1}(\bx;\bp) \right)^{2} \right] \\
		&= \sigma_{v}^{2} \frac{\|\bx\|^{2}}{\din} \left( \prod_{l'=1}^{l-1} \sigma_{v}^{2}C_{\phi} N^{(l')}_{p_{l'}} \right) + \sigma_{b}^{2}\left( \prod_{l'=1}^{l-1} \sigma_{v}^{2}C_{\phi} N^{(l')}_{p_{l'}} + \cdots + \sigma_{v}^{2}C_{\phi} N^{(l-1)}_{p_{l-1}} + 1 \right).
	\end{align*}
	Since $ N^{(l)}_{p_{l}} \to a^{(l)} + M_{1}^{(l)} $ as $ p_{l}\to \infty $ for each $ l=1,\ldots,L $, we have
	\begin{align*}
		&\sup_{\bp} \bfE\left[ \left( Z^{(l)}_{1}(\bx;\bp) \right)^{2} \right] \\
		&\le 2\Bigg[ \sigma_{v}^{2} \frac{\|\bx\|^{2}}{\din} \left( \prod_{l'=1}^{l-1} \sigma_{v}^{2}C_{\phi} (a^{(l')} + M^{(l')}_{1}) \right) \\
		&\qquad + \sigma_{b}^{2}\left( \prod_{l'=1}^{l-1} \sigma_{v}^{2}C_{\phi} (a^{(l')} + M^{(l')}_{1}) + \cdots + \sigma_{v}^{2}C_{\phi} (a^{(l-1)} + M^{(l-1)}_{1}) + 1 \right) \Bigg] \\
		& <\infty.
	\end{align*}
	where the supremum is taken for all $ \bp $'s with large enough $ \min \bp$.
\end{proof}

\paragraph{Proof of \cref{prop: error bound}.}
Note that, for all $ p_{l} $,
\begin{align*}
	&\bfE\left[\left( Z^{(l+1)}_{1}(\bx;\bp) - Z^{*(l+1)}_{1}(\bx;\bp) \right)^{2}\right] \\
	&{} =\sigma_{v}^{2} \bfE\left[\sum_{j=1}^{p_{l}}\lambda^{(l)}_{p_{l},j} \ind_{\{\lambda^{(l)}_{p_{l},j} > \lambda^{*(l)}_{p_{l}} \}} \left( \phi(Z^{(l)}_{j}(\bx;\bp) - \phi(Z^{*(l)}_{j}(\bx;\bp)) \right)^{2}\right] \\
	&\phantom{{} = {}}
	{} + \sigma_{v}^{2}\bfE\left[\sum_{j=1}^{p_{l}}\lambda^{(l)}_{p_{l},j} \ind_{\{\lambda^{(l)}_{p_{l},j} \le \lambda^{*(l)}_{p_{l}} \}} \phi^{2}(Z^{(l)}_{j}(\bx;\bp))\right] \\
	& {} \le \sigma_{v}^{2}{\Clip}N^{(l)}_{p_{l}} \bfE\left[\left( Z^{(l)}_{1}(\bx;\bp) - Z^{*(l)}_{1}(\bx;\bp) \right)^{2}\right] \\
	& \phantom{{} \le {}}
	{} + \sigma_{v}^{2} {\Clip}\bfE\left[(Z^{(l)}_{1}(\bx;\bp))^{2}\right] \bfE\left[\sum_{j=1}^{p_{l}}\lambda^{(l)}_{p_{l},j} \ind_{\{\lambda^{(l)}_{p_{l},j} \le \lambda^{*(l)}_{p_{l}} \}}\right] .
\end{align*}
If $ U^{(l)}:= \sup_{\bp} \bfE\left[(Z^{(l)}_{1}(\bx;\bp))^{2}\right] < \infty$ for all $ l=1,\ldots,L $, we get the following recurrence relation:
\begin{align*}
	&\bfE\left[\left( Z^{(l+1)}_{1}(\bx;\bp) - Z^{*(l+1)}_{1}(\bx;\bp) \right)^{2}\right]
	\\
	&{} \le \sigma_{v}^{2}{\Clip}N^{(l)}_{p_{l}} \bfE\left[\left( Z^{(l)}_{1}(\bx;\bp) - Z^{*(l)}_{1}(\bx;\bp) \right)^{2}\right]
	+ \sigma_{v}^{2} {\Clip}U^{(l)} \bfE\left[\sum_{j=1}^{p_{l}}\lambda^{(l)}_{p_{l},j} \ind_{\{\lambda^{(l)}_{p_{l},j} \le \lambda^{*(l)}_{p_{l}} \}}\right]
	\\
	&=: \sigma_{v}^{2} {\Clip}N^{(l)}_{p_{l}} \bfE\left[\left( Z^{(l)}_{1}(\bx;\bp) - Z^{*(l)}_{1}(\bx;\bp) \right)^{2}\right] + \sigma_{v}^{2} {\Clip}U^{(l)} A^{(l)}_{p_{l}}
\end{align*}
Noting that $ Z^{(1)}_{1}(\bx;\bp) = Z^{*(1)}_{1}(\bx;\bp) $, it inductively follows that
\begin{align}
	&\bfE\left[\left( Z^{(l+1)}_{1}(\bx;\bp) - Z^{*(l+1)}_{1}(\bx;\bp) \right)^{2}\right]
	\label{eq:error l2 bound eps}\\
	&\le \sigma_{v}^{2}{\Clip}U^{(l)}A^{(l)}_{p_{l}} + (\sigma_{v}^{2}{\Clip})^{2}N^{(l)}_{p_{l}}U^{(l-1)}A^{(l-1)}_{p_{l-1}} + \cdots + (\sigma_{v}^{2}{\Clip})^{l}N^{(l)}_{p_{l}}\cdots N^{(2)}_{p_{2}} U^{(1)}A^{(1)}_{p_{1}} . \notag
\end{align}
On the other hand, if $U^{(l)}=\infty$ for some $l$, the above inequality holds trivially.

\paragraph{Proof of \cref{thm: eps-pruning}.}

As mentioned in \cref{rem: trade-off pruning}, we prove the corollary in the setting where, for each $ l=1,\ldots,L $, the tail L\'evy measure satisfies
\begin{align}
	\label{eq: pruning assumption on rho}
	\overline{\rho}^{(l)} \overset{u\to 0}{\sim} u^{-\alpha^{(l)}}L^{(l)}\left(\frac{1}{u}\right)
\end{align}
where $ \alpha^{(l)} \in [0,1) $ and $ L^{(l)} $ is a slowly varying function. We also let
$M_1^{(l)} = \int_0^\infty u \rho^{(l)}(du)$, $\alpha = \max_{l}{\alpha^{(l)}} $ and $0 < \delta < 1-\alpha $.

From \cref{th:convergenceid}, it is straightforward to check that, for each $ l=1,\ldots,L, $
\begin{align}
	\sum_{j=1}^{p_{l}} \lambda^{(l)}_{p_{l},j} \ind_{\{ \lambda^{(l)}_{p_{l},j} \leq \eps \}} \cvdto \operatorname{ID}(0,\rho^{(l)}|_{(0,\eps]}), \label{eq: pruned conv to id}
\end{align}
as $ p_{l} \to \infty, $ where we denote by $ \rho^{(l)}|_{(0,\eps]} $ the measure $ \rho^{(l)} $ restricted to $ (0,\eps] $. (Technically, we need to assume that $ \eps $ is a continuity point of $ \rho^{(l)}. $ See \cref{rmk: continuity point}.)
As the family in \cref{eq:family sum lambda is ui} is uniformly integrable, it follows that, for each $ l=1,\ldots,L, $
\begin{align*}
	A^{(l)}_{p_{l}} = \bfE\left[ \sum_{j=1}^{p_{l}} \lambda^{(l)}_{p_{l},j} \ind_{\{ \lambda^{(l)}_{p_{l},j} \le \eps \}} \right] \to \int x \ind_{\{ x\le\eps \}} \, \rho^{(l)}(dx)
\end{align*}
as $ p_{l} \to \infty $.
To finish the proof, we introduce a corollary to \cref{lem: feller slowly varying bound}.
\begin{corollary}\label{cor: new potter type bound for all layer}
	Recall that $ \alpha = \max_{l} \alpha^{(l)} $ and $ 0< \delta < 1-\alpha $. Then, for each $ l = 1,\ldots,L $, there exists $ \epsilon^{(l)} > 0 $ such that
	$$ \overline\rho^{(l)}(x) \le x^{-(\alpha^{(l)}+\delta)} $$
	for all $ x \le \epsilon^{(l)} $.
	Consequently, for all $ l = 1,\ldots,L $,
	\begin{align*}
		\overline\rho^{(l)}(x) \le x^{-(\alpha + \delta)}
	\end{align*}
	for all $ x \le \eps_{0} $, where $ \eps_{0} = \min_{l} \eps^{(l)} > 0 $.
\end{corollary}
\begin{proof}
	Since $ \overline\rho^{(l)}(x) \sim x^{-\alpha^{(l)}}L^{(l)}(1/x) $, for any $ \eta>0 $, we can find $ \t\eps^{(l)} $ such that for all $ x \le \t\eps^{(l)} $
	\begin{align*}
		\overline\rho^{(l)}(x)
		\le (1+\eta)x^{-\alpha^{(l)}}L^{(l)}(1/x) .
	\end{align*}
	Note that $ (1+\eta)L^{(l)}(1/x) $ is still a slowly varying function. By \cref{lem: feller slowly varying bound}, there exists $ \eps^{(l)}\le \t\eps^{(l)} $ such that $ (1+\eta)L^{(l)}(1/x) \le x^{-\delta} $ for all $ x \le \eps^{(l)} $.
	Thus, for every $ x \le \eps^{(l)} $, we have
	$$
	\overline\rho^{(l)}(x) \le x^{-(\alpha^{(l)} + \delta)} .
	$$
	The second statement automatically follows.
\end{proof}
By \cref{cor: new potter type bound for all layer}, for any $ 0< \delta < 1-\alpha $, there exists $ \eps_{0}(\delta) $ such that for $ \eps < \eps_{0}(\delta) $,
\begin{align*}
	\int x \ind_{\{ x\le\eps \}} \, \rho^{(l)}(dx) \le \int_{0}^{\eps} \overline\rho^{(l)}(t) \, dt \le \frac{1}{1-(\alpha+\delta)}\eps^{1-(\alpha+\delta)} .
\end{align*}
Thus, for each $ l=1,\ldots,L$,
\begin{align*}
	\lim_{p_{l} \to \infty} A^{(l)}_{p_{l}} \le \frac{1}{1-(\alpha + \delta)} \eps^{1-(\alpha + \delta)} .
\end{align*}
By taking the limit of \cref{eq:error l2 bound eps} as $ \min\bp \to \infty $, we get
\begin{align*}
	&\lim_{\min\bp \to \infty} \bfE\left[\left( Z^{(l+1)}_{1}(\bx;\bp) - Z^{*(l+1)}_{1}(\bx;\bp) \right)^{2}\right] \\
	&\le ( \sigma_{v}^{2}{\Clip}U^{(l)} + (\sigma_{v}^{2}{\Clip})^{2}M^{(l)}U^{(l-1)} + \cdots + (\sigma_{v}^{2}{\Clip})^{l}M^{(l)}\cdots M^{(2)} U^{(1)} )
         \frac{\eps^{1-(\alpha+\delta)}}{1-(\alpha + \delta)} \\
	&= D(l) \eps^{1-(\alpha+\delta)}
\end{align*}
where $$ D(l) := \frac{\sigma_{v}^{2}{\Clip}}{1-(\alpha + \delta)}( U^{(l)} + (\sigma_{v}^{2}{\Clip})M^{(l)}U^{(l-1)} + \cdots + (\sigma_{v}^{2}{\Clip})^{l-1}M^{(l)}\cdots M^{(2)} U^{(1)} ) $$ is a constant not depending on $\eps$.

	\begin{remark}\label{rmk: continuity point}
		When checking \cref{eq: pruned conv to id}, it is necessary that the pruning level $ \eps $ is a continuity point of $ \rho^{(l)} $ for all $ l = 1,\ldots,L $. However, if $\eps$ is not a continuity point of some $ \rho^{(l)} $, we can find such a continuity point that is arbitrarily close to $ \eps $ (or arbitrarily small; see \cref{cor: new potter type bound for all layer}).
\end{remark}

\paragraph{Proof of \cref{th:kappapruning}.}
By \cref{prop:compressibilityID}, for each $ l=1,\ldots,L $ and any $ \kappa \in (0,1), $
\begin{align*}
	\sum_{j=1}^{p_{l}}\lambda^{(l)}_{p_{l},j} \ind_{\{ \lambda^{(l)}_{p_{l},j} \le \lambda^{(l)}_{p_{l},(\kappa p_{l})} \}} \cvpto 0
\end{align*}
as $ p_{l} \to \infty$.
As the family in \cref{eq:family sum lambda is ui} is uniformly integrable, it follows that, for each $ l=1,\ldots,L $ and $ \kappa \in (0,1), $
\begin{align*}
	A^{(l)}_{p_{l}} = \bfE\left[ \sum_{j=1}^{p_{l}}\lambda^{(l)}_{p_{l},j} \ind_{\{ \lambda^{(l)}_{p_{l},j} \le \lambda^{(l)}_{p_{l},(\kappa p_{l})} \}} \right] \to 0
\end{align*}
as $ p_{l} \to \infty $.
Thus, by taking the limit of \cref{eq:error l2 bound eps} as $ \min \bp \to \infty $, we get
\begin{align*}
	\lim_{\min\bp \to \infty} \bfE\left[\left( Z^{(l+1)}_{1}(\bx;\bp) - Z^{*(l+1)}_{1}(\bx;\bp) \right)^{2}\right] = 0.
\end{align*}

\subsection{Proof of \cref{th:multipleinputs}}

Let us denote $\vec\bx:=(\bx_1,\ldots,\bx_n)$, the tuple of the $n$ inputs in the theorem.
Throughout this subsection, let  $ \Sigma^{(l)}(\vec\bx) \in \R^{n\times n}$ be a (possibly random) covariance matrix depending on $\vec\bx$, and let $(\vec{\zeta}^{(l)}_{j}(\vec\bx))_{j\ge 1}$ be iid centred Gaussian random vectors in $\R^n$, given the covariance matrix $ \Sigma^{(l)}(\vec\bx)$.
Since the matrix $ \phi(\vec{\zeta}^{(l)}_{j}(\vec\bx))\phi(\vec{\zeta}^{(l)}_{j}(\vec\bx))^{T} $ is clearly positive semi-definite, we obtain the following corollary from \cref{thm:tensoring with cones}.
Consider $\vec z:=(z_1,\ldots,z_n)^T \in \R^n$. For a given activation function $\phi$, define the map $$ L_{\vec z}(u) = u\phi(\vec z) \phi(\vec z)^{T} : [0,\infty) \to \bbR^{n\times n}. $$

\begin{corollary}\label{technical lemma}
	Let $(\vec{\zeta}^{(l)}_{j}(\vec\bx))_{j}$ be iid centred Gaussian random vectors in $\R^n$ with covariance $ \Sigma^{(l)}(\vec\bx)$.
	Then, for $l\ge 1$,
	\[
	\sum_{j=1}^{p_{l}} \lambda^{(l)}_{p_{l},j} \phi\left(\vec{\zeta}^{(l)}_{j}(\vec\bx)\right) \phi\left(\vec{\zeta}^{(l)}_{j}(\vec{\bx})\right)^{T}
	\]
	converges in distribution, as $p_l\to\infty$, to $ S^{(l)}(\vec\bx) \sim \id\left( \tilde{a}^{(l)}, \t\rho^{(l)}\right)$
	concentrating on the cone $ \K_n $ of ${n\times n} $ positive semi-definite matrices, where, for $ l\ge1 $, $  \tilde{a}^{(l)} $ and $ \t\rho^{(l)} $ are defined as
	\begin{align*}
		&\tilde{a}^{(l)} := a^{(l)}\bfE\left[\phi\left(\vec\zeta^{(l)}_{j}(\vec\bx)\right)\phi\left(\vec\zeta^{(l)}_{j}(\vec\bx)\right)^{T} \right] \\
		&\t\rho^{(l)}(B) := \int ({L}_{\vec z})_{\star}(\rho^{(l)})(B) \, \xi(d\vec z)
	\end{align*}
	with $ \xi$ being the measure of $\cN(0,\Sigma^{(l)}(\vec\bx))$ (i.e., the law of $ \vec\zeta^{(l)}_{j}(\vec\bx)$)  and $ (L_{\vec z})_{\star}(\rho^{(l)}) $ denoting the pushforward of $ \rho^{(l)} $ under the map $  L_{\vec z} $.
\end{corollary}
Note that the above corollary is later used in the setting of a random covariance matrix. However, when it is used, we condition on the covariance, and thus for all intents and purposes, the covariance is nonrandom.

\begin{remark}
	Throughout the paper, we use two `dual' perspectives of viewing the measure $ \nu $. For a distribution $ F $ and a measure $ \rho $, the measure $ \nu(dz) = \int \rho(dx) F(dz/x) $ is a mixture of a pushforward distribution $ F(dz/x) $ (which maps each Borel set $B \subseteq \K_n$ to $F(\{z/x : z \in B\})$) with mixing measure $ \rho(dx) $. Note that, in some other places, we use $ \nu(dx) = \int (L_{z})_{\star}(\rho)(dx) F(dz) $, where $ L_{z}(x) = x z : \bbR \to \K_n $ for $ x \in [0,\infty) $ and $ z \in \K_n $, which can be described as a mixture of a pushforward of $ \rho $ by $L_z$ with mixing measure $ F(dz) $. Indeed, both definitions refer to the same measure: for a Borel set $ B \subset \K_n $, $ \nu(B) = (\rho\bigotimes F)(\{(x,z)\colon x z \in B\}) $, where $ \bigotimes $ denotes the product of measures.
\end{remark}

\paragraph{Proof of \cref{th:multipleinputs}.}	
	
	Recall that $\vec\bx=(\bx_{1},\ldots,\bx_{n})$ is the tuple of the $n$ inputs in the theorem.
	Let $ \vec t_{k} \in \bbR^{n} $ for $ k \le m $ and $ \vec 1 = (1,\ldots,1)^T \in \bbR^{n} $.
	
	We start by calculating, for $m \in \mathbb{N}$, $l\ge 2$ and $\bp \in \N^L$, the conditional characteristic function of the size-$m$ tuple of $n$-dimensional random vectors
	\begin{align*}
		(\vec{Z}^{(l)}_{k}(\vec\bx;\bp))_{k\le m} := (Z^{(l)}_{k}(\bx_{1};\bp),\ldots,Z^{(l)}_{k}(\bx_{n};\bp))_{k\le m} ,
	\end{align*}
	given $ \{ \lambda^{(l-1)}_{p_{l-1},j} \}_{j} $ and $ \{\vec{Z}^{(l-1)}_{j}\}_j=\{\vec{Z}^{(l-1)}_{j}(\vec\bx;\bp)\}_{j} $.
	Since both $ \{V^{(l)}_{jk}\}_{jk} $ and $ \{B^{(l)}_{k}\}_{k} $ are iid and also independent of each other, we have
	\begin{align*}
		&\bfE \left[ \prod_{k\le m} \left. \exp\!\left( i\langle \vec{t}_{k} , B^{(l)}_{k}\vec{1} \rangle + i\sum_{j=1}^{p_{l-1}} \Bigr\langle \vec{t}_{k} , \sqrt{\lambda^{(l-1)}_{p_{l-1},j}}V^{(l)}_{jk} \phi(\vec{Z}^{(l-1)}_{j}) \Bigr\rangle \right) \right| \{ \lambda^{(l-1)}_{p_{l-1},j} \}_{j}, \{\vec{Z}^{(l-1)}_{j}\}_{j} \right] \\
		& {=}\, \prod_{k\le m} \bfE\!\left[\exp\!\left( i\langle \vec{t}_{k} , B^{(l)}_{k}\vec{1} \rangle \right)\right] \prod_{j=1}^{p_{l-1}} \bfE\!\left[ \left. \exp\!\left( i \Bigr\langle \vec{t}_{k} , \sqrt{\lambda^{(l-1)}_{p_{l-1},j}}V^{(l)}_{jk} \phi(\vec{Z}^{(l-1)}_{j}) \Bigr\rangle \right) \right| \{ \lambda^{(l-1)}_{p_{l-1},j} \}_{j}, \{\vec{Z}^{(l-1)}_{j}\}_{j} \right] \\
		& {=}\, \prod_{k\le m} \exp\!\left( -\frac{1}{2} \left\langle \vec{t}_{k} , \left[ \sigma_b^2 \vec{1} \vec{1}^{T} + \sigma_v^{2} \sum_{j=1}^{p_{l-1}} \lambda^{(l-1)}_{p_{l-1},j} \phi(\vec{Z}^{(l-1)}_{j})\phi(\vec{Z}^{(l-1)}_{j})^{T} \right] \vec{t}_{k} \right\rangle \right) .
	\end{align*}
	Thus, the (unconditional) characteristic function $\psi_{(\vec{Z}^{(l)}_{k}(\vec\bx;\bp))_{k\le m}}$ satisfies
	\begin{multline}
		 \psi_{(\vec{Z}^{(l)}_{k}(\vec\bx;\bp))_{k\le m}}(\vec{t}_{1},\ldots,\vec{t}_{m})
		\\
		{}
		=
		\bfE\left[ \prod_{k\le m} \exp\!\left( -\frac{1}{2} \left\langle \vec{t}_{k} , \left[ \sigma_b^2 \vec{1} \vec{1}^{T} + \sigma_v^{2} \sum_{j=1}^{p_{l-1}} \lambda^{(l-1)}_{p_{l-1},j} \phi(\vec{Z}^{(l-1)}_{j})\phi(\vec{Z}^{(l-1)}_{j})^{T} \right] \vec{t}_{k} \right\rangle \right) \right].
		\label{eq: ch.f. multi input}
	\end{multline}

	We now prove the following convergence in distribution using induction on the layer $l = 1,\ldots,(L+1)$: for all $m \in \mathbb{N}$, where we henceforth assume without loss of generality that $m\le p_l$,
         \[
         \left(\lim_{p_{l-1}\to\infty}\ldots\lim_{p_1\to\infty}
         (  \vec{Z}_{k}^{(l)}(\vec\bx;\mathbf{p}))_{k \leq m}\right)
         \overset{\mathtt{d}}{=}
         ( \vec{\zeta}_{k}^{(l)}(\vec\bx))_{k \leq m}
         \]
         where $( \vec{\zeta}_{k}^{(l)}(\vec\bx))_{k \leq m}$  is the size-$m$ tuple of the
    random vectors $\vec \zeta^{(l)}_k(\vec\bx) = (\zeta^{(l)}_k(\bx_1),\ldots,\zeta^{(l)}_k(\bx_n))^T$ in $\R^n$  defined as follows. When conditioned on $\Sigma^{(l)}$, the tuple $( \vec{\zeta}_{k}^{(l)}(\vec\bx))_{k \leq m}$ is distributed as $\otimes_{k \leq m} \cN(0,\Sigma^{(l)})$, where $\Sigma^{(l)}$ is a random covariance matrix defined recursively by the kernels in \cref{eq: multi input kernel inductive def}. In other words, conditioned on $\Sigma^{(l)}$, the $m$ components of $( \vec{\zeta}_{k}^{(l)}(\vec\bx))_{k \leq m}$ are iid Gaussian vectors with covariance matrix  $\Sigma^{(l)}$, while if we do not condition, then the joint distribution $( \vec{\zeta}_{k}^{(l)}(\vec\bx))_{k \leq m}$ is a convex mixture of iid Gaussian vectors with covariance matrix  $\Sigma^{(l)}$ with the mixture being governed by the randomness of $\Sigma^{(l)}$.
        Since $(\vec{Z}_{k}^{(l)}(\vec\bx;\mathbf{p}))_{k \leq m}$ does not depend on $p_l,\ldots,p_L$, (as long as $p_l$ is bigger or equal to $m$)
        the above convergence implies the claim of the theorem. In our inductive proof, we explicitly denote the dependency of $\Sigma^{(l)}$ on the inputs $\vec\bx$ by writing $\Sigma^{(l)}(\vec\bx)$. Also, we fix $m \in \mathbb{N}$.
	\\
	
	\noindent\textbf{Case $l = 1$:}\\
        Note that  in this case,	for all $\bp \in \mathbb{N}^L$,
        both $(\vec{Z}^{(1)}_k(\vec\bx;\bp))_{k \geq 1}$ and $(\vec\zeta^{(1)}_k(\vec\bx))_{k \geq 1}$ are iid even without conditioning since $\Sigma^{(1)}$ is nonrandom. Each component follows the law $ \cN(0,\Sigma^{(1)}(\vec\bx))$ and
        \[
	(\vec{Z}^{(1)}_k(\vec\bx;\bp))_{k \leq m} \overset{\mathtt{d}}{=} (\vec\zeta^{(1)}_k(\vec\bx))_{k \leq m},
	\]
	holds for all $\bp \in \mathbb{N}^L$.
        \\

	\noindent\textbf{Case $l\geq 2$:}\\
        By the induction hypothesis, we have the following convergence in distribution: for all $m' \in \mathbb{N}$, where we assume without loss of generality that $m'\le p_{l-1}$,
	\begin{equation}
	\label{eqn:multi-input:induction-hypo}
	    \left(\lim_{p_{l-2}\to\infty}\ldots\lim_{p_1\to\infty}
		(\vec{Z}^{(l-1)}_{j}(\vec\bx;\bp))_{j\le m'}\right)
		\overset{\mathtt{d}}{=}
		(\vec{\zeta}^{(l-1)}_{j}(\vec\bx))_{j\le m'},
	\end{equation}
where the left-hand side is interpreted as $(\vec{Z}^{(1)}_{j}(\vec\bx;\bp))_{j\le m'}$ when $l=2$.

For all non-negative reals $c_1,\ldots,c_{p_{l-1}}$,
and for all $m \in \mathbb{N}$ and $\vec{t}_1,\ldots,\vec{t}_m \in\R^n$,
the function $h:\mathbb R^{n\times p_{l-1}} \to (0,\infty)$ defined by
\begin{multline}
		 h\left ((\vec{z}_j)_{j \leq p_{l-1}};(\vec{t}_k)_{k\leq m},(c_j)_{j\leq p_{l-1}}\right )= {}
		 \\ \prod_{k\le m} \exp\left( -\frac{1}{2} \left\langle \vec{t}_{k} , \left[ \sigma_b^2 \vec{1} \vec{1}^{T} + \sigma_v^{2} \sum_{j=1}^{p_{l-1}} c_j \phi(\vec{z}_j)\phi(\vec{z}_j)^{T} \right] \vec{t}_{k} \right\rangle \right),
	\end{multline}	
 is continuous in $(\vec{z}_j)_{j \leq p_{l-1}}$, non-negative, and bounded by $1$.
  Thus, we can use the induction hypothesis for the $(l{-}1)$-th layer in \cref{eqn:multi-input:induction-hypo} and, from the definition of convergence in distribution,
  deduce the following convergence for this bounded continuous function $h$:
\begin{multline*}
\lim_{p_{l-2} \to \infty}\ldots \lim_{p_{1}\to \infty} \bfE \left [ h\left ((\vec{Z}^{(l-1)}_{j}(\vec\bx;\bp'))_{j\le p_{l-1}};(\vec{t}_k)_{k\leq m},(c_j)_{j\leq p_{l-1}}\right )\right ]= {}
\\
\bfE \left [ h\left ((\vec{\zeta}^{(l-1)}_{j}(\vec\bx))_{j\le p_{l-1}};(\vec{t}_k)_{k\leq m},(c_j)_{j\leq p_{l-1}}\right )\right ].
\end{multline*}
Also, by \cref{eq: ch.f. multi input},
 \begin{align*}
& \psi_{(\vec{Z}^{(l)}_{k}(\vec\bx;\bp))_{k\le m}}(\vec{t}_{1},\ldots,\vec{t}_{m})
\\
& \qquad {} =
\bfE \left [ h\left ((\vec{Z}^{(l-1)}_{j}(\vec\bx;\bp))_{j\le p_{l-1}};(\vec{t}_k)_{k\leq m},(\lambda^{(l-1)}_{p_{l-1},j})_{j\leq p_{l-1}}\right )\right ]
\\
& \qquad {} =\bfE\left [\bfE \left [\left. h\left ((\vec{Z}^{(l-1)}_{j}(\vec\bx;\bp))_{j\le p_{l-1}};(\vec{t}_k)_{k\leq m},(\lambda^{(l-1)}_{p_{l-1},j})_{j\leq p_{l-1}}\right ) \,\right|\, (\lambda^{(l-1)}_{p_{l-1},j})_{j\leq p_{l-1}}\right ] \right ].
	\end{align*}	
Since $h$ is bounded,
by the dominated convergence theorem, we have
 \begin{align*}
 &
\lim_{p_{l-2} \to \infty} {...} \lim_{p_{1}\to \infty}		
\psi_{(\vec{Z}^{(l)}_{k}(\vec\bx;\bp))_{k\le m}}(\vec{t}_{1},\ldots,\vec{t}_{m})
\\
& {} =
\lim_{p_{l-2} \to \infty} {...} \lim_{p_{1}\to \infty}	
\bfE\left [\bfE \left [\left. h\left ((\vec{Z}^{(l-1)}_{j}(\vec\bx;\bp))_{j\le p_{l-1}};(\vec{t}_k)_{k\leq m},(\lambda^{(l-1)}_{p_{l-1},j})_{j\leq p_{l-1}}\right ) \,\right|\, (\lambda^{(l-1)}_{p_{l-1},j})_{j\leq p_{l-1}}\right ] \right ]
\\
& {} =
\bfE\left [
\lim_{p_{l-2} \to \infty} {...} \lim_{p_{1}\to \infty}	
\bfE \left [\left. h\left ((\vec{Z}^{(l-1)}_{j}(\vec\bx;\bp))_{j\le p_{l-1}};(\vec{t}_k)_{k\leq m},(\lambda^{(l-1)}_{p_{l-1},j})_{j\leq p_{l-1}}\right ) \,\right|\, (\lambda^{(l-1)}_{p_{l-1},j})_{j\leq p_{l-1}}\right ] \right ]
\\
& {} =
\bfE \left [
\bfE\left[\left.
h\left ((\vec{\zeta}^{(l-1)}_{j}(\vec\bx))_{j\le p_{l-1}};(\vec{t}_k)_{k\leq m},(\lambda^{(l-1)}_{p_{l-1},j})_{j\leq p_{l-1}}\right )\,\right|\, (\lambda^{(l-1)}_{p_{l-1},j})_{j\leq p_{l-1}}\right ] \right ]
\\
& {} =
\bfE \left [ h\left ((\vec{\zeta}^{(l-1)}_{j}(\vec\bx))_{j\le p_{l-1}};(\vec{t}_k)_{k\leq m},(\lambda^{(l-1)}_{p_{l-1},j})_{j\leq p_{l-1}}\right )\right ].
	\end{align*}
To complete the inductive step, we now show that
\[
\lim_{p_{l-1}\to\infty}\bfE \left [ h\left ((\vec{\zeta}^{(l-1)}_{j}(\vec\bx))_{j\le p_{l-1}};(\vec{t}_k)_{k\leq m},(\lambda^{(l-1)}_{p_{l-1},j})_{j\leq p_{l-1}}\right )\right ]= \psi_{(\vec{\zeta}^{(l)}_{k}(\vec\bx))_{k\le m}}(\vec{t}_{1},\ldots,\vec{t}_{m}).
\]
Since $h$ is bounded, by the dominated convergence theorem, we have
\begin{align*}
&
\lim_{p_{l-1}\to\infty}\bfE \left [ h\left ((\vec{\zeta}^{(l-1)}_{j}(\vec\bx))_{j\le p_{l-1}};(\vec{t}_k)_{k\leq m},(\lambda^{(l-1)}_{p_{l-1},j})_{j\leq p_{l-1}}\right )\right ]
\\
& \qquad {} =
\lim_{p_{l-1}\to\infty}\bfE \left [
\bfE \left[ \left.
h\left((\vec{\zeta}^{(l-1)}_{j}(\vec\bx))_{j\le p_{l-1}};(\vec{t}_k)_{k\leq m},(\lambda^{(l-1)}_{p_{l-1},j})_{j\leq p_{l-1}}\right)
\,\right|\,\Sigma^{(l-1)}(\vec\bx)\right]\right ]
\\
& \qquad {} =
\bfE \left [
\lim_{p_{l-1}\to\infty}
\bfE \left[ \left.
h\left((\vec{\zeta}^{(l-1)}_{j}(\vec\bx))_{j\le p_{l-1}};(\vec{t}_k)_{k\leq m},(\lambda^{(l-1)}_{p_{l-1},j})_{j\leq p_{l-1}}\right)
\,\right|\,\Sigma^{(l-1)}(\vec\bx)\right]\right ]
\end{align*}
where the nested conditional expectation has the following form by the definition of $h$:
\begin{multline*}
\bfE \left[ \left.
h\left((\vec{\zeta}^{(l-1)}_{j}(\vec\bx))_{j\le p_{l-1}};(\vec{t}_k)_{k\leq m},(\lambda^{(l-1)}_{p_{l-1},j})_{j\leq p_{l-1}}\right)
\,\right|\,\Sigma^{(l-1)}(\vec\bx)\right] = {}
\\
\bfE \left[ \left.
\prod_{k\le m} \!\exp\!\left( -\frac{1}{2} \left\langle \vec{t}_{k} , \left[ \sigma_b^2 \vec{1} \vec{1}^{T} + \sigma_v^{2} \sum_{j=1}^{p_{l-1}} \lambda^{(l-1)}_{p_{l-1},j} \phi(\vec{\zeta}^{(l-1)}_{j}(\vec\bx))\phi(\vec{\zeta}^{(l-1)}_{j}(\vec\bx))^{T} \right] \vec{t}_{k} \right\rangle\!\right)
\,\right|\,\Sigma^{(l-1)}(\vec\bx)\right].
\end{multline*}
Note that by \cref{technical lemma}, when conditioned on $\Sigma^{(l-1)}(\vec\bx)$,
\[
\left(
\sum_{j=1}^{p_{l-1}} \lambda^{(l-1)}_{p_{l-1},j} \phi(\vec{\zeta}^{(l-1)}_{j}(\vec\bx))\phi(\vec{\zeta}^{(l-1)}_{j}(\vec\bx))^{T} \right)
	\cvdto {S}^{(l-1)}(\vec\bx) \in \bbR^{n\times n}
	\ \text{as}\ p_{l-1} \to \infty,
\]
where the random matrix ${S}^{(l-1)}(\vec\bx)$ has an infinitely divisible distribution with L\'evy characteristic $(\tilde{a}^{(l-1)}, \t\rho^{(l-1)}) $ as defined in \cref{technical lemma}. Since
$\prod_{k \leq m} \exp(-y_k/2)$ is bounded by $1$ for all non-negative reals $y_1,\ldots,y_k$, the above convergence implies that as $p_{l-1}$ tends to $\infty$,
\begin{multline*}
\bfE \left[ \left.
\prod_{k\le m} \!\exp\!\left( -\frac{1}{2} \left\langle \vec{t}_{k} , \left[ \sigma_b^2 \vec{1} \vec{1}^{T} + \sigma_v^{2} \sum_{j=1}^{p_{l-1}} \lambda^{(l-1)}_{p_{l-1},j} \phi(\vec{\zeta}^{(l-1)}_{j}(\vec\bx))\phi(\vec{\zeta}^{(l-1)}_{j}(\vec\bx))^{T} \right] \vec{t}_{k} \right\rangle\!\right)
\,\right|\,\Sigma^{(l-1)}(\vec\bx)\right]
\\
{} \to
\bfE \left[ \left.
\prod_{k\le m} \!\exp\!\left( -\frac{1}{2} \left\langle \vec{t}_{k} , \left[ \sigma_b^2 \vec{1} \vec{1}^{T} + \sigma_v^{2} S^{(l-1)}(\vec\bx)\right] \vec{t}_{k} \right\rangle\!\right)
\,\right|\,\Sigma^{(l-1)}(\vec\bx)\right].
\end{multline*}
Thus,
\begin{align*}
&
\lim_{p_{l-1}\to\infty}\bfE \left [ h\left ((\vec{\zeta}^{(l-1)}_{j}(\vec\bx))_{j\le p_{l-1}};(\vec{t}_k)_{k\leq m},(\lambda^{(l-1)}_{p_{l-1},j})_{j\leq p_{l-1}}\right )\right ]
\\
& \qquad {} =
\bfE \left[
\bfE \left[ \left.
\prod_{k\le m} \exp\left( -\frac{1}{2} \left\langle \vec{t}_{k} , \left[ \sigma_b^2 \vec{1} \vec{1}^{T} + \sigma_v^{2} S^{(l-1)}(\vec\bx)\right] \vec{t}_{k} \right\rangle\right)
\,\right|\,\Sigma^{(l-1)}(\vec\bx)\right] \right]
\\
& \qquad {} =
\bfE \left[
\bfE \left[ \left.
\prod_{k\le m} \exp\left( -\frac{1}{2} \left\langle \vec{t}_{k} , \Sigma^{(l)}(\vec\bx) \vec{t}_{k} \right\rangle\right)
\,\right|\,\Sigma^{(l-1)}(\vec\bx)\right] \right]
\\
& \qquad {} =
\bfE \left[
\prod_{k\le m} \exp\left( -\frac{1}{2} \left\langle \vec{t}_{k} , \Sigma^{(l)}(\vec\bx) \vec{t}_{k} \right\rangle\right)
\right]
\\
& \qquad {} = \psi_{(\vec{\zeta}^{(l)}_{k}(\vec\bx))_{k\le m}}(\vec{t}_{1},\ldots,\vec{t}_{m}).
\end{align*}
The last equality follows from the definition of $(\vec{\zeta}^{(l)}_{k}(\vec\bx))_{k\le m}$ and the fact that when conditioned on $\Sigma^{(l)}(\vec\bx)$, the random variables $(\vec{\zeta}^{(l)}_{k}(\vec\bx))_{k\le m}$ are iid with each component having the distribution $\mathcal{N}(0,\Sigma^{(l)}(\vec\bx))$. The justification of the second equality is slightly more involved. It follows from the boundedness of $\prod_{k \leq m}\exp(-y_k/2)$ for all non-negative reals $y_1,\ldots,y_k$, and the below conditional distributional equality:
when conditioned on $\Sigma^{(l-1)}(\vec\bx)$,
\[
\sigma_b^2 \vec{1} \vec{1}^{T} + \sigma_v^{2} S^{(l-1)}(\vec\bx)
\overset{\mathtt{d}}{=}
\Sigma^{(l)}(\vec\bx),
\]
which follows from the definition of $ \t\rho^{(l-1)} $.
To see this, condition on $\Sigma^{(l-1)}(\vec\bx)$. Then,
 as $  S^{(l-1)}(\vec\bx) $ follows $ \operatorname{ID}( \tilde{a}^{(l-1)}, \t\rho^{(l-1)}) $, it can be represented as $  \tilde{a}^{(l-1)} + \sum_{j\ge 1}  X_{j} $ where $ ( X_{j})_{j \geq 1} $ are points in $\bbR^{n\times n}$ which result from the pushforward of a Poisson process on $(0,\infty)$ with mean measure $\rho^{(l-1)} $,  as described in \cref{technical lemma}.
	Since $ \t\rho^{(l-1)} $ is a pushforward of $ \rho^{(l-1)}(du) $,
        $X_{j}$ can be represented as $ \widetilde{\lambda}^{(l-1)}_{j} \phi(\vec\zeta^{(l-1)}_{j}(\vec\bx)) \phi(\vec\zeta^{(l-1)}_{j}(\vec\bx))^{T} $. Thus, under our assumed conditioning on $\Sigma^{(l-1)}(\vec\bx)$, we have
\begin{align*}
\sigma_b^2 \vec{1} \vec{1}^{T} + \sigma_v^{2} S^{(l-1)}(\vec\bx)
&\overset{\mathtt{d}}{=}
\sigma_{b}^{2}\vec 1 \vec 1^{T} + \sigma_{v}^{2}\tilde{a}^{(l-1)} + \sigma_{v}^{2} \sum_{j\ge 1} \widetilde{\lambda}^{(l-1)}_{j} \phi(\vec\zeta^{(l-1)}_{j}(\vec\bx)) \phi(\vec\zeta^{(l-1)}_{j}(\vec\bx))^{T}  &\overset{\mathtt{d}}{=}
\Sigma^{(l)}(\vec\bx).
\end{align*}
This completes the proof of the inductive case.

\newpage

\section{Additional Theoretical Results}
\label{app:additionalresults}

\subsection{Properties of Small Weights in our Model}
\label{app:properties of small weights}
The following proposition characterises the rate of decay of the variances/weights, under a polynomial decay of the tail L\'evy measure at 0.

\begin{proposition}[Asymptotic properties of small variances and weights]\label{prop:regvarweightszero}
Assume $\rho^{(l)}$ is an infinite L\'evy measure with tail $\overline\rho^{(l)}(x)\overset{x\to 0}{\sim}c  x^{-\alpha}$ for some $\alpha\in(0,1)$ and some constant $c >0$. Let $\Phi^{(l)}_{(k)}$ and $\Psi^{(l+1)}_{(k),m}$ be random variables in \cref{prop: extremes} such that $\lambda_{p_l,(k)}^{(l)}\cvdto \Phi^{(l)}_{(k)}$ and $(W^{(l+1)}_{(k),m})^2 \cvdto \Psi^{(l+1)}_{(k),m}$ as $p_l$ tends to $\infty$. Then, in probability,
\begin{align}
\Phi^{(l)}_{(k)} & \overset{k\to\infty}{\sim} (\overline\rho^{(l)})^{-1}(k)\overset{k\to\infty}{\sim} c ^{1/\alpha}k^{-1/\alpha}\\
\intertext{and for any $m\geq 1$, in probability,}
\Psi^{(l+1)}_{(k),m} &\overset{k\to\infty}{\sim}\sigma_v^2\times (\overline\nu^{(l)})^{-1}(k) \overset{k\to\infty}{\sim}\left ( \frac{2^\alpha\Gamma(\alpha+1/2)}{\sqrt{\pi}}(\sigma_v^2)^\alpha c  \right)^{1/\alpha}  k^{-1/\alpha}\, .
\end{align}
\end{proposition}
\begin{proof}
By \cref{prop: extremes}, $\lambda_{p_l,(k)}^{(l)}\cvdto (\overline\rho^{(l)})^{-1}(G_k)$ as $p_l\to\infty$, where $G_k\sim\gammadist(k,1)$. Additionally, by the law of large numbers $\frac{G_k}{k}\cvpto 1$ as $k\to\infty$. Also, $\overline\rho^{(l)}(x)\overset{x\to 0}{\sim}c x^{-\alpha}$ implies $(\overline\rho^{(l)})^{-1}(x)\overset{x\to \infty}{\sim} (x/c)^{-1/\alpha}$. The rest follows from properties of regularly varying functions, see \cref{prop:resnickProp5} in \cref{app:backgroundRV}.
Additionally, similarly to the proof of \cref{prop:regvarweightsinf}, $\overline\rho^{(l)}(x)\overset{x\to 0}{\sim}c x^{-\alpha}$ implies
\begin{align*}
\overline\nu^{(l)}(x)&\overset{x\to 0}{\sim}\frac{2^\alpha\Gamma(\alpha+1/2)}{\sqrt{\pi}} x^{-\alpha}c
\end{align*}
concluding the proof.
\end{proof}

\subsection{Infinite-Width Limit for Multiple Inputs in the Symmetric $\alpha$-Stable Case}
\label{app:multiple-input-results-stable-case}

If, for some $\alpha\in(0,1)$,
\[
        \sum_{j=1}^{p_l} \lambda_{p_l,j}^{(l)}\cvdto \stable(\alpha,1)\text{ as } p_l \to \infty,
\]
then the Poisson point process $\{\widetilde{\lambda}_{j}^{(l)}\}_{j\geq1}$ in \cref{th:multipleinputs} has mean measure
\[
         \rho^{(l)}(du)=\frac{\alpha}{\Gamma(1-\alpha)} u^{-\alpha-1}du.
\]

Let $\K_n$ denote the set of $n$-by-$n$ positive semi-definite matrices
and define the limit in distribution
\[
        \zeta_{k}^{(l)}(\bx) = \lim_{p_{l-1}\to\infty}\ldots\lim_{p_1\to\infty} Z_k^{(l)}(\bx;\bp)
\]
which we recall is conditionally Gaussian given the previous layers $1,\ldots,l-1$.
If $\phi(\zeta_{k}^{(l)}(\bx))$ has sufficient moments,  then for all $n$ inputs $\bx_1,\ldots,\bx_n \in \R^\din$,
the $n$-by-$n$ random matrix
\begin{equation}
\label{def:S}
\begin{aligned}
{R}^{(l)}
& {} := \left(K^{(l)}(\bx_i,\bx_j) - \sigma_b^2\right)_{i,j =1}^n
\\
& {} = \left(\sigma_{v}^{2}\sum_{k\geq1}^{{}}\widetilde{\lambda}_{k}^{(l)}\phi\left(  \zeta_{k}^{(l)}(\bx_i)\right)
\phi\left(  \zeta_{k}^{(l)}(\bx_j)\right)\right)_{i,j=1}^n
\end{aligned}
\end{equation}
 has a (strictly) $\alpha$-stable distribution in $\K_n$ with characteristic exponent $\alpha$, i.e.,
 $$c_1^{1/\alpha} R^{(l)}_1+c_2^{1/\alpha}R^{(l)}_2 \overset{\cvd}{=} (c_1+c_2)^{1/\alpha} R^{(l)},$$
 where $R^{(l)}_1$ and $R^{(l)}_2$ are independent copies of $R^{(l)}$. This allows us to strengthen \cref{th:multipleinputs} in this special case.
 For a single input $\bx\in\R^{\din}$, we obtain a more precise form of the limiting output distribution.

\begin{theorem}[$\alpha$-stable case, further results]\label{thm: stable case}
	Suppose  that in the $l$-th hidden layer,
	\begin{align}\label{eq:stable conv}
		\sum_{j=1}^{p_l} \lambda_{p_l,j}^{(l)}\overset{\cvd}{\to} \stable(\alpha,1) \text{ as } p_l \to \infty	
	\end{align}
	for $\alpha\in(0,1]$.
	Then, for any inputs $\bx_1,\ldots,\bx_n$,
        \cref{th:multipleinputs} holds with the random
        matrix $(K^{(l)}(\bx_i,\bx_j) - \sigma_b^2)_{i,j=1}^n$ in \cref{def:S} having a conditional distribution, given $K^{(1)},\ldots,K^{(l-1)}$, that is
        $\alpha$-stable in $\K_n$. Moreover, in the case $n=1$ with the single input $\bx$, the conditional distribution of $ K^{(l)}(\bx,\bx) - \sigma_b^2$ is $\stable(\alpha,r^{(l)}(\bx))$ where
	\begin{align*}	
		r^{(l)}(\bx)&:=\sigma_v^2 \cdot \left(\E\left[\left.\left|\phi(\zeta^{(l)}_1(\bx))\right|^{2\alpha} \,\right|\, K^{(1)},\ldots,K^{(l-1)}\right]\right)^{1/\alpha}.
	\end{align*}
        Thus, given $K^{(1)},\ldots,K^{(l-1)}$, the random variance $K^{(l)}(\bx,\bx)$ has the same conditional distribution as
        \[
                \sigma_b^2 + \sigma_v^2 \cdot \tilde{\Lambda} \cdot
		\left(\E\left[\left.\left|\phi(\zeta^{(l)}_1(\bx))\right|^{2\alpha} \,\right|\, K^{(1)},\ldots,K^{(l-1)}\right]\right)^{1/\alpha}
        \]
        for $\tilde{\Lambda} \sim \stable(\alpha,1)$.
\end{theorem}
\begin{proof}
	We assume, without loss of generality, that $(\widetilde{\lambda}_{k}^{(l)})_{k\geq1}$ is ordered:
	\[
	\widetilde{\lambda}_{1}^{(l)}\ge \widetilde{\lambda}_{2}^{(l)}\ge \cdots.
	\]
	By \cref{prop:extremesID}, the order statistics $ (\lambda^{(l)}_{p_l,(1)},\ldots,\lambda^{(l)}_{p_l,(p_l)},0,\ldots)$ converge in distribution to  $(\widetilde\lambda^{(l)}_{k})_{k\ge 1}$, a Poisson point process with intensity measure $\rho^{(l)}(du)=\alpha \Gamma(1-\alpha)^{-1} u^{-\alpha-1} du$. Such a Poisson process takes the form
	$((G_k \Gamma(1-\alpha))^{-1/\alpha})_{k\ge 1}$ where $(G_k)_{k\ge 1}$ is a standard rate-one Poisson point process on $(0,\infty)$ \cite{Lepage81, davydov2008strictly}.
	
	Given   ${}{K}^{(1)},\ldots,{}{K}^{(l-1)}$, we have that $\bigl((\zeta_k^{(l)}(\bx_1),\ldots,\zeta_k^{(l)}(\bx_n))^T\bigr)_{k\ge 1}$ are iid with common distribution $\cN(0,{}{\Sigma}^{(l)})$. Since $\phi$ satisfies the polynomial envelope condition, we have, conditioned on  ${}{K}^{(1)},\ldots,{}{K}^{(l-1)}$, that $	\phi\bigr(  \zeta_{1}^{(l)}(\bx_i)\bigr)\phi\bigr(\zeta_{1}^{(l)}(\bx_j)\bigr)$ has conditional moments of all order for all pairs $(i,j)$ (importantly, recall that we defined the conditional expectation via regular conditional probabilities). By the three-series theorem,
	$\sum_{k\geq1}^{{}}\widetilde{\lambda}_{k}^{(l)}\phi\bigr(  \zeta_{k}^{(l)}(\mathbf{x}_i)\bigr)
	\phi\bigr(  \zeta_{k}^{(l)}(\mathbf{x}_j)\bigr)$ converges for all pairs $(i,j)$. Therefore, by Theorem 2 in \cite{Lepage81}, ${}{R}^{(l)}$ is $\alpha$-stable.

	For the special case $n=1$ with the single input $\bx$, setting $U_k:=\phi^2\bigr(  \zeta_{k}^{(l)}(\mathbf{x})\bigr)$, by \cref{lemma:idlaplaceproduct}, the L\'evy-Khintchine formula and the integral formula \cite[p. 15]{samorodnitsky1994stable}
	$$ y^\alpha = \int_0^\infty (1 - e^{-xy}) \frac{\alpha}{\Gamma(1-\alpha)} x^{-\alpha-1} dx \quad\text{ for } \ \ 0 < \alpha < 1$$
	we have that
	\begin{align*}
	\E\left[\exp\left(-t\sum_k \widetilde{\lambda}_{k}^{(l)} U_k\right)\right] &=\exp\left( -\E\left[\int_0^\infty\left(1-e^{-txU_1}\right)\frac{\alpha}{\Gamma(1-\alpha)} x^{-\alpha-1} dx\right]\right)\\
	& = \exp\left(-t^{\alpha}\E[U_1^\alpha]\right).
	\end{align*}
\end{proof}

\newpage

\section{Additional Details on the Examples}
\label{app:examples_all}

\subsection{Detailed Illustration of the Main Results on the Simple Model in \cref{sec:intro}}\label{app:illustrative_example}
Recall the four models briefly discussed in \cref{sec:intro}: for $p_1\geq 2$,
\begin{center}
\begin{tabular}{ll}
  (a) $\lambda^{(1)}_{p_1,j}\sim \IG\left(2, \frac{2}{p_1} \right)$ & (b) $\lambda^{(1)}_{p_1,j}\sim \Ber\left(\frac{2}{p_1} \right)$ \\
        (c) $\lambda^{(1)}_{p_1,j}\sim \betadist\left(\frac{1}{p_1},\frac{1}{2}\right)$ & (d) $\lambda^{(1)}_{p_1,j}=\pi^2\frac{U_j^2}{p^2_1}\text{ where }U_j\sim \Cauchy(0,1)$
\end{tabular}
\end{center}
where $\IG(\beta_1,\beta_2)$ denotes the inverse gamma distribution with shape $\beta_1>0$ and scale $\beta_2>0$, and $\Cauchy(0,1)$ denotes the half-Cauchy distribution with pdf
in \cref{eq:halfcauchypdf}.
The 50 largest values of a realisation of $(\lambda^{(1)}_{p_1,j})_{j=1,\ldots,p_1}$ for a neural network of width $p_1=5000$ are represented in \cref{fig:exampleplotlambda} under these   models.

These four models have different infinite-width limits. Under the inverse gamma model (a), the infinite-width limit is the same as the iid Gaussian case, and \cref{eq:exampleGPlimit} holds. Under models (b-d), the infinite-width limit is a mixture of Gaussian processes (see \cref{th:multipleinputs}). That is, for each case $s\in\{b,c,d\}$,
\begin{equation}
        \left(
        \begin{array}{cc}
                Z^{(2)}_{k}(\bx;p_1) \\
                Z^{(2)}_{k}(\bx';p_1) \\
        \end{array}
        \right)
        \overset{\cvd}{\to}
        \mathcal N \left (0, \left(
        \begin{array}{cc}
                K_s^{(2)}(\bx,\bx) & K_s^{(2)}(\bx,\bx') \\
                K_s^{(2)}(\bx,\bx') & K_s^{(2)}(\bx',\bx') \\
        \end{array}
        \right)\right )\text{ as }p_1\to\infty
\end{equation}
where $K_b^{(2)}$, $K_c^{(2)}$ and $K_d^{(2)}$ are \textit{random} covariance kernels defined by
\begin{align}\label{eq:kernelexamples}
K_s^{(2)}(\mathbf{x},\mathbf{x}^{\prime})&:=\sum_{j\geq1}^{{}%
}\widetilde{\lambda}_{(j),s}^{(1)}\max\left(0,  \zeta_{(j),s}^{(1)}(\mathbf{x})\right)
\max\left(0,  \zeta_{(j),s}^{(1)}(\mathbf{x}^{\prime})\right). %
\end{align}
Here $\widetilde{\lambda}_{(1),s}^{(1)}\geq \widetilde{\lambda}_{(2),s}^{(1)}\geq \ldots \geq 0$ are random weights defined by (see details in \cref{app:additionalderivationsExamples})
\begin{align}
\widetilde{\lambda}_{(j),b}^{(1)}&=\left \{\begin{array}{cc}
                                    1 & \text{if }j\leq N^{(1)} \\
                                    0 & \text{otherwise}
                                  \end{array}\right .
&\text{where }N^{(1)}&\sim\Poi(2)\\
\widetilde{\lambda}_{(j),c}^{(1)}&=\prod_{k=1}^j \beta_k&\text{where } \beta_k&\overset{\text{iid}}{\sim}\betadist(2,1)\\
\widetilde \lambda_{(j),d}&=\left (\sum_{k=1}^j E_j \right)^{-2}&\text{where } E_k&\overset{\text{iid}}{\sim}\text{Exponential}(2)
\end{align}
and, for $j\geq 1$, $\zeta_{(j),s}^{(1)}\overset{\text{iid}}{\sim}\GP(0,K^{(1)})$ with $K^{(1)}(\mathbf{x},\mathbf{x}^{\prime})=\frac{\mathbf{x}^{T}\mathbf{x}^{\prime}}{\din}$.

 We have $\bbE [K_b^{(2)}]=\bbE [K_c^{(2)}]=\mathcal K^{(2)}$. In the case (d), $\bbE[K_d^{(2)}]$ is undefined.

\begin{figure}
  \centering
  \subfigure[Inverse Gamma]{\includegraphics[width=.33\textwidth]{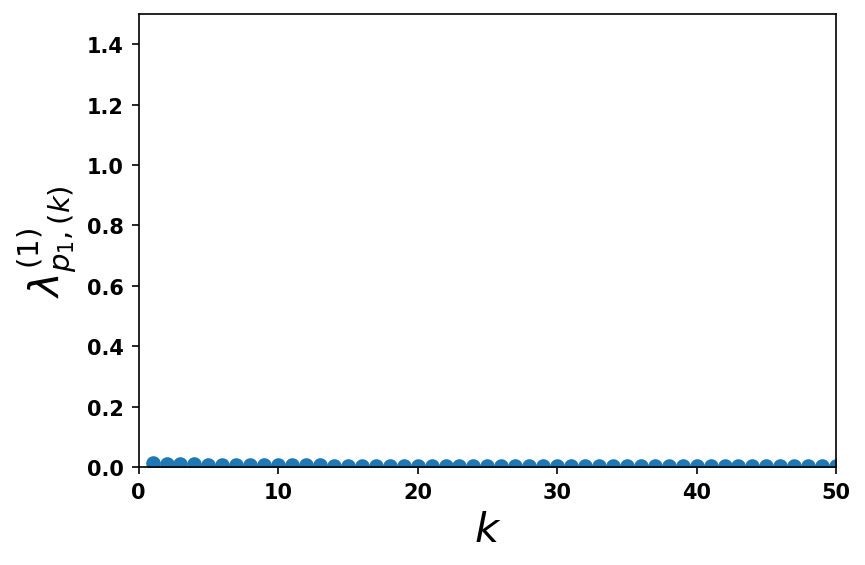}}
  \subfigure[Bernoulli]{\includegraphics[width=.33\textwidth]{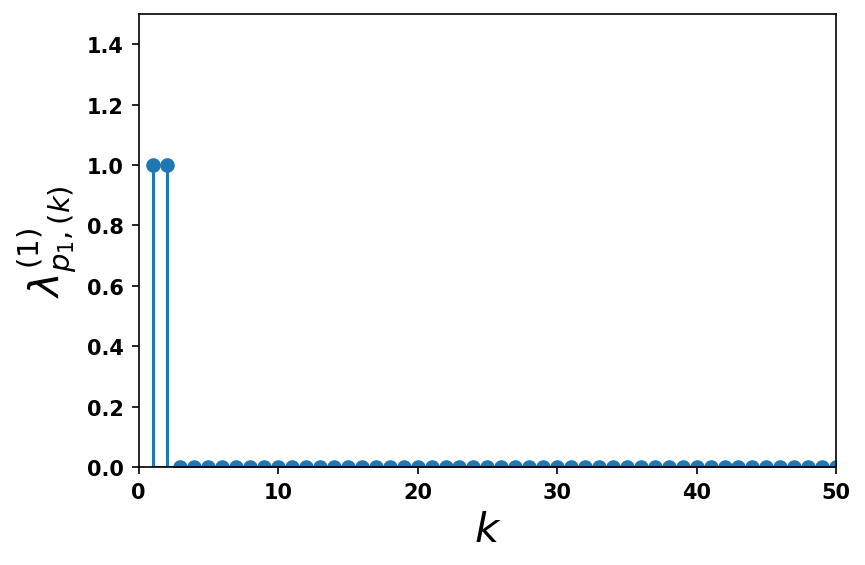}}\\
  \subfigure[Beta]{\includegraphics[width=.33\textwidth]{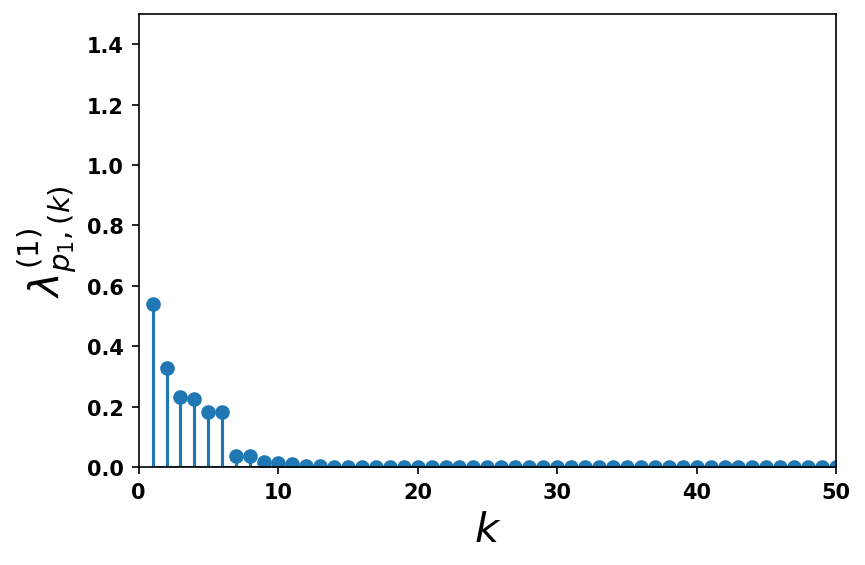}}
  \subfigure[Horseshoe]{\includegraphics[width=.33\textwidth]{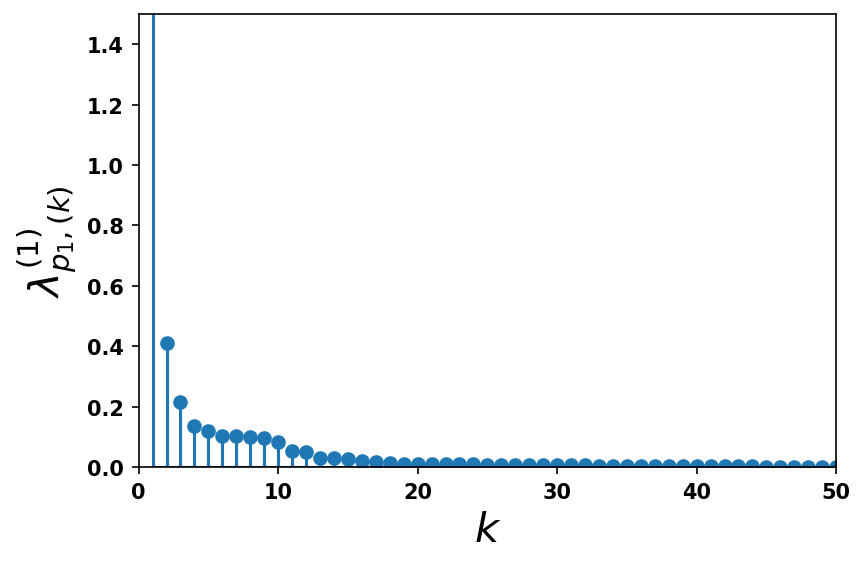}}
  \caption{The 50 largest values $\lambda_{p_1,(1)}^{(1)}\geq \lambda_{p_1,(2)}^{(1)}\geq\ldots \geq \lambda_{p_1,(50)}^{(1)}$ of $(\lambda_{p_1,j}^{(1)})_{j=1,\ldots,p_1}$ in a neural network of width $p_1=5000$ for examples (a-d). The $y$-scale is the same for all the plots. For (a), all the values are non-zero and very small, of order $(2/p_1)$. For (b), only a small number of values are non-zero, all being equal to 1. For (c-d), all the values are non-zero. For model (c), $\lambda_{p_1,(k)}^{(1)}$ decreases exponentially fast with $k$, while it decreases in $O(k^{-2})$ for model (d).}\label{fig:exampleplotlambda}
\end{figure}

Due to the shared random covariance kernel, the outputs are therefore dependent in the infinite-width limit for examples (b-d). In the case (b), only a finite (random) number of nodes are active (that is, such that $\lambda_{p_1,j}^{(1)}>0$) in the infinite-width limit; the infinite-width network is equivalent to a finite network with a $\Poi(2)$ number of hidden nodes. In the cases (c-d), an infinite number of nodes are active in the infinite-width limit. The marginal random variances take the form
$K_s^{(2)}(\bx,\bx)= (S_s^{(1)} \|\bx \|^2) / \din$
for $s\in\{b,c,d\}$,
where (see \cref{th:singleinput,app:betamodel,app:horseshoemodel,cor:scaledstablebeta_sumxy})
\begin{itemize}
\setlength\itemsep{0em}
\item $S_b^{(1)}\sim\gammadist(\frac{N}{2},\frac{1}{2})$ with $N\sim \Poi(1)$\footnote{with the convention that $S_b^{(1)}=0$ if $N=0$.};
\item $S_c^{(1)}\sim\gammadist(\frac{1}{2},\frac{1}{2})$;
\item $S_d^{(1)}\sim\IG(\frac{1}{2},\frac{1}{2})$ is an inverse gamma/L\'evy random variable.
\end{itemize}
In the case (c), $Z^{(2)}_{k}(\bx;p_1)$ converges in distribution to a normal-gamma\footnote{Various authors use the name normal-gamma to mean different things. Here we mean a Gaussian mixture distribution where the variance is governed by a gamma distribution.} random variable. For (d), it converges to a Cauchy random variable, and both dimensions of the output therefore have power-law tails, with exponent $1$. For model (a), for any $k=1,2$, $\max_{j=1,\ldots,p_1}(W_{jk}^{(2)})^2\overset{\cvp}{\to}0$, and the weights are all asymptotically small. For each case $s\in\{b,c,d\}$, $\max_{j}(W_{jk}^{(2)})^2\overset{\cvd}{\to} M_s$ where $M_s$ is a random variable whose cdf is analytically available (see \cref{sec:powerlawweights}). In particular, in the horseshoe example (d), $M_d$ follows a scaled Fr\'echet distribution. For models (b-d), let $(W_{(1),k}^{(2)})^2\geq (W_{(2),k}^{(2)})^2\geq \ldots$ be the ordered weights connected to an output $k$. In the infinite-width limit, $W_{(j),k}^{(2)}$ decreases exponentially fast with $j$ for model  (c), while it decreases in $O(j^{-2})$ for model (d) (see \cref{app:properties of small weights}). These properties are of importance when using such models for pruning the nodes/edges of large networks.
A related important property is that of the compressibility of the network. Let $\lambda_{p,(1)}\geq \lambda_{p,(2)}\geq \ldots   $ be the ordered per-node variance terms. For some compression level $\kappa\in(0,1)$, let
\begin{align*}
Z^{*(2)}_{k}(\bx;p_1,\kappa) := \sum_{j=1}^{p_1} \sqrt{\lambda_{p_1,j}^{(1)}} \ind_{\{\lambda_{p_1,j}^{(1)} > \lambda^{(1)}_{p_1,(\lfloor \kappa p_1  \rfloor)}\}}V^{(2)}_{jk} \max\left (0,\frac{1}{\sqrt{\din}}\sum_{i=1}^{\din}V^{(1)}_{ij} x_{i}\right )
\end{align*}
be the neural network obtained by pruning a $(1-\kappa)$ proportion  of nodes with the smallest $\lambda^{(1)}_{p_1,j}$ values. The models (b-d) are compressible in the sense that the difference between the pruned output $Z^{*(2)}_{k}(\bx;p_1,\kappa)$ and the unpruned output $Z^{(2)}_{k}(\bx;p_1)$ vanishes in probability in the infinite-width limit (see \cref{th:compressibility}). This is not the case for the iid Gaussian model, nor for model (a), which are not compressible.   The properties of the different models are summarised in~\cref{tab:main_examplesintro}.

\subsection{Additional Examples}
\label{app:examples}

In \cref{sec:examples}, we discussed examples of models used in the literature, and the associated parameters of the limiting infinitely divisible random variable of \cref{eq:convergencedist}.
In this part of the appendix, we provide additional examples and a general recipe behind some of these models and the Horseshoe model. As in \cref{sec:examples},
we use a different scaling in some cases so that the limit exists without being degenerate at $0$.

All the proofs of the examples in this subsection rely on \cref{th:convergenceid}. Details are given in~\cref{app:additionalderivationsExamples}.
To simplify notation, we often drop the layer index $l$ fully or partially in the rest of this subsection, writing e.g. $\lambda_{p,j}\sim\mu_p$.

\begin{table}
\begin{adjustwidth}{-5mm}{}
\scriptsize
  \begin{tabular}{@{}|@{\,}c@{\,}|@{\,}c@{\,}|@{\,}c@{\,}|@{\,}c@{\,}|@{\,}c@{\,}|@{\,}c@{\,}|@{\,}c@{\,}|@{\,}c@{\,}|@{\,}c@{\,}|@{}}
  \hline
        Name & Mixture's name & $\mu_p$ & $a$ &  L\'evy measure & Support & Finite? & Exp. $\alpha$ & Exp. $\tau$ \\
  \hline
  Determ. & Gaussian & $\delta_{c_1/p} $& $c_1$ & $0$ & -- & -- & -- & -- \\
  Bernoulli & Spike and Slab & $\left(1-\frac{c}{p}\right) \cdot \delta_0 + \frac{c}{p} \delta_1$ & $0$ & $c\delta_1$ & $\{1\}$& Yes & 0 & -- \\
        Gamma & Group lasso & $\gammadist\left (\frac{p_{l+1}+1}{2},\frac{p_l(p_{l+1}+{1})}{2{c_1}} \right)$ & $c_1$ & 0 & -- &-- & -- & --  \\
  Beta & Normal-beta & $\betadist\left(\frac{1}{p},\frac{1}{c}\right)$ &0&$x^{-1}(1-x)^{1/c-1}$&(0,1)&No& -- & --\\
  Inv.-Gamma & Multivariate t & $\IG\left(2, 2/p\right)$ & 2 & 0 & -- & --& --  & -- \\
  Inv.-Gamma & Multivariate t & $\IG\left(\alpha, \Gamma(1+\alpha)^{-1/\alpha}p^{-1/\alpha}\right)$ & 0 & $\alpha x^{-\alpha-1}$ & $(0,\infty)$ & No & $\alpha\in(0,1)$ & $\tau=\alpha$ \\
  Beta prime & Horseshoe & $\frac{2p}{\pi^2}x^{-1/2}(1+\frac{4xp^2}{\pi^2})^{-1}$ & 0 & $\frac{1}{2}x^{-3/2}$ &$(0,\infty)$ & No & 1/2 & 1/2  \\

  Resc. Beta Prime & Reg. Horseshoe & See \cref{eq:pdfrescaledbetaprime}   & 0 & $\frac{x^{-3/2}}{\pi}(1-\frac{x}{c^2})^{-1/2}$ & $(0,c^2)$& No & $1/2$ & -- \\
  Gen. BFRY &
  \begin{tabular}{@{}l@{}}
  Normal
  \\
  \ \ {}-gen. BFRY
  \end{tabular}
  & See \cref{eq:BFRYmodel}& 0 & $\frac{\eta x^{-1-\tau}}{\Gamma(1-\alpha)}\gamma(\tau-\alpha,x)$ & $(0,\infty)$ &No & $\alpha\in(0,1)$ & $\tau>\alpha$ \\
  \hline
\end{tabular}
\end{adjustwidth}
\caption{List of models and their limiting location parameter and L\'evy measure, with its properties.}
\end{table}

\subsubsection{Generalised Spike and Slab Prior}

As a generalisation of the Bernoulli case in \cref{sec:bernoulliprior}, we can consider the following spike and slab prior for the variances. Let $c>0$ and $\t c\geq0$. Consider
\[
\lambda_{p,j}\sim\left(1-\frac{c}{p}\right) \cdot \delta_{\t c/p} + \frac{c}{p} \cdot H
\]
for some probability distribution $H$ (slab) on $(0,\infty)$. We have
$$
\sum_{j}\lambda_{p,j}\cvdto \id(\t c,cH).
$$

\subsubsection{Stable Limit and the Horseshoe Model}\label{app:stablelimitexample}

We describe here a family of models whose limit is a positive stable random variable, defined in~\cref{app:positive stable}, which is a special kind of infinitely divisible random variable. The horseshoe model \cite{Carvalho2010} in \cref{sec:horseshoe}, which has been used by \citet{Louizos2017,Ghosh2018,Ghosh2019,Popkes2019} as a Bayesian prior for the weights of a deep neural network, arises as a special case.

\paragraph{Models converging to a stable distribution.} We consider that
\begin{equation}
        \lambda_{p,j}=\frac{Y_j}{(c_1 p)^{1/\alpha}}\label{eq:stablemodel}
\end{equation}
where $Y_1,Y_2,\ldots,$ are iid nonnegative random variables with cdf $F$ and survival function $\overline F=1-F$ satisfying
\begin{equation}
\overline{F}(y)\overset{y \to \infty}{\sim} y^{-\alpha}c_1 \label{eq:tailpowerlaw}
\end{equation}
for some index $\alpha\in(0,1)$ and some positive constant $c_1$.\footnote{More generally, one could replace the constant $c_1$ by a slowly varying function $L$.} We have (\citet[Theorem XVII.5.3]{Feller1971}, see also \citet[Example 5.5]{Janson2011})
\begin{align}
        \sum_{j} \lambda_{p,j}\cvdto \id(0,\rho_{\operatorname{stable}}({}\cdot{};\alpha,1))=\stable(\alpha,\Gamma(1-\alpha)^{1/\alpha}).
\end{align}
In this case, the limit L\'{e}vy measure is the positive stable L\'{e}vy measure
with tail L\'{e}vy intensity $\overline\rho_{\operatorname{stable}}(u;\alpha,1)=u^{-\alpha}$. It has power-law tails at $0$ and $\infty$, but with the same exponent $\alpha$, thus lacking
some flexibility. This limitation will be addressed by our later example which permits different
exponents at $0$ and $\infty$. There is a lot of flexibility in the choice of the distribution $F$. For example, all the following distributions have tails that satisfy \cref{eq:tailpowerlaw} for some constant $c_1>0$:
$$
Y_j\sim\pareto(\alpha,c),~~Y_j\sim\IG(\alpha,1),~~Y_j\sim\text{Fr\'echet}(\alpha),~~Y_j\sim\betaprime(c,\alpha)
$$
where $c>0$. $\pareto(\alpha,c)$ denotes the Pareto distribution with pdf $f(x)=\alpha c^\alpha x^{-\alpha-1}\ind_{\{x>c\}}$, $\text{Fr\'echet}(\alpha)$ denotes the Fr\'echet distribution with cdf $F(x)=e^{-x^{-\alpha}}$ and $\betaprime(c,\alpha)$ denotes the beta prime distribution with pdf $f(x)=x^{c-1}(1+x)^{-c-\alpha}\frac{\Gamma(c+\alpha)}{\Gamma(c)\Gamma(\alpha)}$. Combining \cref{eq:stablemodel} with $Y_j\sim\IG(\alpha,1)$ gives
$$
\lambda_{p,j}\sim\IG\left (\alpha, \frac{1}{\Gamma(1+\alpha)^{1/\alpha}p^{1/\alpha}}\right).
$$
While this model appears similar to the model in \cref{eq:inversegammamodel}, the asymptotic properties of the two models are very different.

\paragraph{Horseshoe distribution.} The horseshoe model \cite{Carvalho2010} arises as another special case. One assumes
that the random variables $Y_j$ have the same distribution as $Y=T^{2}$, where $T \sim \Cauchy(0,1)$ is a half-Cauchy random variable, with pdf given by~\cref{eq:halfcauchypdf}. The random variable $Y\sim\betaprime(1/2,1/2)$ is a beta prime random variable (with both shape parameters equal to $1/2$), with pdf
\[
        f_{Y}(y)=\frac{1}{\pi\sqrt{y}(1+y)}.%
\]
Its survival function satisfies
\[
\Pr(Y>y)\overset{y\to\infty}{\sim}\frac{2}{\pi}y^{-1/2},%
\]
and therefore $Y$ has a power-law tail at infinity with exponent $\alpha=1/2$. Let $c>0$ be some scaling parameter. Setting
$$
\lambda_{p,j} = c\times\frac{\pi^2}{4}\frac{Y_j}{p^2},
$$
we obtain
\begin{align}
        \sum_{j} \lambda_{p,j}\cvdto \id(0,\rho_{\operatorname{stable}}({}\cdot{};1/2,c))=\stable(1/2,c\pi)=\IG(1/2,c\pi/4).
\end{align}
The limit is therefore a stable distribution, with exponent $\frac{1}{2}$, which is also inverse gamma with shape parameter $1/2$ in this case. The tail L\'evy intensity $\overline\rho_{\operatorname{stable}}(x;1/2,c)$ has power-law tails at $0$ and $\infty$, with exponent $1/2$.

\subsubsection{Regularised Horseshoe and Stable Beta Process}

\citet{Ghosh2018,Ghosh2019} considered Bayesian learning of neural networks using regularised horseshoe priors. In this case, we have
\[
\lambda_{p,j} = \frac{c^2  \frac{T^2_j}{p^2}}{c^2+\frac{T^2_j}{p^2}}
\]
where $T_j\sim \Cauchy(0,1)$ and $c>0$. Note that $\lambda_{p,j}\in(0,c^2)$ is bounded, with pdf
\begin{equation}
f_p(x)=\frac{p}{\pi} \cdot x^{-\frac{1}{2}}(1-x/c^2)^{-\frac{3}{2}}\left(1+\frac{p^2 x}{(1-x/c^2)}\right)^{-1}\ind_{\{x<c^2\}}.\label{eq:pdfrescaledbetaprime}
\end{equation}
We have, for any $x>0$,
$$
\lim_{p\to\infty} pf_p(x)= \frac{1}{\pi} \cdot x^{-\frac{3}{2}} \left(1-x/c^2\right)^{-\frac{1}{2}} \ind_{\{x<c^2\}}.
$$
An application of~\cref{th:convergenceid} gives
$$
\sum_j \lambda_{p,j}\cvdto\id(0, \rho)
$$
where
$$
\rho(du)= \frac{1}{\pi} \cdot u^{-\frac{3}{2}} \left(1-u/c^2\right)^{-\frac{1}{2}} \ind_{\{u<c^2\}}du
$$
is a (scaled) stable beta L\'evy measure~\cite{Teh2009}. The L\'evy measure has bounded support, and the associated tail L\'evy intensity increases polynomially at 0, with exponent $\alpha=1/2$,
$$
\overline\rho(x)\overset{x\to 0}{\sim} \frac{2}{\pi} x^{-1/2}.
$$
The limiting random variable $\id(0, \rho)$ has support $(0,\infty)$.

\subsubsection{General Models with Arbitrary Limiting L\'evy Measure}
It may be of interest to set the limiting constant $a$ and the L\'evy measure $\rho$, so that they satisfy a number of properties, and then pick a distribution $\mu_p$ that makes $\sum_{j} \lambda_{p,j} \cvdto\id(a,\rho)$.

First note that if $\sum_{j} \lambda_{p,j}\cvdto\id(0,\rho)$, then $\sum_{j} (\lambda_{p,j}+\frac{a}{p}) \cvdto\id(a,\rho)$, so without loss of generality, we restrict the discussion to models with $a=0$.

If $\rho$ is finite, then $H(dx)=\rho(dx)/\overline\rho(0)$ is a probability distribution, and one can simply set
$$
\mu_{p}=\frac{\overline\rho(0)}{p} H + \left(1-\frac{\overline\rho(0)}{p}\right) \delta_0.
$$
If $\rho$ is infinite, on the other hand, one can resort to the construction of~\citet{Perman1992} (see also \cite{Lee2019}), with
\[
\mu_{p}(du)=\frac{(1-e^{-u\psi^{-1}({p})})}{p}\rho(du)
\]
where $\psi^{-1}$ is the generalised inverse of $\psi(t)=\int_{0}^{\infty}(1-e^{-ut})\rho(du)$, the Laplace
exponent of $\rho$.

\subsection{Derivations of the Beta and Horseshoe Examples in \cref{sec:intro,sec:examples}}
\label{app:additionalderivationsExamples}

\subsubsection{Beta Model (c) in \cref{sec:intro,sec:betaexample}}
\label{app:betamodel}
Consider
$$
\lambda_{p,j}\sim \betadist\left(\frac{\eta}{p},b\right).
$$
The pdf of $\lambda_{p,j}$ is
\begin{align*}
f_{p}(x)&=\frac{\Gamma(\eta/{p} +b)}{\Gamma(\eta/{p})\Gamma(b)}x^{\eta/{p}-1}(1-x)^{b-1} \ind_{\{x\in(0,1)\}}\\
        &=\frac{\eta}{{p}}\frac{\Gamma(\eta/{p} +b)}{\Gamma(\eta/{p}+1)\Gamma(b)}x^{\eta/{p}-1}(1-x)^{b-1} \ind_{\{x\in(0,1)\}}.
\end{align*}
We have
\begin{align*}
        {p}f_{p}(x)
        & {} =\eta\frac{\Gamma(\eta/{p} +b)}{\Gamma(\eta/{p}+1)\Gamma(b)}x^{\eta/{p}-1}(1-x)^{b-1} \ind_{\{x\in(0,1)\}}
        \\
        & {} \to \eta x^{-1}(1-x)^{b-1}\ind_{\{x\in(0,1)\}}=\varrho(x).
\end{align*}
This is the density of a stable beta measure $\rho_{\operatorname{sb}}$  with parameters $(\frac{\eta}{b}, 0, b)$. For $x\in(0,1)$,
$$
\frac{{p}f_{p}(x)}{\varrho(x)}=\frac{\Gamma(\eta/{p} +b)}{\Gamma(\eta/{p}+1)\Gamma(b)}x^{\eta/{p}}\ind_{\{x\in(0,1)\}}\leq \frac{\Gamma(\eta/{p} +b)}{\Gamma(\eta/{p}+1)\Gamma(b)}\leq 2
$$
for ${p}$ large enough. Finally, $\mu_{p}$ and $\rho$ have the same support. Thus, \cref{th:convergenceid} implies that $\sum_j \lambda_{p,j} \cvdto\id(0,\rho)$ with $\rho(dx)=\eta x^{-1}(1-x)^{b-1}\ind_{\{x\in(0,1)\}}dx$. The tail L\'evy intensity satisfies
$$
\overline\rho(x)\overset{x\to 0}{\sim} \eta\log(1/x).
$$
 Let $(\widetilde \lambda_{(1)}\geq \widetilde \lambda_{(2)}\geq \ldots)$ denote the ordered weights of a Poisson point process on $(0,\infty)$ with mean measure $\rho$. In the case $b=1$, they admit the simple inverse-L\'evy/stick-breaking construction \cite{Teh2009}
$$
\widetilde \lambda_{(j)}=\prod_{k=1}^j \beta_k,~~\text{where }\beta_1,\beta_2\,\ldots\text{ are iid } \betadist(\eta,1).
$$

\subsubsection{Horseshoe Example (d) in \cref{sec:intro,app:illustrative_example,app:stablelimitexample}}
\label{app:horseshoemodel}

Let $\lambda_{p,j} =c\frac{\pi^2}{4}\frac{U^2_j}{p^2}$ where $U_j\sim \Cauchy(0,1)$, with $c=4$. From \cref{app:stablelimitexample},
$$
\sum_{j} \lambda_{p,j}\cvdto \id(0,\rho_{\operatorname{stable}}({}\cdot{};1/2,c))=\stable(1/2,c\pi)=\IG(1/2,c\pi/4).
$$
 Let $(\widetilde \lambda_{(1)}\geq \widetilde \lambda_{(2)}\geq \ldots)$ denote the ordered weights of a Poisson point process on $(0,\infty)$ with mean measure $\rho_{\operatorname{stable}}({} \cdot {};1/2,c)$. They admit the inverse-L\'evy representation
 $$
 \widetilde \lambda_{(j)}=\frac{c}{\left (\sum_{k=1}^j E_k \right)^{2}}
 $$
 where $E_1,E_2,\ldots$, are iid exponential random variables with unit rate.

From \cref{cor:stable_sumxy},
$$
\sum_{j} \lambda_{p,j}\max(0,X_j)^2\cvdto \id(0,\rho_{\operatorname{stable}}({}\cdot{};1/2,c/(2\pi)))=\stable(1/2,c/2)=\IG(1/2,c/8)
$$
where $X_1,X_2\ldots,$ are iid $N(0,1)$ random variables.  Using \cref{prop: extremes}, for any $x>0$,
$$
\Pr(\max_j \lambda_{p,j}\leq x)\to e^{- (x/c)^{-1/2}}
$$
which is the cdf of a Fr\'echet random variable with shape parameter $\alpha=1/2$ and scale parameter $c$. Similarly,
since $\sigma_v = 1$,
$$
\Pr(\max_j W^2_{jk}\leq x)\to e^{- (x/(2c/\pi))^{-1/2}}
$$
which is the cdf of a Fr\'echet random variable with shape parameter $\alpha=1/2$ and scale parameter $2c/\pi$.
Equivalently,
$$
\Pr(\max_j |W_{jk}|\leq y)\to e^{- \left( y/\sqrt{2c/\pi}\right)^{-1}}
$$
which is the cdf of a Fr\'echet random variable with shape parameter $\alpha=1$ and scale parameter $\sqrt{2c/\pi}$.

\newpage

\section{Additional Experiments}
\label{app:Bayesian-experiment}
	
\subsection{Stability of MoGP for Deep Networks}

A common practical problem that may emerge with deep models is that of the vanishing/exploding norm of the output/gradient. In this subsection, we empirically investigate these phenomena for the ReLU activation in the MoGP context. We consider the FFNN model with the generalised BFRY variance distribution described in \cref{sec:experiments_simu}, with a fixed width of $p=1000$ and univariate input and output. The parameters of the generalised BFRY are fixed as in \cref{sec:experiments_simu}; here, only $\sigma_b$ and $\sigma_v$ vary. For each depth, we simulate $500$ random initialisations of the model. We investigate the stability of forward passes by looking at the empirical distribution of the norm of the output $Z^{(L+1)}(\bx)$ as the depth $L+1$ increases. For backward passes, we look at the empirical distribution of the norm of the gradient of the loss with respect to the input weights $W^{(1)}$.

\Cref{th:singleinput} states that the variance of the output follows the stochastic recurrence equation
	$$ \Sigma^{(l)} = \sigma_b^2 + \sigma_v^2 S^{(l-1)} \Sigma^{(l-1)},$$
where $S^{(l-1)}:=\sum_{j=1}^{p}\lambda_{j}^{(l-1)}\phi(\varepsilon
_{j}^{(l-1)})^{2}$ and $\varepsilon
_{j}^{(l-1)}$ are i.i.d standard normals. It is well known that if $\E \big[ \log \sigma_v^2 S^{(l)} \big] < 0 $, the random process $\Sigma^{(l)}$ is strictly stationary ergodic \citep{buraczewski2013large}. If we further assume that $\sigma_b > 0$, the forward pass is stable (non-vanishing and non-exploding), as illustrated in \cref{fig:depth_stability_bias}. However, one can notice that, similarly to standard initialisations, a non-zero bias leads to an unstable backward pass. With $\sigma_b=0$, taking parameters such that $\E \big[ \log \sigma_v^2 S^{(l)} \big] < 0 $ or $\E \big[ \log \sigma_v^2 S^{(l)} \big] > 0 $ is not practical for deep networks as such parameters lead respectively to vanishing or exploding outputs and gradients (see \cref{fig:depth_stability_bad_mu}). As illustrated in \cref{fig:depth_stability_right_mu}, taking parameters such that $\E \big[ \log \sigma_v^2 S^{(l)} \big] = 0 $ leads to relatively stable forward and backward passes even at depth $20$. However, we do observe an increased variance of the distributions. In rare runs, this may cause difficulties in training the models.

\begin{figure}[h!]
		\centering
		\includegraphics[width=0.45\linewidth]{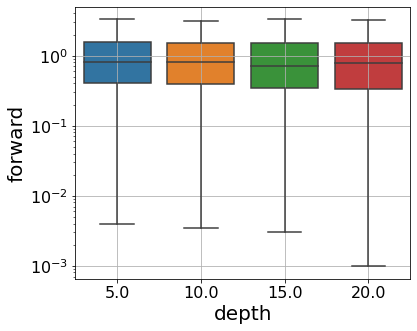}
		\includegraphics[width=0.45\linewidth]{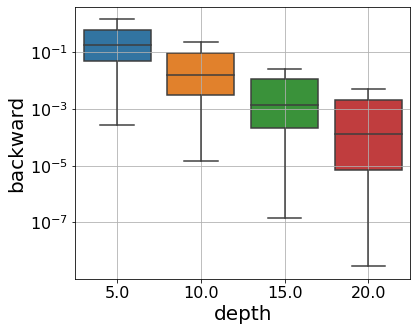}
		\caption{Stability of the forward and backward passes as the depth increases. Initialisation with non-zero bias $\sigma_b = 1$ and $\E \big[ \log \sigma_v^2 S^{(l)} \big] = -1$.}
		\label{fig:depth_stability_bias}
\end{figure}

\begin{figure}[h!]
		\centering
		\subfigure[ {$\E \big[ \log \sigma_v^2 S^{(l)} \big] = 1$} ]{%
		\includegraphics[width=0.45\linewidth]{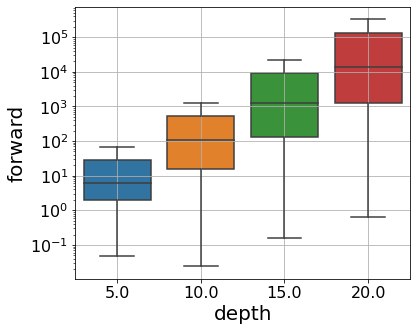}%
		\includegraphics[width=0.45\linewidth]{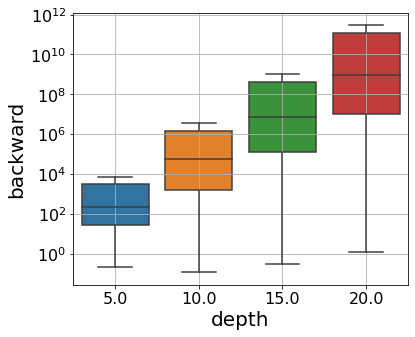}%
		}
		\subfigure[ {$\E \big[ \log \sigma_v^2 S^{(l)} \big] = -1$} ]{%
		\includegraphics[width=0.45\linewidth]{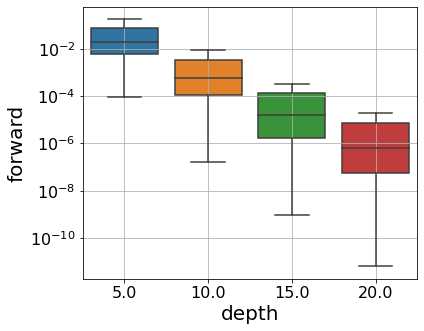}%
		\includegraphics[width=0.45\linewidth]{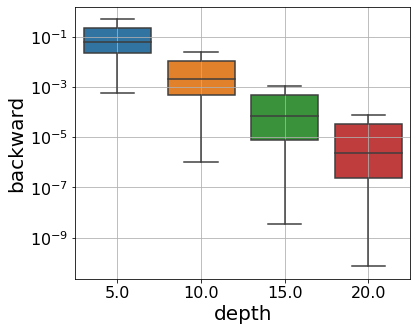}
		}
		\caption{Stability of the forward and backward passes as the depth increases. Initialisation with zero bias $\sigma_b = 0$.}
		\label{fig:depth_stability_bad_mu}
\end{figure}

\begin{figure}[h!]
		\centering
		\includegraphics[width=0.45\linewidth]{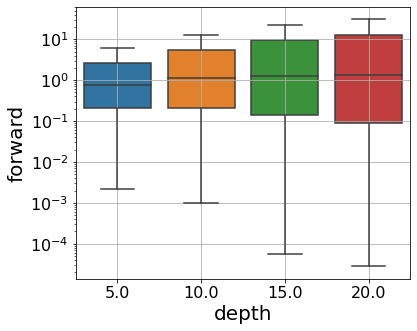}
		\includegraphics[width=0.45\linewidth]{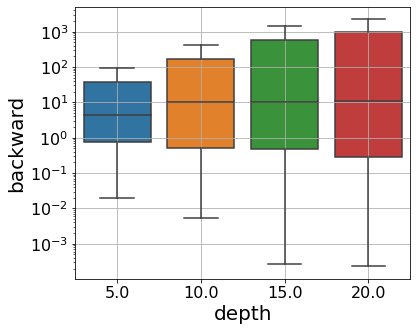}
		\caption{Stability of the forward and backward passes as the depth increases. Initialisation with $\sigma_b = 0$ and $\E \big[ \log \sigma_v^2 S^{(l)} \big] = 0$.}
		\label{fig:depth_stability_right_mu}
\end{figure}

\subsection{Using MoGP as a Regularisation for Convolutional Layers}
	
	In this subsection, we empirically illustrate how the MoGP framework can be used with convolutional neural networks (CNN) to promote compressibility. Here we consider the more challenging dataset Cifar10. The CNN model consists of two convolutional layers (\texttt{Conv1} and \texttt{Conv2}) and a final fully-connected layer (\texttt{fc}). Each convolutional layer is followed by a max pooling layer, with kernel size $2\times2$. The weights of each convolutional layer $l \in \{ 1,2\}$ is a tensor of dimensions
	$$n_f^{(l)} \times C_{in}^{(l)} \times K_x^{(l)} \times K_y^{(l)},$$
where $n_f^{(l)}$ is the number of filters of the layer, $C_{in}^{(l)}$ is the number of input channels (number of channels of the input ``image") and $K_x^{(l)} \times K_y^{(l)}$ is the kernel size. For both \texttt{Conv1} and \texttt{Conv2}, we take $K_x = K_y = 3$. We note that $n_f^{(1)}$, the number of filters of \texttt{Conv1}, is equal to $ C_{in}^{(2)}$ the number of input channels of \texttt{Conv2}. In this CNN setting, the input channels of a convolutional layer play the role of the input nodes in the FFNN setting. We reproduce the experiment of \cref{sec:MoGP regularisation} and assign to each input channel $j$ of \texttt{Conv2}, a scale $\lambda_j^\texttt{Conv2}$ and associate penalisation. Notice that if we prune a fraction $1-\kappa$ of the input channels of \texttt{Conv2}, then only a fraction $\kappa$ of the output filters of \texttt{Conv1} and the input channels of \texttt{Conv2} are active; the total number of parameters of \texttt{Conv1} and \texttt{Conv2} after compression is hence reduced to a fraction $\kappa$ of the original number. We also assign a scale $\lambda_j^\texttt{fc}$ to each input node $j$ of the final fully-connected layer. This enables to compress the full model and not only the convolutional layers. Both convolutional layers have $512$ filters. The results are reported in \cref{tab:cifar10_CNN} and  \cref{fig:CNN_truncation}. We can see that the conclusion of the FFNN setting still hold in this CNN setting: using the MoGP framework as a regularisation during training leads to more compressible models. We believe that similar results would hold for more complex architectures such as ResNets, but leave this open question for future work.
	
	\begin{table}[h]
\begin{center}
\begin{tabular}{lrrr}
\toprule
\pbox{4cm}{Truncation\\ (i.e., $1-\kappa$)} &  Deterministic &  Horseshoe &  Gen. BFRY \\
\midrule
0$\%$         &   70.58 $(\pm 0.11)$&      69.66 $(\pm 0.12)$ &        69.58 $(\pm 0.24)$\\
10$\%$       &   68.26 $(\pm 0.33)$&      69.52 $(\pm 0.20)$&        69.52 $(\pm 0.34)$ \\
20$\%$       &   66.13 $(\pm 0.68)$&      68.81 $(\pm 0.28)$&        68.36 $(\pm 0.46)$\\
50$\%$       &   52.15 $(\pm 2.9)$&      60.95 $(\pm 1.8)$&        59.35 $(\pm 0.57)$\\
80$\%$       &   29.10 $(\pm 1.0)$&      44.34 $(\pm 5.3)$&        40.52 $(\pm 2.7)$\\
90$\%$       &   19.25 $(\pm 1.6)$&      32.60 $(\pm 3.46)$&        29.23 $(\pm 2.7)$\\
\bottomrule
\end{tabular}
\end{center}
\caption{Classification accuracy on Cifar10 
under various truncation ratio for a CNN model.}\label{tab:cifar10_CNN}
\end{table}

\begin{figure}
		\centering
		\includegraphics[width=0.55\linewidth]{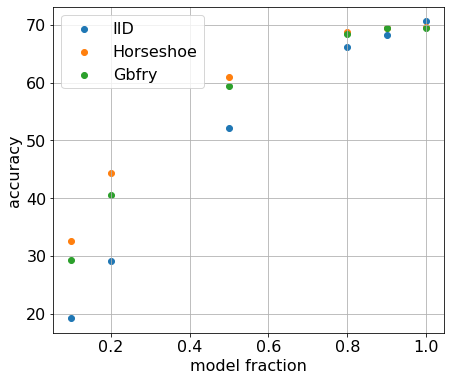}
		\caption{Classification accuracy on Cifar10 with a CNN model.}
		\label{fig:CNN_truncation}
\end{figure}

\subsection{Further Results from MNIST and Fashion-MNIST}

\Cref{fig:bayes_neurons} shows the top-5 activating images for each of the top-8 activating neurons in the deterministic, the horseshoe and the generalised BFRY cases in our full Bayesian experiments.
	
	\begin{figure}
		\centering
		\includegraphics[width=0.3\linewidth]{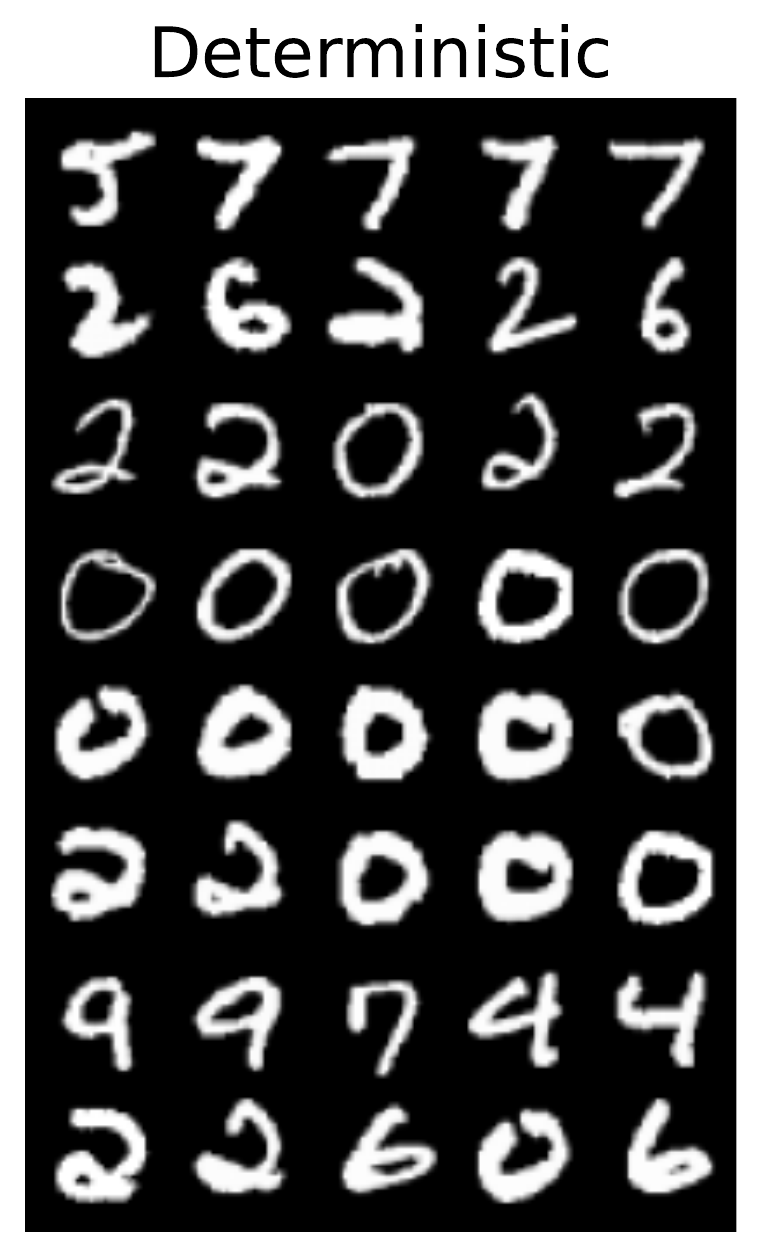}
		\includegraphics[width=0.3\linewidth]{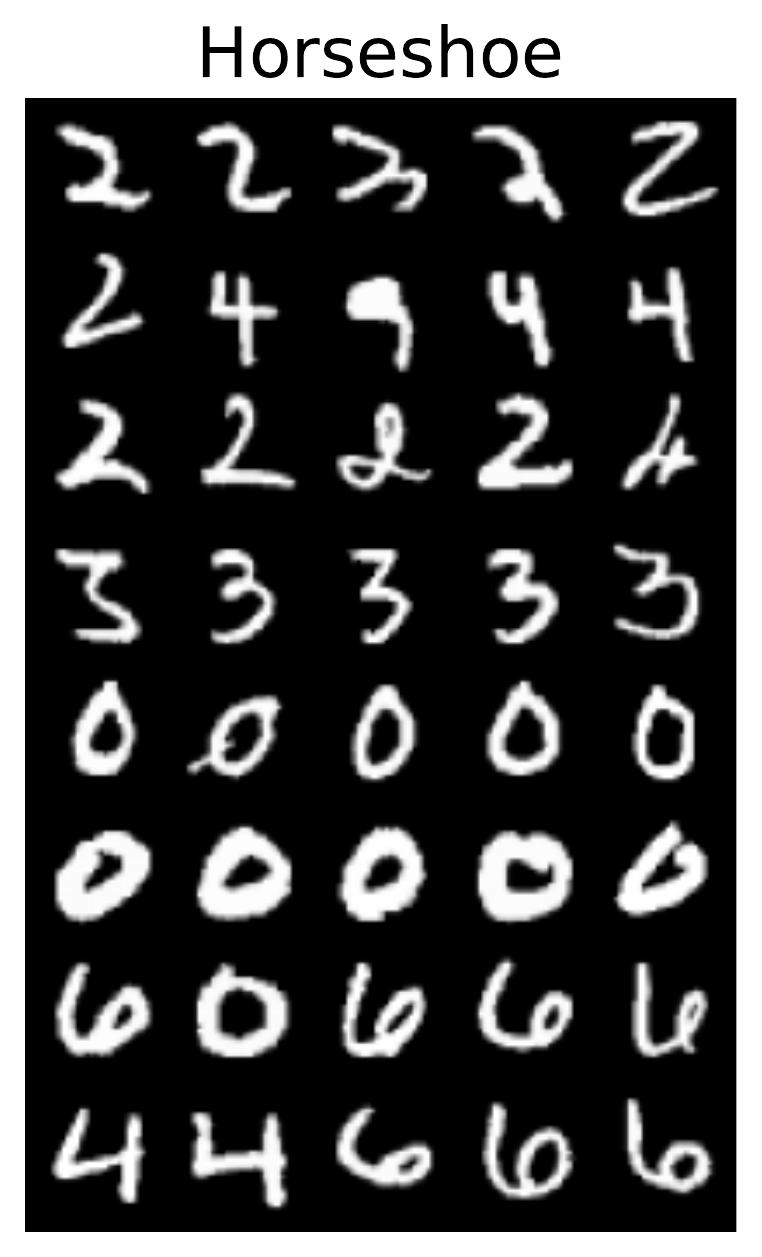}
		\includegraphics[width=0.3\linewidth]{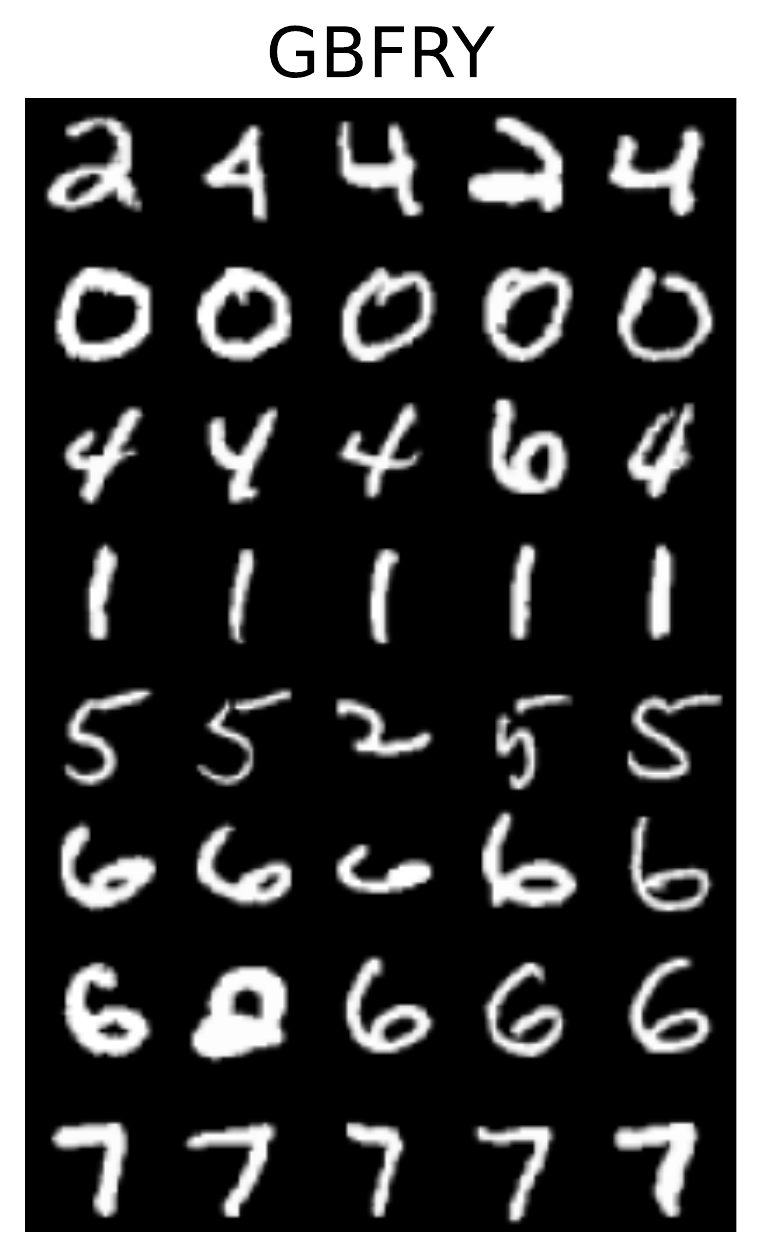}
		\includegraphics[width=0.3\linewidth]{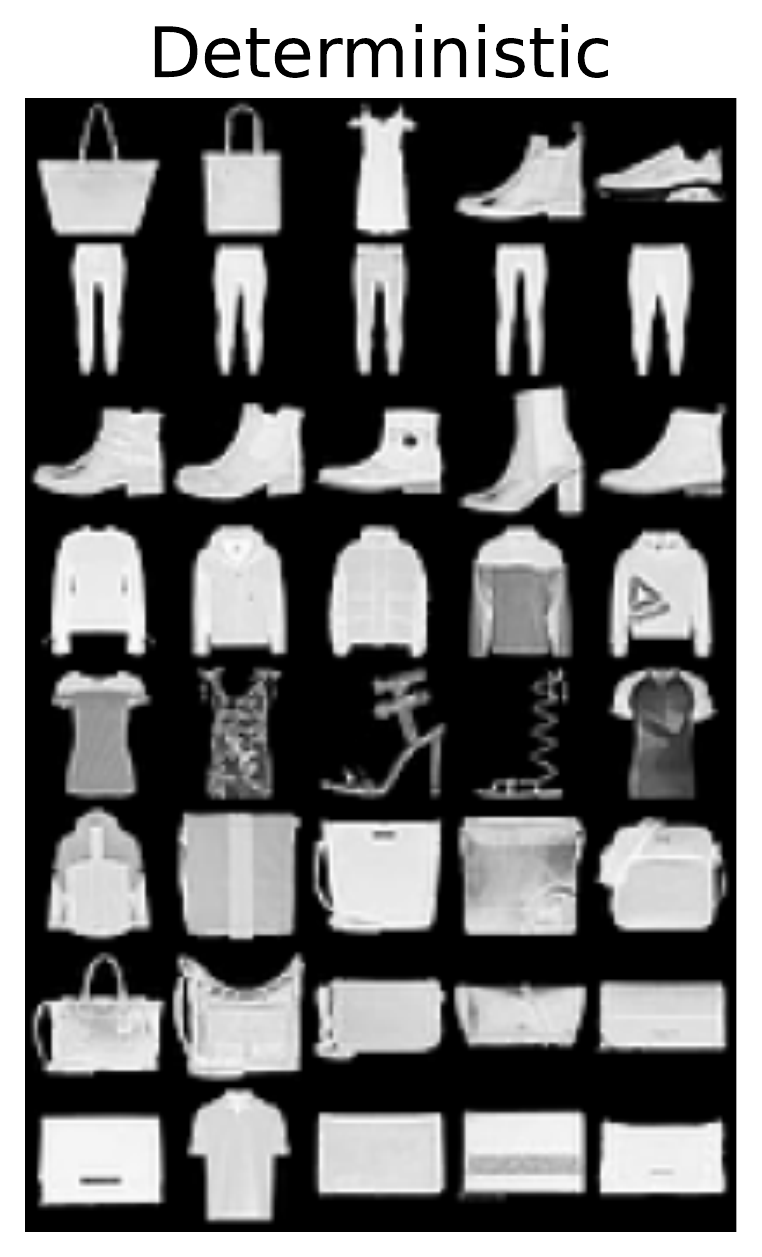}
		\includegraphics[width=0.3\linewidth]{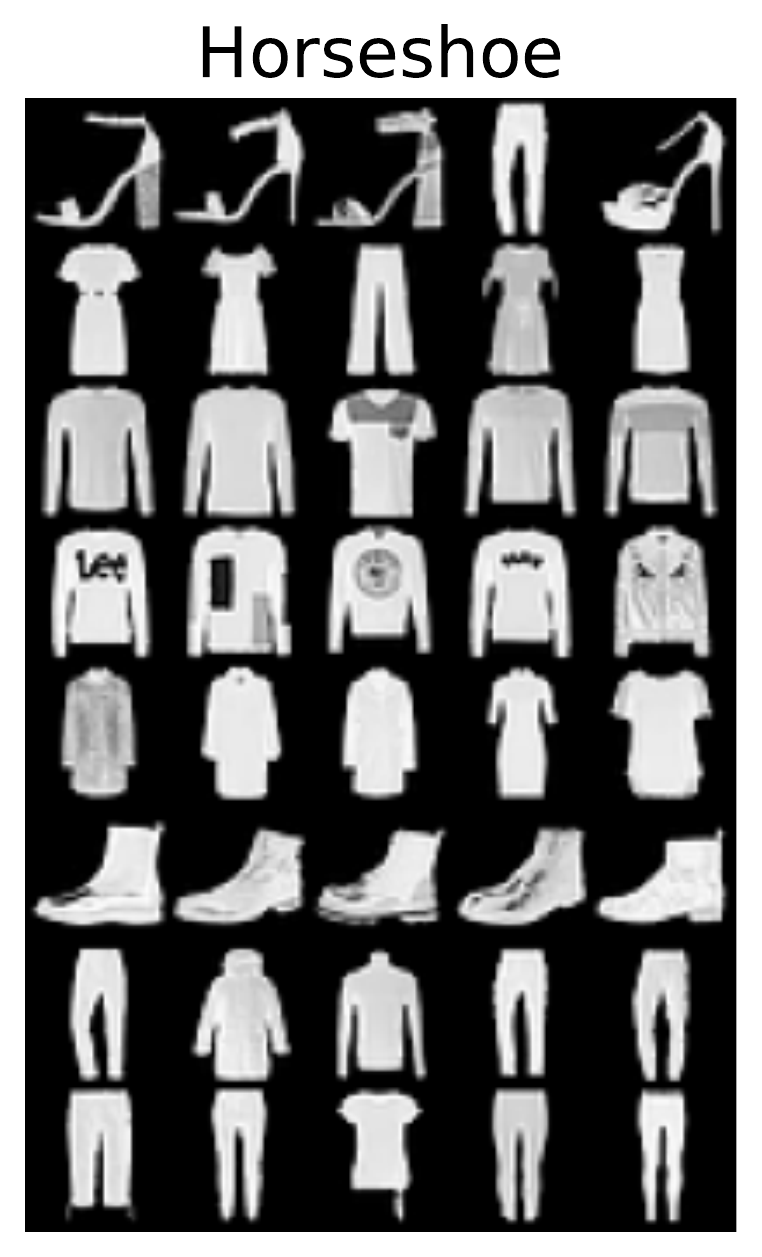}
		\includegraphics[width=0.3\linewidth]{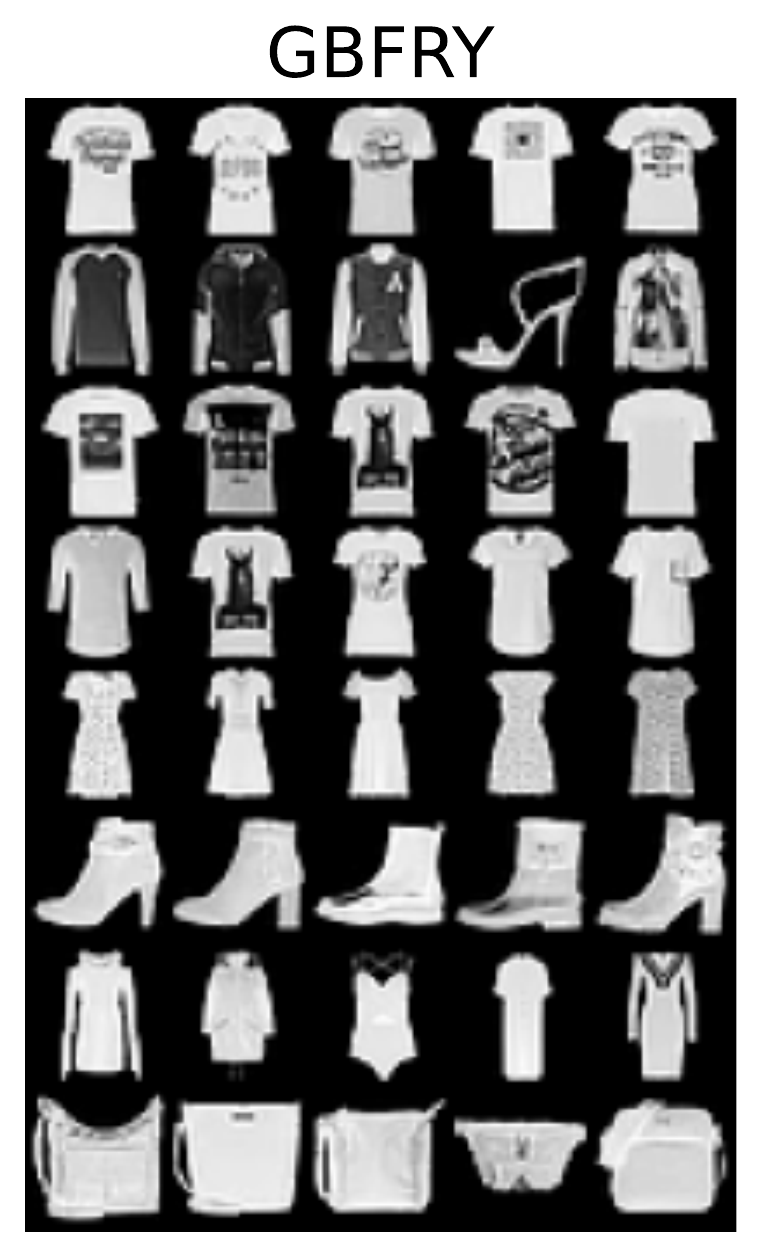}
		\caption{Top activating neurons and corresponding top activationg images. MNIST (left column) and Fashion MNIST (right column). Rows are neurons and columns are corresponding top activating images. GBFRY represents generalised BFRY.}
		\label{fig:bayes_neurons}
	\end{figure}

\vskip 0.2in

\bibliography{sparsity}

\begin{thebibliography}{87}
\providecommand{\natexlab}[1]{#1}
\providecommand{\url}[1]{\texttt{#1}}
\expandafter\ifx\csname urlstyle\endcsname\relax
  \providecommand{\doi}[1]{doi: #1}\else
  \providecommand{\doi}{doi: \begingroup \urlstyle{rm}\Url}\fi

\bibitem[Adamczewski and Park(2021)]{Adamczewski2021}
K.~Adamczewski and M.~Park.
\newblock {D}irichlet pruning for convolutional neural networks.
\newblock In \emph{Proceedings of the 24th International Conference on
  Artificial Intelligence and Statistics (AISTATS'21)}, pages 3637--3645, 2021.

\bibitem[Aitchison(2020)]{Aitchison2020}
L.~Aitchison.
\newblock Why bigger is not always better: on finite and infinite neural
  networks.
\newblock In \emph{Proceedings of the 37th International Conference on Machine
  Learning (ICML'20)}, pages 156--164, 2020.

\bibitem[Arora et~al.(2018)Arora, Ge, Neyshabur, and Zhang]{arora2018stronger}
S.~Arora, R.~Ge, B.~Neyshabur, and Y.~Zhang.
\newblock Stronger generalization bounds for deep nets via a compression
  approach.
\newblock In \emph{Proceedings of the 35th International Conference on Machine
  Learning (ICML'18)}, pages 254--263, 2018.

\bibitem[Ayed et~al.(2019)Ayed, Lee, and Caron]{Ayed2019}
F.~Ayed, J.~Lee, and F.~Caron.
\newblock Beyond the {C}hinese restaurant and {P}itman-{Y}or processes:
  Statistical models with double power-law behavior.
\newblock In \emph{Proceedings of the 36th International Conference on Machine
  Learning (ICML'19)}, pages 395--404, 2019.

\bibitem[Ayed et~al.(2020)Ayed, Lee, and Caron]{Ayed2020}
F.~Ayed, J.~Lee, and F.~Caron.
\newblock The normal-generalised gamma-{P}areto process: A novel pure-jump
  {L}\'evy process with flexible tail and jump-activity properties.
\newblock \emph{arXiv preprint arXiv:2006.10968}, 2020.

\bibitem[Barsbey et~al.(2021)Barsbey, Sefidgaran, Erdogdu, Richard, and
  Simsekli]{barsbey2021heavy}
M.~Barsbey, M.~Sefidgaran, M.~A. Erdogdu, G.~Richard, and U.~Simsekli.
\newblock Heavy tails in {SGD} and compressibility of overparametrized neural
  networks.
\newblock In \emph{Proceedings of the 35th Conference on Neural Information
  Processing Systems (NeurIPS'21)}, pages 29364--29378, 2021.

\bibitem[Bertoin et~al.(2006)Bertoin, Fujita, Roynette, and Yor]{Bertoin2006}
J.~Bertoin, T.~Fujita, B.~Roynette, and M.~Yor.
\newblock On a particular class of self-decomposable random variables: the
  durations of {B}essel excursions straddling independent exponential times.
\newblock 2006.

\bibitem[Bingham et~al.(1989)Bingham, Goldie, and Teugels]{Bingham1989}
N.~H. Bingham, C.~M. Goldie, and J.~L. Teugels.
\newblock \emph{Regular Variation}.
\newblock Cambridge university press, 1989.

\bibitem[Blier et~al.(2019)Blier, Wolinski, and Ollivier]{Blier2018}
L.~Blier, P.~Wolinski, and Y.~Ollivier.
\newblock Learning with random learning rates.
\newblock In \emph{Proceedings of the European Conference on Machine Learning
  and Principles and Practice of Knowledge Discovery in Databases (ECML
  PKDD'19)}, pages 449--464, 2019.

\bibitem[Bracale et~al.(2021)Bracale, Favaro, Fortini, and
  Peluchetti]{Bracale2021}
D.~Bracale, S.~Favaro, S.~Fortini, and S.~Peluchetti.
\newblock Large-width functional asymptotics for deep {G}aussian neural
  networks.
\newblock \emph{arXiv preprint arXiv:2102.10307}, 2021.

\bibitem[Brix(1999)]{Brix1999}
A.~Brix.
\newblock Generalized gamma measures and shot-noise {C}ox processes.
\newblock \emph{Advances in Applied Probability}, 31\penalty0 (4):\penalty0
  929--953, 1999.

\bibitem[Bui et~al.(2016)Bui, Hern{\'a}ndez-Lobato, Hernandez-Lobato, Li, and
  Turner]{Bui2016}
T.~Bui, D.~Hern{\'a}ndez-Lobato, J.~Hernandez-Lobato, Y.~Li, and R.~Turner.
\newblock Deep {G}aussian processes for regression using approximate
  expectation propagation.
\newblock In \emph{Proceedings of the 33rd International Conference on Machine
  Learning (ICML'16)}, pages 1472--1481, 2016.

\bibitem[Buraczewski et~al.(2013)Buraczewski, Damek, Mikosch, and
  Zienkiewicz]{buraczewski2013large}
D.~Buraczewski, E.~Damek, T.~Mikosch, and J.~Zienkiewicz.
\newblock Large deviations for solutions to stochastic recurrence equations
  under kesten’s condition.
\newblock 2013.

\bibitem[Caron and Doucet(2008)]{Caron2008}
F.~Caron and A.~Doucet.
\newblock Sparse {B}ayesian nonparametric regression.
\newblock In \emph{Proceedings of the 25th International Conference on Machine
  learning (ICML'08)}, pages 88--95, 2008.

\bibitem[Carvalho et~al.(2010)Carvalho, Polson, and Scott]{Carvalho2010}
C.~M. Carvalho, N.~G. Polson, and J.~G. Scott.
\newblock The horseshoe estimator for sparse signals.
\newblock \emph{Biometrika}, 97\penalty0 (2):\penalty0 465--480, 2010.

\bibitem[Casella et~al.(2010)Casella, Ghosh, Gill, and Kyung]{Casella2010}
G.~Casella, M.~Ghosh, J.~Gill, and M.~Kyung.
\newblock Penalized regression, standard errors, and {B}ayesian lassos.
\newblock \emph{Bayesian analysis}, 5\penalty0 (2):\penalty0 369--411, 2010.

\bibitem[Chaudhari and Soatto(2018)]{chaudhari2018stochastic}
P.~Chaudhari and S.~Soatto.
\newblock Stochastic gradient descent performs variational inference, converges
  to limit cycles for deep networks.
\newblock In \emph{Proceedings of the 2018 Information Theory and Applications
  Workshop (ITA'18)}, pages 1--10. IEEE, 2018.

\bibitem[Chen et~al.(2014)Chen, Fox, and Guestrin]{Chen2014}
T.~Chen, E.~B. Fox, and C.~Guestrin.
\newblock Stochastic gradient {Hamiltonian Monte-Carlo}.
\newblock In \emph{Proceedings of the 31st International Conference on Machine
  Learning (ICML'14)}, pages 1683--1691, 2014.

\bibitem[Cho and Saul(2009)]{Cho2009}
Y.~Cho and L.~Saul.
\newblock Kernel methods for deep learning.
\newblock In \emph{Proceedings of the 23rd Conference on Neural Information
  Processing Systems (NeurIPS'09)}, pages 342--350, 2009.

\bibitem[Damianou and Lawrence(2013)]{Damianou2013}
A.~Damianou and N.~D. Lawrence.
\newblock Deep {G}aussian processes.
\newblock In \emph{Proceedings of the 16th International Conference on
  Artificial Intelligence and Statistics (AISTATS'13)}, pages 207--215, 2013.

\bibitem[Davydov et~al.(2008)Davydov, Molchanov, and
  Zuyev]{davydov2008strictly}
Y.~Davydov, I.~Molchanov, and S.~Zuyev.
\newblock Strictly stable distributions on convex cones.
\newblock \emph{Electronic Journal of Probability}, 13:\penalty0 259--321,
  2008.

\bibitem[de~Jong(2018)]{Jong2018}
S.~de~Jong.
\newblock \emph{Compressing Neural Networks}.
\newblock PhD thesis, University of Cambridge, 2018.

\bibitem[Der and Lee(2006)]{Der2006}
R.~Der and D.~Lee.
\newblock Beyond {G}aussian processes: On the distributions of infinite
  networks.
\newblock In \emph{Proceedings of the 20th Conference on Neural Information
  Processing Systems (NeurIPS'06)}, pages 275--282, 2006.

\bibitem[Dunlop et~al.(2018)Dunlop, Girolami, Stuart, and
  Teckentrup]{Dunlop2018}
M.~Dunlop, M.~A. Girolami, A.~M. Stuart, and A.~L. Teckentrup.
\newblock How deep are deep {G}aussian processes?
\newblock \emph{Journal of Machine Learning Research}, 19\penalty0
  (54):\penalty0 1--46, 2018.

\bibitem[Durrett(2019)]{durrett2019probability}
R.~Durrett.
\newblock \emph{Probability: Theory and Examples}, volume~49.
\newblock Cambridge university press, 2019.

\bibitem[Favaro et~al.(2020)Favaro, Fortini, and Peluchetti]{Favaro2020}
S.~Favaro, S.~Fortini, and S.~Peluchetti.
\newblock Stable behaviour of infinitely wide deep neural networks.
\newblock In \emph{Proceedings of the 23rd International Conference on
  Artificial Intelligence and Statistics (AISTATS'20)}, pages 1137--1146, 2020.

\bibitem[Feller(1971)]{Feller1971}
W.~Feller.
\newblock \emph{An Introduction to Probability Theory and Its Application Vol
  II}.
\newblock John Wiley and Sons, 1971.

\bibitem[Fortuin(2021)]{Fortuin2021}
V.~Fortuin.
\newblock Priors in {B}ayesian deep learning: A review.
\newblock \emph{arXiv preprint arXiv:2105.06868}, 2021.

\bibitem[Fortuin et~al.(2021)Fortuin, Garriga-Alonso, Wenzel, R{\"a}tsch,
  Turner, van~der Wilk, and Aitchison]{Fortuin2021a}
V.~Fortuin, A.~Garriga-Alonso, F.~Wenzel, G.~R{\"a}tsch, R.~Turner, M.~van~der
  Wilk, and L.~Aitchison.
\newblock {B}ayesian neural network priors revisited.
\newblock \emph{arXiv preprint arXiv:2102.06571}, 2021.

\bibitem[Ghosh et~al.(2018)Ghosh, Yao, and Doshi-Velez]{Ghosh2018}
S.~Ghosh, J.~Yao, and F.~Doshi-Velez.
\newblock Structured variational learning of {B}ayesian neural networks with
  horseshoe priors.
\newblock In \emph{Proceedings of the 35th International Conference on Machine
  Learning (ICML'18)}, pages 1744--1753, 2018.

\bibitem[Ghosh et~al.(2019)Ghosh, Yao, and Doshi-Velez]{Ghosh2019}
S.~Ghosh, J.~Yao, and F.~Doshi-Velez.
\newblock Model selection in {B}ayesian neural networks via horseshoe priors.
\newblock \emph{Journal of Machine Learning Research}, 20\penalty0
  (182):\penalty0 1--46, 2019.

\bibitem[Glorot and Bengio(2010)]{Glorot2010}
X.~Glorot and Y.~Bengio.
\newblock Understanding the difficulty of training deep feedforward neural
  networks.
\newblock In \emph{Proceedings of the 13th International Conference on
  Artificial Intelligence and Statistics (AISTATS'10)}, pages 249--256, 2010.

\bibitem[Gribonval et~al.(2012)Gribonval, Cevher, and
  Davies]{gribonval2012compressible}
R.~Gribonval, V.~Cevher, and M.~E. Davies.
\newblock Compressible distributions for high-dimensional statistics.
\newblock \emph{IEEE Transactions on Information Theory}, 58\penalty0
  (8):\penalty0 5016--5034, 2012.

\bibitem[Griffin and Leisen(2017)]{Griffin2017a}
J.~E. Griffin and F.~Leisen.
\newblock Compound random measures and their use in {B}ayesian non-parametrics.
\newblock \emph{Journal of the Royal Statistical Society: Series B (Statistical
  Methodology)}, 79\penalty0 (2):\penalty0 525--545, 2017.

\bibitem[Gurbuzbalaban et~al.(2021)Gurbuzbalaban, Simsekli, and
  Zhu]{gurbuzbalaban2021heavy}
M.~Gurbuzbalaban, U.~Simsekli, and L.~Zhu.
\newblock The heavy-tail phenomenon in {SGD}.
\newblock In \emph{Proceedings of the 38th International Conference on Machine
  Learning (ICML'21)}, pages 3964--3975, 2021.

\bibitem[He et~al.(2015)He, Zhang, Ren, and Sun]{He2015}
K.~He, X.~Zhang, S.~Ren, and J.~Sun.
\newblock Delving deep into rectifiers: Surpassing human-level performance on
  imagenet classification.
\newblock In \emph{Proceedings of the IEEE International Conference on Computer
  Vision (ICCV'15)}, pages 1026--1034, 2015.

\bibitem[Hjort(1990)]{Hjort1990}
N.~L. Hjort.
\newblock Nonparametric {B}ayes estimators based on beta processes in models
  for life history data.
\newblock \emph{The Annals of Statistics}, pages 1259--1294, 1990.

\bibitem[Hodgkinson and Mahoney(2021)]{hodgkinson2021multiplicative}
L.~Hodgkinson and M.~Mahoney.
\newblock Multiplicative noise and heavy tails in stochastic optimization.
\newblock In \emph{Proceedings of the 38th International Conference on Machine
  Learning (ICML'21)}, pages 4262--4274, 2021.

\bibitem[Hougaard(1986)]{Hougaard1986}
P.~Hougaard.
\newblock Survival models for heterogeneous populations derived from stable
  distributions.
\newblock \emph{Biometrika}, 73\penalty0 (2):\penalty0 387--396, 1986.

\bibitem[Hu et~al.(2017)Hu, Li, Li, and Liu]{hu2017diffusion}
W.~Hu, C.~J. Li, L.~Li, and J.-G. Liu.
\newblock On the diffusion approximation of nonconvex stochastic gradient
  descent.
\newblock \emph{arXiv preprint arXiv:1705.07562}, 2017.

\bibitem[Jacot et~al.(2018)Jacot, Gabriel, and Hongler]{Jacot2018}
A.~Jacot, F.~Gabriel, and C.~Hongler.
\newblock Neural tangent kernel: Convergence and generalization in neural
  networks.
\newblock In \emph{Proceedings of the 32nd Conference on Neural Information
  Processing Systems (NeurIPS'18)}, pages 8571--8580, 2018.

\bibitem[Jang et~al.(2017)Jang, Loeb, Davidow, and Wilson]{Jang2017}
P.~A. Jang, A.~Loeb, M.~Davidow, and A.~G. Wilson.
\newblock Scalable {L}\'evy process priors for spectral kernel learning.
\newblock In \emph{Proceedings of the 31st Conference on Neural Information
  Processing Systems (NeurIPS'17)}, pages 3940--3949, 2017.

\bibitem[Janson(2011)]{Janson2011}
S.~Janson.
\newblock Stable distributions.
\newblock \emph{arXiv preprint arXiv:1112.0220}, 2011.

\bibitem[Jantre et~al.(2021)Jantre, Bhattacharya, and Maiti]{Jantre2021}
S.~Jantre, S.~Bhattacharya, and T.~Maiti.
\newblock Layer adaptive node selection in {B}ayesian neural networks:
  Statistical guarantees and implementation details.
\newblock \emph{arXiv preprint arXiv:2108.11000}, 2021.

\bibitem[Jessen and Mikosch(2006)]{Jessen2006}
H.~A. Jessen and T.~Mikosch.
\newblock Regularly varying functions.
\newblock \emph{Publications de L'institut Mathematique}, 80\penalty0
  (94):\penalty0 171--192, 2006.

\bibitem[Jung et~al.(2023)Jung, Lee, Lee, and Yang]{Jung2021}
P.~Jung, H.~Lee, J.~Lee, and H.~Yang.
\newblock $\alpha $-stable convergence of heavy-tailed infinitely-wide neural
  networks.
\newblock \emph{Advances in Applied Probability}, 2023.

\bibitem[Kallenberg(2002)]{Kallenberg2002}
O.~Kallenberg.
\newblock \emph{Foundations of Modern Probability}.
\newblock Springer, second edition, 2002.

\bibitem[Kuhn et~al.(2021)Kuhn, Lyle, Gomez, Rothfuss, and
  Gal]{kuhn2021robustness}
L.~Kuhn, C.~Lyle, A.~N. Gomez, J.~Rothfuss, and Y.~Gal.
\newblock Robustness to pruning predicts generalization in deep neural
  networks.
\newblock \emph{arXiv preprint arXiv:2103.06002}, 2021.

\bibitem[Lee et~al.(2022)Lee, Yoon, Yang, and Lee]{LeeYYL22}
H.~Lee, E.~Yoon, H.~Yang, and J.~Lee.
\newblock Scale mixtures of neural network {G}aussian processes.
\newblock In \emph{Proceedings of the 10th International Conference on Learning
  Representations (ICLR'22)}, 2022.

\bibitem[Lee et~al.(2016)Lee, James, and Choi]{Lee2016}
J.~Lee, L.~F. James, and S.~Choi.
\newblock Finite-dimensional {BFRY} priors and variational {B}ayesian inference
  for power law models.
\newblock In \emph{Proceedings of the 30th Conference on Neural Information
  Processing Systems (NeurIPS'16)}, pages 3162--3170, 2016.

\bibitem[Lee et~al.(2018)Lee, Bahri, Novak, Schoenholz, Pennington, and
  Sohl-Dickstein]{Lee2018}
J.~Lee, Y.~Bahri, R.~Novak, S.~S. Schoenholz, J.~Pennington, and
  J.~Sohl-Dickstein.
\newblock Deep neural networks as {G}aussian processes.
\newblock In \emph{Proceedings of the 6th International Conference on Learning
  Representations (ICLR'18)}, 2018.

\bibitem[Lee et~al.(2019)Lee, Miscouridou, and Caron]{Lee2019}
J.~Lee, X.~Miscouridou, and F.~Caron.
\newblock A unified construction for series representations and finite
  approximations of completely random measures.
\newblock \emph{arXiv preprint arXiv:1905.10733}, 2019.

\bibitem[LePage et~al.(1981)LePage, Woodroofe, and Zinn]{Lepage81}
R.~LePage, M.~Woodroofe, and J.~Zinn.
\newblock Convergence to a stable distribution via order statistics.
\newblock \emph{The Annals of Probability}, pages 624--632, 1981.

\bibitem[Littwin et~al.(2021)Littwin, Saremi, Zhai, Thilak, Goh, Susskind, and
  Yang]{Littwin2021}
E.~Littwin, O.~Saremi, S.~Zhai, V.~Thilak, H.~Goh, J.~M. Susskind, and G.~Yang.
\newblock Implicit acceleration and feature learning in infinitely wide neural
  networks with bottlenecks.
\newblock \emph{CoRR}, abs/2107.00364, 2021.

\bibitem[Louizos et~al.(2017)Louizos, Ullrich, and Welling]{Louizos2017}
C.~Louizos, K.~Ullrich, and M.~Welling.
\newblock {B}ayesian compression for deep learning.
\newblock In \emph{Proceedings of the 31st Conference on Neural Information
  Processing Systems (NeurIPS'17)}, pages 3288--3298, 2017.

\bibitem[Mandt et~al.(2016)Mandt, Hoffman, and Blei]{mandt2016variational}
S.~Mandt, M.~Hoffman, and D.~Blei.
\newblock A variational analysis of stochastic gradient algorithms.
\newblock In \emph{Proceedings of the 33rd International Conference on Machine
  Learning (ICML'16)}, pages 354--363, 2016.

\bibitem[Martin and Mahoney(2019)]{Martin2019}
C.~H. Martin and M.~W. Mahoney.
\newblock Traditional and heavy-tailed self regularization in neural network
  models.
\newblock \emph{arXiv preprint arXiv:1901.08276}, 2019.

\bibitem[Matthews et~al.(2018)Matthews, Hron, Rowland, Turner, and
  Ghahramani]{Matthews2018}
A.~G. d.~G. Matthews, J.~Hron, M.~Rowland, R.~E. Turner, and Z.~Ghahramani.
\newblock {G}aussian process behaviour in wide deep neural networks.
\newblock In \emph{Proceedings of the 6th International Conference on Learning
  Representations (ICLR'18)}, 2018.

\bibitem[McCullagh and Polson(2018)]{McCullagh2018}
P.~McCullagh and N.~Polson.
\newblock Statistical sparsity.
\newblock \emph{Biometrika}, 105\penalty0 (4):\penalty0 797--814, 2018.

\bibitem[Murray and Chiang(2015)]{Murray2015}
K.~Murray and D.~Chiang.
\newblock Auto-sizing neural networks: With applications to n-gram language
  models.
\newblock In \emph{Proceedings of the 2015 Conference on Empirical Methods in
  Natural Language Processing (EMNLP'15)}, pages 908--916, 2015.

\bibitem[Neal(1995)]{Neal1995}
R.~M. Neal.
\newblock \emph{{B}ayesian Learning for Neural Networks}.
\newblock PhD thesis, Department of Computer Science, University of Toronto,
  1995.

\bibitem[Neal(1996)]{Neal1996}
R.~M. Neal.
\newblock Priors for infinite networks.
\newblock In \emph{Bayesian Learning for Neural Networks}, pages 29--53.
  Springer New York, 1996.

\bibitem[Ober and Aitchison(2021)]{Ober2021}
S.~W. Ober and L.~Aitchison.
\newblock Global inducing point variational posteriors for {B}ayesian neural
  networks and deep {G}aussian processes.
\newblock In \emph{Proceedings of the 38th International Conference on Machine
  Learning (ICML'21)}, pages 8248--8259, 2021.

\bibitem[Ochiai et~al.(2017)Ochiai, Matsuda, Watanabe, and
  Katagiri]{Ochiai2017}
T.~Ochiai, S.~Matsuda, H.~Watanabe, and S.~Katagiri.
\newblock Automatic node selection for deep neural networks using group lasso
  regularization.
\newblock In \emph{Proceedings of the 2017 IEEE International Conference on
  Acoustics, Speech and Signal Processing (ICASSP'17)}, pages 5485--5489. IEEE,
  2017.

\bibitem[Perman et~al.(1992)Perman, Pitman, and Yor]{Perman1992}
M.~Perman, J.~Pitman, and M.~Yor.
\newblock Size-biased sampling of {P}oisson point processes and excursions.
\newblock \emph{Probability Theory and Related Fields}, 92\penalty0
  (1):\penalty0 21--39, 1992.

\bibitem[Polson and Scott(2012)]{Polson2012}
N.~Polson and J.~Scott.
\newblock Local shrinkage rules, {L}{\'e}vy processes and regularized
  regression.
\newblock \emph{Journal of the Royal Statistical Society: Series B (Statistical
  Methodology)}, 74\penalty0 (2):\penalty0 287--311, 2012.

\bibitem[Popkes et~al.(2019)Popkes, Overweg, Ercole, Li, Hern{\'a}ndez-Lobato,
  Zaykov, and Zhang]{Popkes2019}
A.-L. Popkes, H.~Overweg, A.~Ercole, Y.~Li, J.~M. Hern{\'a}ndez-Lobato,
  Y.~Zaykov, and C.~Zhang.
\newblock Interpretable outcome prediction with sparse {B}ayesian neural
  networks in intensive care.
\newblock \emph{arXiv preprint arXiv:1905.02599}, 2019.

\bibitem[Raman et~al.(2009)Raman, Fuchs, Wild, Dahl, and Roth]{Raman2009}
S.~Raman, T.~Fuchs, P.~Wild, E.~Dahl, and V.~Roth.
\newblock The {B}ayesian group-lasso for analyzing contingency tables.
\newblock In \emph{Proceedings of the 26th International Conference on Machine
  Learning (ICML'09)}, pages 881--888, 2009.

\bibitem[Resnick(2005)]{Resnick2005}
S.~Resnick.
\newblock Heavy tailed analysis eurandom summer 2005.
\newblock 2005.

\bibitem[Samorodnitsky and Taqqu(1994)]{samorodnitsky1994stable}
G.~Samorodnitsky and M.~Taqqu.
\newblock \emph{Stable Non-{G}aussian Random Processes: Stochastic Models with
  Infinite Variance: Stochastic Modeling}.
\newblock Chapman \& Hall, 1994.

\bibitem[Sato(1999)]{Sato1999levy}
K.~Sato.
\newblock \emph{L{\'e}vy Processes and Infinitely Divisible Distributions}.
\newblock Cambridge university press, 1999.

\bibitem[Scardapane et~al.(2017)Scardapane, Comminiello, Hussain, and
  Uncini]{Scardapane2017}
S.~Scardapane, D.~Comminiello, A.~Hussain, and A.~Uncini.
\newblock Group sparse regularization for deep neural networks.
\newblock \emph{Neurocomputing}, 241:\penalty0 81--89, 2017.

\bibitem[Shin(2021)]{shin2021compressing}
J.~Y. Shin.
\newblock Compressing heavy-tailed weight matrices for non-vacuous
  generalization bounds.
\newblock \emph{arXiv preprint arXiv:2105.11025}, 2021.

\bibitem[Suzuki et~al.(2019)Suzuki, Abe, and Nishimura]{suzuki2019compression}
T.~Suzuki, H.~Abe, and T.~Nishimura.
\newblock Compression based bound for non-compressed network: unified
  generalization error analysis of large compressible deep neural network.
\newblock \emph{arXiv preprint arXiv:1909.11274}, 2019.

\bibitem[Suzuki et~al.(2020)Suzuki, Abe, Murata, Horiuchi, Ito, Wachi, Hirai,
  Yukishima, and Nishimura]{suzuki2018spectral}
T.~Suzuki, H.~Abe, T.~Murata, S.~Horiuchi, K.~Ito, T.~Wachi, S.~Hirai,
  M.~Yukishima, and T.~Nishimura.
\newblock Spectral pruning: Compressing deep neural networks via spectral
  analysis and its generalization error.
\newblock In \emph{Proceedings of the 29th International Joint Conference on
  Artificial Intelligence (IJCAI'20)}, pages 2839--2846, 2020.

\bibitem[Teh and Gorur(2009)]{Teh2009}
Y.~W. Teh and D.~Gorur.
\newblock Indian buffet processes with power-law behavior.
\newblock In \emph{Proceedings of the 23rd Conference on Neural Information
  Processing Systems (NeurIPS'09)}, pages 1838--1846, 2009.

\bibitem[Thibaux and Jordan(2007)]{Thibaux2007}
R.~Thibaux and M.~I. Jordan.
\newblock Hierarchical beta processes and the {I}ndian buffet process.
\newblock In \emph{Proceedings of the 11th International Conference on
  Artificial Intelligence and Statistics (AISTATS'07)}, pages 564--571, 2007.

\bibitem[Tsuchida et~al.(2019)Tsuchida, Roosta, and Gallagher]{Tsuchida2019}
R.~Tsuchida, F.~Roosta, and M.~Gallagher.
\newblock Richer priors for infinitely wide multi-layer perceptrons.
\newblock \emph{arXiv preprint arXiv:1911.12927}, 2019.

\bibitem[Wang et~al.(2017)Wang, Xu, Yang, and Zurada]{Wang2017}
J.~Wang, C.~Xu, X.~Yang, and J.~M. Zurada.
\newblock A novel pruning algorithm for smoothing feedforward neural networks
  based on group lasso method.
\newblock \emph{IEEE Transactions on Neural Networks and Learning Systems},
  29\penalty0 (5):\penalty0 2012--2024, 2017.

\bibitem[Wenzel et~al.(2020)Wenzel, Roth, Veeling, Swiatkowski, Tran, Mandt,
  Snoek, Salimans, Jenatton, and Nowozin]{Wenzel2020}
F.~Wenzel, K.~Roth, B.~S. Veeling, J.~Swiatkowski, L.~Tran, S.~Mandt, J.~Snoek,
  T.~Salimans, R.~Jenatton, and S.~Nowozin.
\newblock How good is the bayes posterior in deep neural networks really?
\newblock In \emph{Proceedings of the 37th International Conference on Machine
  Learning (ICML'20)}, pages 10248--10259, 2020.

\bibitem[Wolinski et~al.(2020{\natexlab{a}})Wolinski, Charpiat, and
  Ollivier]{Wolinski2020}
P.~Wolinski, G.~Charpiat, and Y.~Ollivier.
\newblock An equivalence between {B}ayesian priors and penalties in variational
  inference.
\newblock \emph{arXiv preprint arXiv:2002.00178}, 2020{\natexlab{a}}.

\bibitem[Wolinski et~al.(2020{\natexlab{b}})Wolinski, Charpiat, and
  Ollivier]{Wolinski2020a}
P.~Wolinski, G.~Charpiat, and Y.~Ollivier.
\newblock Asymmetrical scaling layers for stable network pruning.
\newblock 2020{\natexlab{b}}.

\bibitem[Wolpert et~al.(2011)Wolpert, Clyde, and Tu]{Wolpert2011}
R.~L. Wolpert, M.~A. Clyde, and C.~Tu.
\newblock Stochastic expansions using continuous dictionaries: L{\'e}vy
  adaptive regression kernels.
\newblock \emph{Annals of Statistics}, 39\penalty0 (4):\penalty0 1916--1962,
  2011.

\bibitem[Yang(2019)]{Yang2019}
G.~Yang.
\newblock Wide feedforward or recurrent neural networks of any architecture are
  {G}aussian processes.
\newblock In \emph{Proceedings of the 33rd Conference on Neural Information
  Processing Systems (NeurIPS'19)}, pages 9947--9960, 2019.

\bibitem[Yuan and Lin(2006)]{Yuan2006}
M.~Yuan and Y.~Lin.
\newblock Model selection and estimation in regression with grouped variables.
\newblock \emph{Journal of the Royal Statistical Society: Series B (Statistical
  Methodology)}, 68\penalty0 (1):\penalty0 49--67, 2006.

\bibitem[Zhang et~al.(2020)Zhang, Li, Zhang, Chen, and Wilson]{Zhang2020}
R.~Zhang, C.~Li, J.~Zhang, C.~Chen, and A.~G. Wilson.
\newblock Cyclical stochastic gradient {MCMC} for {Bayesian} deep learning.
\newblock In \emph{Proceedings of the 8th International Conference on Learning
  Representations (ICLR'20)}, 2020.

\bibitem[Zhu et~al.(2019)Zhu, Wu, Yu, Wu, and Ma]{zhu2018anisotropic}
Z.~Zhu, J.~Wu, B.~Yu, L.~Wu, and J.~Ma.
\newblock The anisotropic noise in stochastic gradient descent: Its behavior of
  escaping from sharp minima and regularization effects.
\newblock In \emph{Proceedings of the 36th International Conference on Machine
  Learning (ICML'19)}, pages 7654--7663, 2019.

\end{thebibliography}

\end{document}